\documentclass{article}

\usepackage{microtype}
\usepackage{graphicx}
\usepackage{subcaption}
\usepackage{booktabs} 

\usepackage[ruled,vlined]{algorithm2e}
\usepackage{multicol, multirow}
\usepackage[table]{xcolor}
\usepackage{colortbl} 
\usepackage{rotating}
\usepackage{array}
\usepackage{arydshln} 
\usepackage{float}

\usepackage{booktabs}
\usepackage{xcolor} 
\usepackage{threeparttablex}
\usepackage{wrapfig}  
\usepackage{rotating}
\newcommand{\boldres}[1]{{\textbf{\textcolor{black}{#1}}}}
\newcommand{\secondres}[1]{{\underline{\textcolor{black}{#1}}}}
\usepackage[most]{tcolorbox}

\definecolor{lightblue}{rgb}{0.85,0.95,1}
\definecolor{lightgray}{rgb}{0.95,0.95,0.95}

\usepackage{hyperref}

\usepackage[preprint]{icml2026}

\usepackage{amsmath}
\usepackage{amssymb}
\usepackage{mathtools}
\usepackage{amsthm}

\usepackage[capitalize,noabbrev]{cleveref}

\theoremstyle{plain}
\newtheorem{theorem}{Theorem}[section]

\theoremstyle{definition}
\newtheorem{definition}[theorem]{Definition}

\theoremstyle{remark}

\usepackage[textsize=tiny]{todonotes}

\icmltitlerunning{Bi-level Heterogeneous Learning for Time Series Foundation Models}

\begin{document}

\twocolumn[
  \icmltitle{Bi-level Heterogeneous Learning for Time Series Foundation Models: \\
  A Federated Learning Approach}



  \icmlsetsymbol{equal}{*}

  \begin{icmlauthorlist}
    \icmlauthor{Shengchao Chen}{yyy}
    \icmlauthor{Guodong Long}{yyy}
    \icmlauthor{Dikai Liu}{zzz}
    \icmlauthor{Jing Jiang}{yyy}
  \end{icmlauthorlist}

  \icmlaffiliation{yyy}{Australian AI Institute, University of Technology Sydney}
  \icmlaffiliation{zzz}{Robotics Institute, University of Technology Sydney}

  \icmlcorrespondingauthor{Shengchao Chen}{shengchao.chen@uts.edu.au}

  \icmlkeywords{Machine Learning, ICML}

  \vskip 0.3in
]



\printAffiliationsAndNotice{}  

\begin{abstract}
Heterogeneity in time series data is more pronounced than in vision or language, as temporal dynamics vary substantially across domains and tasks. Existing efforts on training time series foundation models (TSFMs) from scratch are often trained with mixed-batch strategies that merge large-scale datasets, which can cause gradient conflicts and degrade representation quality. To address this, we propose a fine-grained learning method that distills invariant knowledge from heterogeneous series while reducing cross-domain interference. We characterize heterogeneity at two levels: inter-domain and intra-domain. To tackle this bi-level heterogeneity, we design a federated learning method that mitigates intra-domain conflicts by enforcing domain-invariant and semantically consistent representations through local regularization, and addresses inter-domain discrepancies by enhancing cross-domain collaboration via domain-aware aggregation. Experiments across diverse benchmarks show that TSFMs trained with our method consistently outperform both centralized and federated TSFM baselines in point and probabilistic forecasting, while also achieving competitive zero-shot performance at scale, offering a flexible pathway for training TSFMs from scratch in heterogeneous environments.
\end{abstract}

\section{Introduction}
Time series forecasting plays a critical role in decision-making domains such as weather~\cite{bi2023accurate}, energy~\cite{wu2022timesnet}, and urban computing~\cite{zhang2024towards}. Recent advances in foundation models have shifted research from task-specific structures toward general-purpose forecasting models trained on diverse datasets~\cite{goswami2024moment,shi2024time}. However, existing time series foundation models (TSFMs) are often built under centralized training paradigms~\cite{shi2024time,das2024decoder,woo2024unifiedtraininguniversaltime}, where heterogeneous datasets are simply merged for pretraining. In practice, this mixed-batch strategy can amplify gradient conflicts and obscure domain-specific structures, ultimately limiting the quality of learned representations. Moreover, temporal data are inherently fragmented across organizations and applications, making centralized collection impractical and further aggravating heterogeneity across diverse domains.

Concretely, these limitations manifest as a form of bi-level heterogeneity. The first level is inter-domain heterogeneity~\cite{chen2025federated}, where datasets from different sources exhibit covariate shifts and divergent temporal patterns due to variations in sensing modalities, sampling rates, and environments. This misalignment often causes centralized or mixed-batch training to overfit domain-specific signals and fail to capture globally consistent dynamics. The second level is intra-domain conflict, which arises within a single dataset or domain~\cite{wang2024medformer}. Temporal concept drift from reconfiguration or task evolution, together with latent sub-domains such as geographically distributed devices, leads to semantically similar but contextually divergent fragments. Naive training over such mixtures produces incoherent embeddings that are neither specialized nor aligned within sub-domains. The bi-level heterogeneity undermines the learning of domain-invariant and temporally coherent patterns, thereby limiting the generalization of TSFMs.

A more flexible learning paradigm is required to handle bi-level heterogeneity in TSFMs training, one that can integrate knowledge across diverse domains while limiting cross-domain interference. Federated learning (FL)~\cite{mcmahan2017communication} naturally offers such a strategy: by enabling decentralized training over distributed datasets, it allows domains to collaborate without centralizing raw data. Yet, existing FL methods primarily address inter-domain heterogeneity through personalization~\cite{tan2022towards} or domain-aware adaptations~\cite{yang2023survey}, while typically assuming that each client`s local data is homogeneous and stationary. While such assumptions may hold in static modalities like CV and NLP, they break down in time series domains where a single client can contain multiple latent sub-domains or undergo temporal concept drift. These intra-domain conflicts blur local representations, leaving them neither domain-specialized nor globally aligned. Recent work~\cite{chen2025federated} explores task-agnostic FL for time series, but it largely targets inter-domain misalignment while overlooking intra-domain conflicts and evolving semantics. The open question is how to design a FL paradigm that can simultaneously resolve bi-level heterogeneity, and in doing so, provide a scalable pathway for building robust TSFMs.

\begin{figure}[tbh]
    \centering
    \includegraphics[width=1\columnwidth]{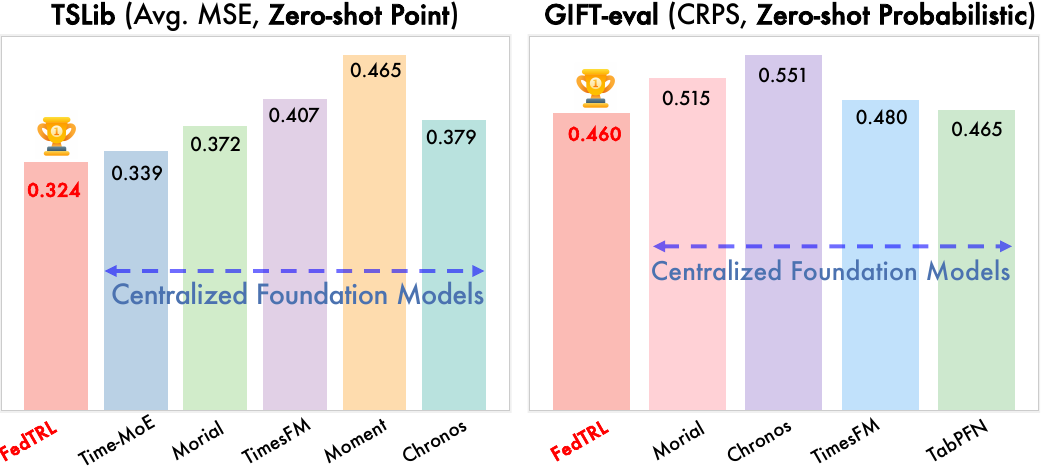}  
    \caption{FedTRL outperforms centralized TSFMs in \emph{zero-shot point} and \emph{probabilistic forecasting}. Details are provided in Sec.~\ref{sec:main_res}}
    \label{fig:intro}
    \vspace{-3pt}
\end{figure}

To this end, we propose FedTRL, a FL method for bi-level heterogeneous learning that enables domain-invariant time series representations under both inter-domain heterogeneity and intra-domain conflicts. Our FedTRL integrates two complementary components, including: (1) a domain-adversarial optimization objective combined with prototype alignment, which mitigates sub-domain drift and suppresses domain-specific artifacts to enforce semantically consistent local representations, and (2) a domain-aware aggregation strategy that adaptively balances invariance and alignment across domains. FedTRL enables the training of robust TSFMs under decentralization. As shown in \textbf{Fig.~\ref{fig:intro}}, FedTRL not only closes but surpasses centralized TSFMs in both zero-shot point (TSLib~\cite{wu2022timesnet}) and probabilistic forecasting (GIFT-eval~\cite{aksu2024gift}), demonstrating its effectiveness. Our contributions are summarized as:
\begin{itemize}
    \item We propose FedTRL, a FL method for bi-level heterogeneous learning that explicitly addresses both inter-domain and intra-domain heterogeneity challenges, providing a scalable and more flexible pathway for training time series foundation models from scratch.
    \item We propose a fine-grained joint optimization \& aggregation that combines domain-adversarial regularization with prototype alignment and domain-aware aggregation, enforcing semantically consistent representations while adaptively weighting client updates.
    \item Extensive experiments across representation learning, full/zero-shot forecasting, and probabilistic forecasting benchmarks show that FedTRL effectively scales FL to train TSFMs, achieving performance competitive with or even superior to centralized approaches.
\end{itemize}

\section{Related Work}
\paragraph{Unsupervised Time Series Representation Learning.} Unsupervised representation learning extracts general temporal patterns from unlabeled sequences~\cite{zhang2024self}, forming the foundation of TSFMs (\emph{see Appendix~\ref{sec:more_related}}), in contrast to task-specific supervised methods~\cite{wu2022timesnet}. Existing works mainly fall into masked reconstruction and contrastive learning. Masked reconstruction predicts missing inputs, with point-wise strategies such as TST~\cite{zerveas2021transformer}, TimeMAE~\cite{cheng2023timemae}, and SimMTM~\cite{dong2023simmtm}, and patch-wise variants like PatchTST~\cite{nie2022time} and CrossTimeNet~\cite{cheng2025cross} that better capture local structures. Contrastive learning instead distinguishes positives from negatives, exemplified by TS-TCC~\cite{eldele2021time}, which exploits temporal smoothness, and CoST~\cite{woo2022cost}, which leverages time–frequency augmentations. While effective on single datasets, these approaches face two issues when scaled to TSFMs: \emph{(i) dependence on dataset-specific fine-tuning} and \emph{(ii) conflicting embeddings across heterogeneous domains}. This highlights the need for more flexible training paradigms that unify representation learning across domains, where FL~\cite{mcmahan2017communication} offers a strategy for collaborative pretraining on decentralized heterogeneous data.

\paragraph{Federated Learning in Time Series.} Recent studies have extended FL to time series, exploring personalized cross-domain adaptation~\cite{liu2024time,soi2024federated} and task-agnostic modeling across domains~\cite{chen2025federated, chen2023prompt, sun2024hifi}. However, these methods often assume that local data within each client is homogeneous and domain-stable, an assumption rarely satisfied in practice. Most of them primarily target inter-domain heterogeneity across clients while overlooking intra-domain variability within individual clients. In reality, a single client may contain multiple latent sub-domains with divergent patterns, and cross-client distributions often remain misaligned despite similar marginal statistics. Such bi-level heterogeneity fundamentally disrupts representation learning, yielding incoherent local embeddings and misaligned global models that struggle to generalize across domains and tasks. To address these challenges, we propose FedTRL, a FL method for bi-level heterogeneous learning that explicitly tackles both inter-domain and intra-domain heterogeneity, enabling robust and scalable pretraining of TSFMs from scratch.

\begin{figure*}[tbh]
    \centering
    \includegraphics[width=.95\textwidth]{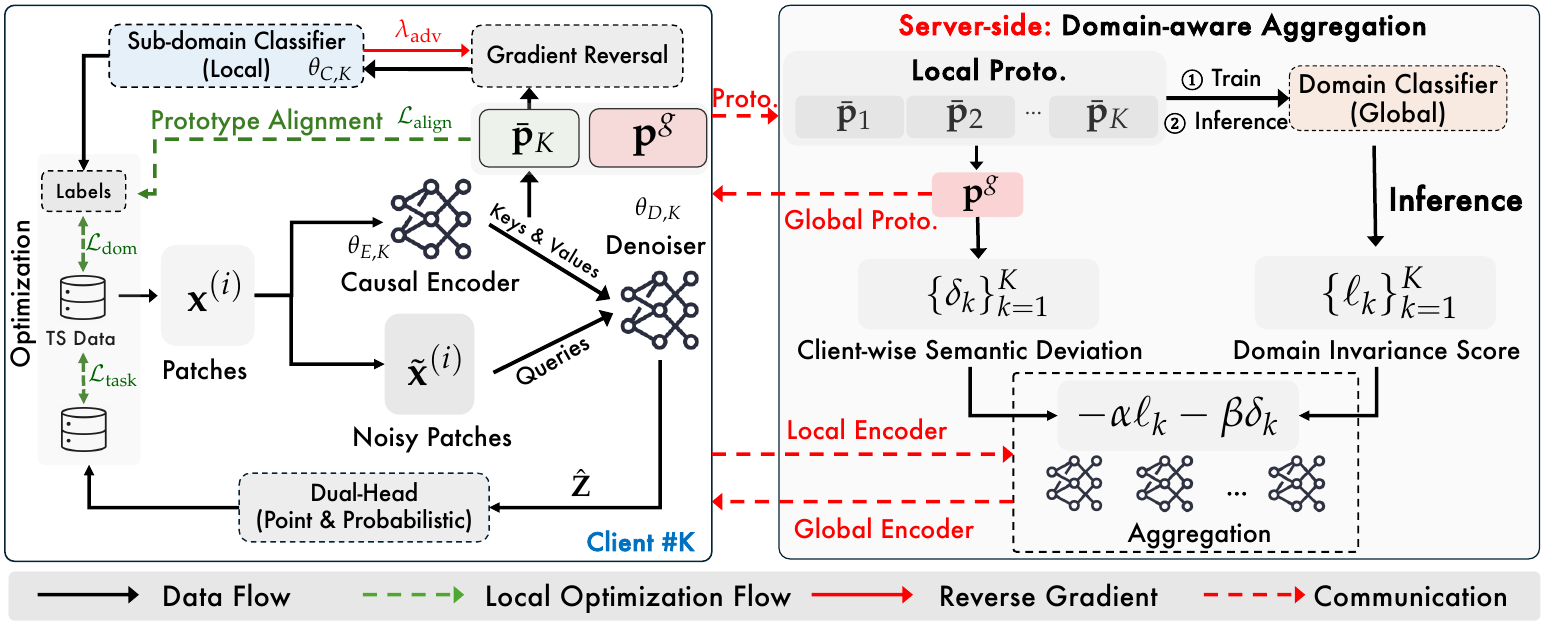}
    \caption{Structure of FedTRL (single client example). Each client performs unsupervised diffusion-based reconstruction and uploads only the encoder $\theta_E$ and prototypes $\bar{\mathbf{p}}$ per round; the server applies domain-aware aggregation (DaG) to produce a unified encoder.}
    \label{fig:overall_arch}
    \vspace{-6pt}
\end{figure*}
\section{Methodology}
\paragraph{Problem Definition.} We consider an FL system with a server and $K$ clients (domains), where each client $k$ holds a local dataset $\mathcal{D}_k$ that may include data from one/multiple sub-domains. These datasets are heterogeneous due to differences in sampling rates, temporal dynamics, and domain-specific patterns. The objective is to collaboratively learn a unified representation that generalizes across domains and supports diverse forecasting tasks without supervision. The global objective can be formulated as:
\begin{equation}
    \min_{\theta} F(\theta) := \sum\nolimits_{k=1}^{K} \frac{n_k}{n} F_k(\theta_k; \mathcal{D}_k),
\end{equation}
where $F_k$ is the local objective on client (domain) $k$, $n_k$ is its sample size, and $n=\sum_{k=1}^K {n_k}$. However, this standard objective fails to yield generalizable patterns because: (i) \emph{inter-domain heterogeneity}, where distinct distributions across domains lead to divergent local optima $\theta_k^*$, making naive aggregation ineffective; and (ii) \emph{intra-domain conflicts}, where a single domain host multiple latent sub-domains, and contextual shifts create inconsistent or hybrid embeddings.

\paragraph{Overview.}
We propose FedTRL (\textbf{Fig.~\ref{fig:overall_arch}}), a FL method for bi-level heterogeneous learning that explicitly addresses both inter- and intra-domain heterogeneity. FedTRL integrates: \textbf{(i)} a fine-grained local optimization objective that enforces domain invariance through adversarial supervision and prototype alignment and \textbf{(ii)} an domain-aware aggregation mechanism that re-weights client updates by combining domain discrimination risk with representation alignment.

\subsection{Local Updating on Clients}
To address heterogeneity at different levels, each client optimizes a unified objective: (i) reconstruction for temporal semantics and (ii) adversarial regularization for mitigating intra-domain conflicts, while (iii) prototype alignment is designed to promote inter-domain consistency. We first decompose a multivariate time series $\mathbf{X} \in \mathbb{R}^{T \times C}$ of length $T$ into univariate series $\mathbf{x} \in \mathbb{R}^T$ along channels. Each univariate series is first stationarized using a non-parametric two-stage instance normalization~\cite{liu2022non} to mitigate temporal distribution shifts and anomalous ranges, thereby improving robustness across latent sub-domains. We then divide each sequence into non-overlapping patches of length $P$, yielding $N = \lfloor T/P \rfloor$ patches $\mathcal{P} = \{\mathbf{x}^{(i)}\}_{i=1}^N$. This preserves local temporal structure while reducing computation for long sequences. Given these patches, we adopt diffusion  reconstruction~\cite{wang2025timedartdiffusionautoregressivetransformer} to perform unsupervised representation learning. Specifically, we apply a forward diffusion process $\mathcal{D}_{\text{fwd}}$ to obtain its noisy counterpart:
\begin{equation}
    \tilde{\mathbf{x}}^{(i)} = \sqrt{\bar{\alpha}_t}\,\mathbf{x}^{(i)} + \sqrt{1-\bar{\alpha}_t}\,\boldsymbol{\epsilon}, \boldsymbol{\epsilon}\sim\mathcal{N}(0,I), \ t\sim\mathcal{U}(0,T_d)
\end{equation}
Unlike contrastive~\cite{woo2022cost} or masked reconstruction methods~\cite{chen2025federated} that implicitly assume data homogeneity, our approach prioritizes intrinsic temporal dynamics over domain-specific artifacts by jointly modeling point-, patch-, and sequence-level patterns. All clean patches are embedded and stacked into a token sequence: $\mathbf{E} \in \mathbb{R}^{N \times d}$, which is processed by a causal Transformer encoder $\theta_E$: $\mathbf{H} = \theta_E(\mathbf{E}) \in \mathbb{R}^{N \times d}$. In parallel, noisy patches $\tilde{\mathbf{x}}^{(i)}$ are embedded as query tokens $\tilde{\mathbf{E}} \in \mathbb{R}^{N \times d}$. A denoising decoder $\theta_D$ reconstructs latent signals using cross-attention mechanism, with queries $\tilde{\mathbf{E}}\mathbf{W}_q$ attending to encoder outputs $\mathbf{H}\mathbf{W}_k$,$\mathbf{H}\mathbf{W}_v$ as $\hat{\mathbf{Z}} = \theta_D(\tilde{\mathbf{E}}\mathbf{W}_q,\ \mathbf{H}\mathbf{W}_k,\ \mathbf{H}\mathbf{W}_v)$.

\vspace{-2pt}
\paragraph{Subdomain-Adversarial and Prototype-Aligned Regularization.} To mitigate conflicts among semantically related but contextually divergent sub-domains, we adapt subdomain-invariant learning via adversarial training~\cite{ganin2016domain}. We introduce a sub-domain classifier $D_{\theta_C}$ on clients, trained to predict the sequence-level sub-domain label $y^{\text{dom}}$, where each time series (and all patches derived from it) shares the same sub-domain category. Given encoder outputs $\mathbf{H}$, we compute the temporal prototype by mean pooling $\bar{\mathbf{p}}$ = $\mathrm{Mean}(\mathbf{H})$. The classifier objective is
\begin{equation}
    \mathcal{L}_{\text{dom}} = \frac{1}{N} \sum\nolimits_{i=1}^N \mathrm{CE}\!\big(D_{\theta_C}(\mathrm{GRL}(\bar{\mathbf{p}}^{(i)})),\ y^{\text{dom}}_i \big),
\end{equation}
where GRL (Gradient Reverse Layer) acts as identity in the forward pass and scales gradients by $-\lambda$ in the backward pass, encouraging the encoder to suppress subdomain-specific signals. The parameter $\lambda$ balances temporal modeling and subdomain invariance. Beyond local adversarial regularization, we facilitate cross-domain knowledge transfer via lightweight prototype-based alignment. Rather than constraining full model parameters, which are entangled with domain-specific dynamics, each client aligns its local prototype $\bar{\mathbf{p}}$ with the global prototype $\mathbf{p}^g$ as:
\begin{equation}\label{eq:align}
    \mathcal{L}_{\text{align}} = ||\bar{\mathbf{p}} - \mathbf{p}^g ||_2^2.
\end{equation}
This ensures semantic consistency across clients while preserving local flexibility, yielding a domain-invariant representation space that facilitates robust aggregation.
\vspace{-2pt}

\paragraph{Point and Probabilistic Forecasting Support.} Each latent token $\hat{\mathbf{z}}^{(i)} \in \hat{\mathbf{Z}}$ is projected back to the patch space through an MLP, yielding reconstructed patches $\hat{\mathbf{x}}^{(i)}$. On top of these shared representations, we support both point and probabilistic forecasting within a unified dual-head architecture. The deterministic head outputs point predictions $\hat{y}_i$, optimized by mean square error (MSE), while the probabilistic head parameterizes a Student-t distribution with mean $\mu_i$, scale $\sigma_i$ and degrees of freedom $\nu$ (fixed $\nu=5$ empirically) encouraging the model to capture heavy-tailed uncertainties via negative log-likelihood (NLL). Formally, the joint training objective can be formulated as:
\begin{equation}
\begin{aligned}
\mathcal{L}_{\text{task}}
&= \frac{1}{N}\sum_{i=1}^{N}(y_i-\hat{y}_i)^2
+ \beta \cdot \frac{1}{N}\sum_{i=1}^{N}\Bigg[
\frac{1}{2}\log\!\big(\pi(\nu-2)\sigma_i^2\big) \\
&\qquad\qquad\qquad
+ \frac{\nu+1}{2}\log\!\left(1+\frac{(y_i-\mu_i)^2}{(\nu-2)\sigma_i^2}\right)
\Bigg],
\end{aligned}
\end{equation}
where $y_i$ is the ground-truth value of the forecasting target derived from the original time series $\mathbf{x}$, and $\beta$ balances the deterministic and probabilistic objectives. To stabilize optimization under federated heterogeneity, we adopt a warm-up schedule: only the deterministic head is trained in the initial $R_{\text{warm}}$ rounds, after which the probabilistic head is gradually activated by annealing $\beta=0$ to $1$. This training design ensures accurate point forecasts first, followed by well-calibrated probabilistic predictions.

\paragraph{Local Optimization Objective.} Each client jointly optimizes the three components introduced above. The reconstruction objective ensures temporal fidelity, the adversarial loss suppresses subdomain-specific bias, and the prototype alignment term promotes semantic consistency across clients. The overall local objective can be formulated as:
\begin{equation}
\mathcal{L} = \mathcal{L}_{\text{task}} + \lambda_{\text{dom}} \cdot \mathcal{L}_{\text{dom}} + \lambda_{\text{align}} \cdot \mathcal{L}_{\text{align}},
\end{equation}
where $\lambda_{\text{dom}}$ and $\lambda_{\text{align}}$ are balancing coefficients controlling the contribution of subdomain adversarial regularization and prototype alignment, respectively. Only the encoder $\theta_E$ and the local prototypes $\bar{\mathbf{p}}$ are shared with the server, as the goal is to collaboratively learn a domain-generalizable encoder. The decoder $\theta_D$ and sub-domain classifier $\theta_C$ are retained locally as auxiliary components to facilitate denoising and mitigate domain-specific bias. Their exclusion from aggregation prevents local semantics from interfering with the global encoder, enhancing generalization across domains.

\subsection{Domain-aware Aggregation on Sever}
Naive FL aggregation~\cite{mcmahan2017communication} often dilutes discriminative signals and exacerbates domain entanglement under heterogeneous time series. To address this, we propose Domain-aware Aggregation (DaG), which integrates invariance and alignment into encoder aggregation. The server receives local encoders $\{\theta_{E,k}\}_{k=1}^{K}$ and prototypes $\{\bar{\mathbf{p}}_k\}_{k=1}^{K}$, and first updates a global prototype $\mathbf{p}_g$ via FedAvg. To measure invariance, we train a global domain classifier $f_d:\mathbb{R}^d \to \mathbb{R}^K$ using local prototypes, where the client (domain) index serves as the category label. After training, the classifier is fixed for inference. For each prototype $\bar{\mathbf{p}}_k$, the Domain Invariance Score is defined as $\ell_k=\mathcal{L}_{\mathrm{CE}}(f_d(\bar{\mathbf{p}}_k),k)$: domain-specific prototypes yield low $\ell_k$, while domain-invariant ones approach uniform predictions and higher $\ell_k$. To complement this, we also measure semantic alignment with the global prototype (as Eq.~\ref{eq:align}), which captures residual domain drift beyond classifier predictions. We then integrate the two signals into a unified score: $s_k = -\alpha \ell_k - \beta \delta_k$, with $\alpha,\beta>0$ balancing discriminability and alignment. Final aggregation is obtained via softmax weighting, which can be formulated as: 
\begin{equation}
\theta_E^{g} = \sum\nolimits_{k=1}^{K} \operatorname{Softmax}(s_k) \cdot \theta_{E,k}.
\end{equation}
By dynamically weighting client contributions based on both invariance and alignment, DaG prevents domain-dominant clients from skewing the global encoder and enhances cross-domain generalization. The global prototype $\mathbf{p}^g$ and updated global model $\theta_E^g$ are broadcast to clients for the next training round until convergence. The workflow description and algorithm are provided in \textbf{Algorithm~\ref{alg:fedtrl_workflow}}.

\begin{table*}[tbh]
    \centering
    \caption{In-domain forecasting results (averaged across forecasting horizons $\{96, 192, 336, 720\}$). \textbf{Bold}: the best; \underline{Underline}: the second best. The $^\dagger$ symbol denotes that based on the FedAvg aggregation protocol. Full results are provided in \textbf{Table \ref{tab:forecasting-full}}.}
    \label{tab:forecasting-in-sim}
        \resizebox{.95\textwidth}{!}{
        \begin{tabular}{l|cc|cccccccc|cccc|cc}
        \toprule
         & \multicolumn{2}{c|}{\textsc{Ours}}     
         & \multicolumn{8}{c|}{\textsc{Federated Self-supervised}$^\dagger$}                          
         & \multicolumn{4}{c|}{\textsc{Federated FMs}} 
         & \multicolumn{2}{c}{\textsc{Centralized}} \\
        Methods & \multicolumn{2}{c|}{\textbf{FedTRL}} 
                & \multicolumn{2}{c}{SimMTM} 
                & \multicolumn{2}{c}{PatchTST} 
                & \multicolumn{2}{c}{TimeMAE} 
                & \multicolumn{2}{c|}{CoST} 
                & \multicolumn{2}{c}{FFTS} 
                & \multicolumn{2}{c|}{FedAvg} 
                & \multicolumn{2}{c}{All Mixed} \\
                
        Metric & MSE & MAE   
               & MSE & MAE   
               & MSE & MAE   
               & MSE & MAE   
               & MSE & MAE   
               & MSE & MAE   
               & MSE & MAE  
               & MSE & MAE  \\
        \midrule
        ETTh1 & \textbf{0.448} & \textbf{0.472} & 0.495 & 0.512 & 0.478 & 0.495 & 0.502 & 0.518 & 0.521 & 0.540 & 0.463 & \underline{0.485} & 0.476 & 0.495 & \underline{0.458} & 0.488 \\ 
        
        ETTh2 & \textbf{0.382} & \textbf{0.414} & 0.442 & 0.468 & 0.410 & 0.432 & 0.455 & 0.474 & 0.471 & 0.490 & \underline{0.395} & \underline{0.419} & 0.410 & 0.433 & 0.399 & 0.424 \\
        
        ETTm1 & \underline{0.375} & \underline{0.402} & 0.415 & 0.439 & 0.386 & 0.408 & 0.421 & 0.444 & 0.436 & 0.455 & 0.380 & 0.406 & 0.399 & 0.406 & \textbf{0.372} & \textbf{0.398} \\
        
        ETTm2 & \textbf{0.278} & \textbf{0.330} & 0.324 & 0.352 & 0.296 & 0.340 & 0.330 & 0.360 & 0.345 & 0.369 & 0.290 & 0.338 & 0.299 & 0.341 & \underline{0.286} & \underline{0.336} \\
        
        Electricity & \textbf{0.181} & \textbf{0.272} & 0.216 & 0.306 & 0.192 & 0.288 & 0.225 & 0.315 & 0.242 & 0.329 & \underline{0.187} & \underline{0.280} & 0.198 & 0.293 & 0.189 & 0.285 \\
        
        Traffic  & \underline{0.426} & \underline{0.288} & 0.472 & 0.315 & 0.438 & 0.295 & 0.485 & 0.322 & 0.498 & 0.335 & \textbf{0.420} & \textbf{0.285} & 0.434 & 0.294 & 0.428 & 0.289 \\
        
        Weather & \textbf{0.241} & \textbf{0.282} & 0.285 & 0.310 & 0.260 & 0.295 & 0.295 & 0.312 & 0.305 & 0.322 & \underline{0.252} & \underline{0.290} & 0.264 & 0.299 & 0.254 & 0.292 \\
        
        Exchange & \textbf{0.375} & \textbf{0.423} & 0.460 & 0.480 & 0.405 & 0.445 & 0.468 & 0.490 & 0.472 & 0.492 & \underline{0.390} & 0.440 & 0.403 & 0.444 & \underline{0.390} & \underline{0.438} \\
        \midrule
        \rowcolor[gray]{0.95}
        \bf $1^{\text{st}}$ Count & \multicolumn{2}{c|}{\textbf{12}} & \multicolumn{2}{c}{0} & \multicolumn{2}{c}{0} & \multicolumn{2}{c}{0} & \multicolumn{2}{c|}{0} & \multicolumn{2}{c}{2} & \multicolumn{2}{c|}{0} & \multicolumn{2}{c}{2} \\
        \bottomrule
        \end{tabular}%
        }
        \vspace{-5pt}
\end{table*}
\begin{table*}[tbh]
  \centering
  \caption{Full-shot point forecasting results (averaged across $\{96, 192, 336, 720\}$). \boldres{Bold}: the best; \secondres{Underline}: the second best. Note that FedTRL-trained models are adapted to the target dataset with only one epoch.}
  \resizebox{1\textwidth}{!}{
    \begin{tabular}{l|cc|cc|cc|cc|cc|cc|cc|cc|cc|cc}
    \toprule
    Models & \multicolumn{2}{c|}{\bf FedTRL}  
           & \multicolumn{2}{c|}{TimeMixer} 
           & \multicolumn{2}{c|}{TimeXer} 
           & \multicolumn{2}{c|}{iTransformer} 
           & \multicolumn{2}{c|}{PatchTST} 
           & \multicolumn{2}{c|}{TimesNet} 
           & \multicolumn{2}{c|}{DLinear} 
           & \multicolumn{2}{c|}{Autoformer} 
           & \multicolumn{2}{c|}{Stationary} 
           & \multicolumn{2}{c}{LightTS}  \\
\cmidrule{2-21}    
    Metrics & \multicolumn{1}{c}{MSE} & \multicolumn{1}{c|}{MAE} 
            & \multicolumn{1}{c}{MSE} & \multicolumn{1}{c|}{MAE} 
            & \multicolumn{1}{c}{MSE} & \multicolumn{1}{c|}{MAE} 
            & \multicolumn{1}{c}{MSE} & \multicolumn{1}{c|}{MAE} 
            & \multicolumn{1}{c}{MSE} & \multicolumn{1}{c|}{MAE} 
            & \multicolumn{1}{c}{MSE} & \multicolumn{1}{c|}{MAE} 
            & \multicolumn{1}{c}{MSE} & \multicolumn{1}{c|}{MAE} 
            & \multicolumn{1}{c}{MSE} & \multicolumn{1}{c|}{MAE} 
            & \multicolumn{1}{c}{MSE} & \multicolumn{1}{c|}{MAE} 
            & \multicolumn{1}{c}{MSE} & \multicolumn{1}{c}{MAE}   \\
    \midrule
    ETTh1 & \boldres{0.372} & \boldres{0.399} & 0.448 & 0.443 & 0.491 & 0.489 & 0.474 & 0.476 & \secondres{0.413} & \secondres{0.430} & 0.458 & 0.450 & 0.422 & 0.437 & 0.496 & 0.487 & 0.570 & 0.537 & 0.491 & 0.479 \\
    
    ETTh2 & \boldres{0.330} & \secondres{0.383} & 0.364 & 0.394 & \secondres{0.356} & 0.398 & 0.372 & 0.411 & \boldres{0.330} & \boldres{0.379} & 0.414 & 0.427 & 0.431 & 0.446 & 0.450 & 0.459 & 0.526 & 0.516 & 0.602 & 0.543 \\
    
    ETTm1 & \boldres{0.316} & \boldres{0.369} & 0.381 & 0.395 & 0.381 & 0.402 & 0.368 & 0.396 & \secondres{0.351} & 0.380 & 0.400 & 0.406 & 0.357 & \secondres{0.378} & 0.588 & 0.517 & 0.481 & 0.456 & 0.435 & 0.437 \\
    
    ETTm2 & \secondres{0.257} & \boldres{0.313} & 0.275 & 0.323 & 0.275 & 0.329 & 0.271 & 0.331 & \boldres{0.255} & \secondres{0.315} & 0.291 & 0.333 & 0.267 & 0.333 & 0.327 & 0.371 & 0.306 & 0.347 & 0.409 & 0.436 \\

    Electricity & \boldres{0.143} & \boldres{0.247} & 0.175 & 0.272 & 0.171 & 0.273 & \secondres{0.161} & 0.256 & \secondres{0.161} & \secondres{0.252} & 0.192 & 0.295 & 0.188 & 0.295 & 0.227 & 0.338 & 0.193 & 0.296 & 0.229 & 0.329 \\

    Traffic & \boldres{0.355} & \boldres{0.256} & 0.405 & 0.284 & 0.419 & 0.302 & \secondres{0.385} & 0.275 & 0.390 & \secondres{0.263} & 0.620 & 0.336 & 0.433 & 0.295 & 0.628 & 0.379 & 0.624 & 0.340 & 0.622 & 0.392 \\

    Weather & \boldres{0.214} & \boldres{0.252} & 0.241 & 0.272 & \secondres{0.227} & 0.268 & 0.235 & 0.274 & 0.255 & \secondres{0.264} & 0.259 & 0.287 & 0.248 & 0.300 & 0.338 & 0.382 & 0.288 & 0.314 & 0.261 & 0.312 \\
    
    Exchange & \boldres{0.372} & \boldres{0.390} & 0.428 & 0.445 & 0.474 & 0.466 & 0.442 & 0.457 & \secondres{0.400} & \secondres{0.431} & 0.674 & 0.591 & 0.556 & 0.542 & 0.762 & 0.686 & 0.954 & 0.669 & 0.704 & 0.616 \\
    \midrule
     \rowcolor[gray]{0.95}
    \bf $1^{\text{st}}$ Count & \multicolumn{2}{c|}{\textbf{12}} & \multicolumn{2}{c|}{0} & \multicolumn{2}{c|}{0} & \multicolumn{2}{c|}{0} & \multicolumn{2}{c|}{\secondres{3}} & \multicolumn{2}{c|}{0} & \multicolumn{2}{c|}{0} & \multicolumn{2}{c|}{0} & \multicolumn{2}{c|}{0} & \multicolumn{2}{c}{0} \\
    \bottomrule
    \end{tabular}}
  \label{tab:cross_domain}
  \vspace{-4pt}
\end{table*}
\section{Main Results}\label{sec:main_res}
We evaluate FedTRL across multiple forecasting settings (details in \textbf{Appendix~\ref{app:baseline_bench}}). First, we compare it with advanced unsupervised representation learning methods on in-domain forecasting (Sec.~\ref{subsec:in-domain}). We then assess its effectiveness in training large time series forecasting models (Sec.~\ref{subsec:large_model}), as well as its full-shot adaption and zero-shot generalization (Sec.~\ref{subsec:zero_point} and Sec.~\ref{subsec:zero_pro}) on widely recognized benchmarks and real-world weather station datasets. Finally, we analyze FedTRL`s benefits for unsupervised pretraining, conduct ablation and hyperparameter studies, and examine its scalability (Sec.~\ref{subsec:analysis}). FedTRL proves that federated pretraining can scale into large TSFMs, achieving competitive deterministic point or probabilistic forecasting performance while even surpassing advanced centralized models. Our code for review is available at \url{https://anonymous.4open.science/r/FedTRL-Review-7BDA}.

\subsection{In-domain Point Forecasting}\label{subsec:in-domain}
\paragraph{Setup.} We conduct experiments on eight benchmark datasets from TSLib~\cite{wu2022timesnet}. Each dataset is treated as an independent client during federated training (without probabilistic head). For fine-tuning, we adapt the pretrained model to the corresponding target datasets. The look-back window is fixed at 512, and forecasting performance is evaluated across four horizons \{96,192,336,720\} using MSE and MAE. Model architectures and training configurations are summarized in \textbf{Table~\ref{tab:configuration}}.

\paragraph{Results.} We compare FedTRL with three categories of baselines: (i) Masked reconstruction-based methods, including SimMTM~\cite{dong2023simmtm} and PatchTST~\cite{nie2022time}; (ii) Contrastive Learning-based methods, including TS-TCC~\cite{eldele2021time} and CoST~\cite{woo2022cost}; (iii) FL methods, including FFTS~\cite{chen2025federated} and FedAvg \citep{mcmahan2017communication}. Each client perform TimeDART~\citep{wang2025timedartdiffusionautoregressivetransformer} in FedAvg. Across eight datasets, FedTRL achieves consistent improvements over both federated and centralized baselines. It not only surpasses existing federated pretraining approaches by clear margins, but also outperforms centralized training on several datasets, showing that properly aligned and aggregated FL can exceed centralized solutions. Its notable gains over FedAvg highlight robustness under heterogeneous clients.

\subsection{Scaling to Time Series Foundation Models}\label{subsec:large_model}
FedTRL can be scaled to learn a large TSFMs from cross-domain datasets, thereby enabling strong generalization capabilities on unseen data. To evaluate this capability, we pretrain a model on Time-MoE-300B~\cite{shi2024time}, which spans nine domains and contains over 300 billion time points. Each domain is treated as an independent client, allowing FedTRL to learn cross-domain patterns in a federated setting. Note that all evaluation datasets were excluded from the pre-training dataset to ensure fairness. Model and training setups for large-scale pretraining are in \textbf{Table~\ref{tab:configuration}}. 

\begin{table*}[tbh]
  \centering
  \caption{Zero-shot point forecasting results, averaged over horizons ${96,192,336,720}$ with look-back lengths ${512,1024,2048,3072}$. Bold and underline indicate the best and second-best results; “–” denotes not evaluable. TSLib results are from~\citep{shi2024time}; full RW-Bench results are in \textbf{Table~\ref{tab:zero_shot_full_rw}}, with VisionTS~\cite{chen2024visionts} and Sundial~\cite{liu2025sundial} comparisons in \textbf{Table~\ref{tab:R6_zero_shot_forecasting}}.}
  \setlength{\tabcolsep}{1.4pt}
  \setlength{\extrarowheight}{3pt}
  \resizebox{1\textwidth}{!}{
    \begin{tabular}{l|cc|cc|cc|cc|cc|cc|cc|cc|cc|cc|cc|cc}
    \toprule
    \bf Models & \multicolumn{2}{c|}{\bf FedTRL} & \multicolumn{2}{c|}{\bf Time-MoE$_{b}$} & \multicolumn{2}{c|}{\bf Time-MoE$_{l}$} & \multicolumn{2}{c|}{\bf Time-MoE$_{u}$} &  \multicolumn{2}{c|}{\bf Moirai$_{s}$} & \multicolumn{2}{c|}{\bf Moirai$_{b}$} & \multicolumn{2}{c|}{\bf Moirai$_{l}$} & \multicolumn{2}{c|}{\bf TimesFM} & \multicolumn{2}{c|}{\bf Moment} & \multicolumn{2}{c|}{\bf Chronos$_{s}$} & \multicolumn{2}{c|}{\bf Chronos$_{b}$} & \multicolumn{2}{c}{\bf Chronos$_{l}$} \\
    \cmidrule(lr){1-1} \cmidrule(lr){2-3} \cmidrule(lr){4-9} \cmidrule(lr){10-15} \cmidrule(lr){16-17} \cmidrule(lr){18-19} \cmidrule(lr){20-25}
    Metrics & MSE   & MAE   & MSE   & MAE  & MSE   & MAE   & MSE   & MAE   & MSE   & MAE   & MSE   & MAE   & MSE   & MAE   & MSE   & MAE   & MSE   & MAE   & MSE   & MAE   & MSE   & MAE & MSE   & MAE  \\
    \midrule
    ETTh1 & \secondres{0.399} & \secondres{0.420} & 0.400 & 0.424 & \boldres{0.394} & \boldres{0.419} & 0.412 & 0.426 & 0.428 & 0.427 & 0.417 & \boldres{0.419} & 0.480 & 0.439 & 0.473 & 0.443 & 0.683 & 0.566 & 0.545 & 0.472 & 0.591 & 0.468 & 0.588 & 0.466 \\
    
    ETTh2 & \boldres{0.350} & \secondres{0.380} & 0.366 & 0.404 & 0.405 & 0.415 & 0.371 & 0.399 & \secondres{0.361} & 0.384 & 0.362 & 0.382 & 0.367 & \boldres{0.377} & 0.392 & 0.406 & \secondres{0.361} & 0.409 & 0.424 & 0.430 & 0.405 & 0.410 & 0.455 & 0.427 \\

    ETTm1 & \boldres{0.349} & \boldres{0.383} & 0.394 & 0.415 & 0.376 & 0.405 & \secondres{0.356} & 0.391 & 0.436 & 0.410 & 0.406 & \secondres{0.385} & 0.422 & 0.391 & 0.433 & 0.418 & 0.670 & 0.536 & 0.640 & 0.499 & 0.645 & 0.500 & 0.555 & 0.465 \\
    
    ETTm2 & \boldres{0.282} & \boldres{0.330} & 0.317 & 0.365 & 0.316 & 0.361 & \secondres{0.288} & 0.344 & 0.307 & 0.347 & 0.311 & \secondres{0.337} & 0.329 & 0.343 & 0.328 & 0.346 & 0.316 & 0.365 & 0.349 & 0.380 & 0.310 & 0.350 & 0.295 & 0.338 \\
    
    Weather & \boldres{0.238} & \secondres{0.279} & 0.265 & 0.297 & 0.270 & 0.300 & 0.270 & 0.300 & 0.275 & 0.286 & 0.287 & 0.281 & \secondres{0.264} & \boldres{0.273} & -- & -- &0.294 & 0.326 & 0.300 & 0.318 & 0.292 & 0.315 & 0.279 & 0.306 \\
    
    \midrule
    RW-Bench &  \bf 0.599 & \bf 0.519 & 0.621 & 0.536 & \secondres{0.608} & \secondres{0.525} & -- & -- & 1.328 & 0.627 & 1.029 & 0.598 & 1.244 & 0.606 & 0.991 & 0.747 & 1.091 & 0.756 & 1.097 & 0.683 & 1.156 & 0.697 & 1.117  &  0.683 \\
    \midrule
    \rowcolor[gray]{0.95}
    \bf 1$^{st}$ Count &
    \multicolumn{2}{c|}{\boldres{8}} &
    \multicolumn{2}{c|}{1} &
    \multicolumn{2}{c|}{\secondres{2}} &
    \multicolumn{2}{c|}{1} &
    \multicolumn{2}{c|}{0} &
    \multicolumn{2}{c|}{1} &
    \multicolumn{2}{c|}{\secondres{2}} &
    \multicolumn{2}{c|}{0} &
    \multicolumn{2}{c|}{1} &
    \multicolumn{2}{c|}{0} &
    \multicolumn{2}{c|}{1} &
    \multicolumn{2}{c}{0} \\
    \bottomrule
    \end{tabular}}
    \label{tab:zero_shot}%
    \vspace{-3pt}
\end{table*}%

\begin{table*}[tbh]
    \caption{Zero-shot results on GIFT-eval, which contains 32 datasets characterized by a variety of frequencies/variate numbers/prediction lengths. We evaluate performance using $100$ generated series, being consistent with~\citep{woo2024unifiedtraininguniversaltime}. Lower MASE/CRPS indicates a better performance. Rank assigns a numerical ranking of all $97$ configurations. Baseline results are officially from~\citep{aksu2024gift}.}
    \label{tab:gift_eval}
    \centering
    \renewcommand{\multirowsetup}{\centering}
    \resizebox{1\textwidth}{!}{
    \begin{tabular}{l|cccccccccccccccc}
    \toprule
    \textbf{Type} & \multicolumn{4}{c}{\textbf{Statistical Methods}} & \multicolumn{5}{c}{\textbf{Supervised Task-Specific Models}} & \multicolumn{4}{c}{\textbf{Zero-Shot Models}} & \multicolumn{3}{c}{\textbf{Zero-Shot Models (FL)}} \\
    \cmidrule(lr){1-1} \cmidrule(lr){2-5} \cmidrule(lr){6-10} \cmidrule(lr){11-14} \cmidrule(lr){15-17}
        \multirow{2}{*}{\textbf{Model}}
            & \multirow{2}{*}{{Na\"ive}} & {Seasonal} & {Auto} & {Auto} & \multirow{2}{*}{{DeepAR}} & \multirow{2}{*}{{TiDE}} & \multirow{2}{*}{{N-BEATS}} & \multirow{2}{*}{{PTST.}} & \multirow{2}{*}{{iTrans.}} & \multirow{2}{*}{{TimesFM}} & \multirow{2}{*}{{TabPFN}} & \multirow{2}{*}{{Chronos}} & \multirow{2}{*}{{Moirai}} & \multirow{2}{*}{FedAvg} & \multirow{2}{*}{FFTS} & \textbf{FedTRL} \\
            & & {Na\"ive} & {ARIMA} & {Theta} &  &  &  &  &  &  &  &  & &  & &  \textbf{(Ours)} \\
        \midrule
            \textbf{MASE} & 1.260 & 1.000 & 0.964 & 0.978 & 1.206  & 0.980 & 0.842 & 0.762 & 0.802 & \secondres{0.680} & 0.748 & 0.786 &  0.809 &  0.872 &  0.766 &\boldres{0.675}\\
            \textbf{CRPS}  & 1.383 & 1.000 & 0.770 & 1.051 & 0.721 & 0.652 & 0.689 & 0.496 & 0.524 & \secondres{0.465} & 0.480 & 0.551 &  0.515 & 0.700 & 0.521 &\boldres{0.460}\\
            \textbf{Rank}  & 28.072  & 26.175 & 21.515 & 24.031 & 18.938 & 18.557 & 21.381 & 10.052 & 11.320 & \boldres{8.237} & \secondres{8.268} & 14.309 & 10.175 & 18.231 & 9.787 & 8.676 \\
        \bottomrule
    \end{tabular}}
    \vspace{-2pt}
\end{table*}
\subsubsection{Full/Zero-shot Point Forecasting}\label{subsec:zero_point} 
We evaluate FedTRL on widely recognized forecasting benchmarks. Specifically, we adopt TSFLib~\cite{wu2022timesnet} for evaluating point forecasting. The ECL is excluded to avoid data leakage. Beyond standard benchmarks, we introduce RW-Bench, a collection of real-world weather datasets from 15 geographically diverse regions. Unlike curated academic corpora, RW-Bench retains  noisy observations, providing a realistic testbed for zero-shot forecasting. Covering 10 variables at hourly resolution, it offers a timely (by Aug 2025) and challenging benchmark to evaluate TSFM`s applicability under real-world conditions. Dataset statistics are provided in \textbf{Appendix~\ref{app:rw-bench}}.

\paragraph{Full-shot Adaption.} We compare FedTRL against several advanced deep time series forecasting models, including TimeMixer~\cite{wang2024timemixer}, TimeXer~\cite{wang2024timexer}, PatchTST~\cite{nie2022time}, TimesNet~\cite{wu2022timesnet}, DLinear~\cite{zeng2023transformers}, iTransformer~\cite{liu2023itransformer}, Autoformer~\cite{wu2021autoformer}, Non-Stationary Transformer~\cite{liu2022non}, and LightTS~\cite{zhang2022less}, as shown in \textbf{Table~\ref{tab:cross_domain}}. FedTRL-trained model outperforms all baselines, showing notable advantages in both complex and simple domains despite requiring only 1 adaptation epoch.

\paragraph{Zero-shot Inference.} We evaluate the FedTRL-trained TSFM on TSLib and RW-Bench, following the evaluation protocol of~\cite{shi2024time}. \textbf{Table~\ref{tab:zero_shot}} shows that FedTRL consistently achieves state-of-the-art performance, delivering consistent MSE reductions across 20 datasets (5 in TSLib and 15 in RW-Bench). In particular, it demonstrates strong robustness on complex real-world weather station data, where temporal variability and distribution shifts are most pronounced. Notably, FedTRL-trained models outperform Time-MoE across all three reported scales, despite using fewer parameters than Time-MoE$_\text{ultra}$ (details about models are provided in \textbf{Table~\ref{tab:scale_config}}). They also surpass TimesFM and Moment, which require additional fine-tuning for variable-length adaptation, whereas FedTRL operates fully zero-shot. This highlights that our optimization–aggregation design enables efficient pretraining of large time series foundation models under federated settings, producing representations that are both generalizable and well-calibrated across diverse domains. These suggest that FedTRL not only matches but in some cases exceeds centralized paradigms, demonstrating the viability of federated pretraining as a powerful approach for TSFM training.

\subsubsection{Zero-shot Probabilistic Forecasting}\label{subsec:zero_pro} 
Beyond point forecasting, probabilistic forecasting is important to real-world applications with uncertainty. We experiment on GIFT-Eval~\cite{aksu2024gift} and FEV leaderboard~\cite{ansari2024chronos}, following their official evaluation suite and assessing point (MASE) and probabilistic (CPRS and WQL) metrics. Note that all evaluated datasets are excluded from the pre-training dataset.

\begin{figure*}[tbh]
    \centering
    \includegraphics[width=.95\textwidth]{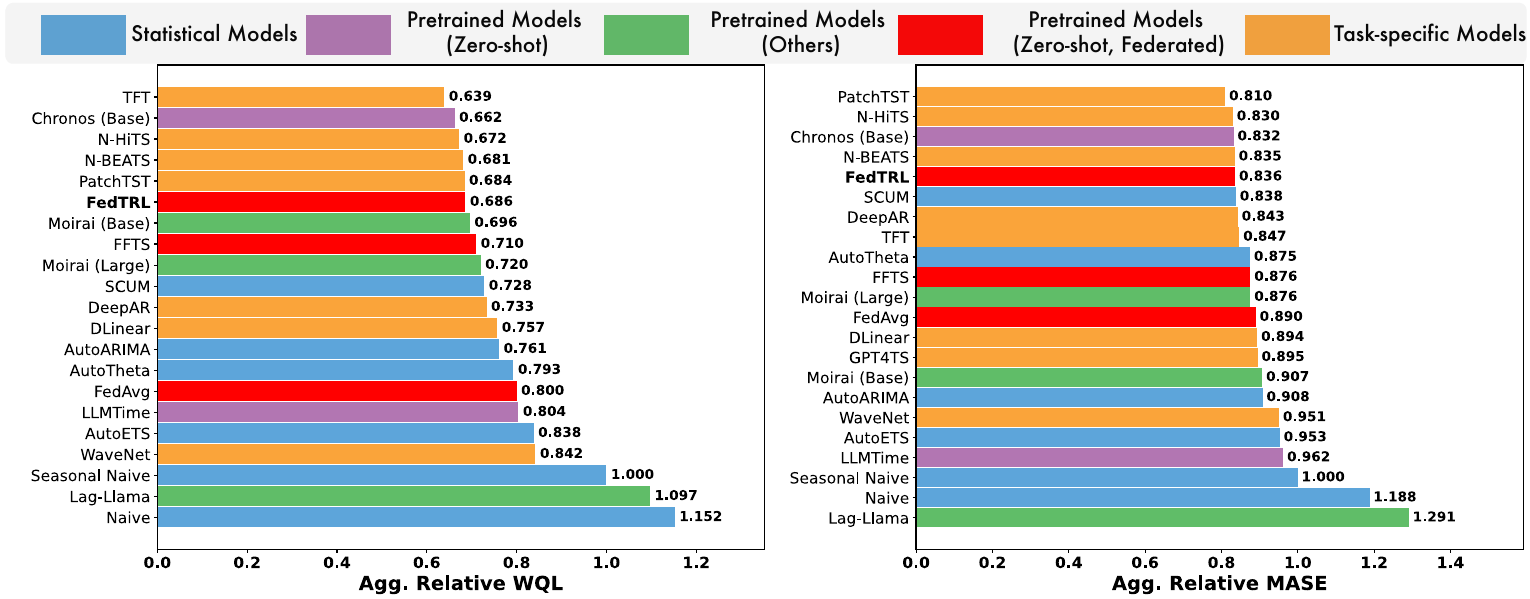}
    \caption{Results on FEV leaderboard. Baseline includes statistical methods, task-specific deep models trained on each dataset, and pre-trained foundation models. Pre-trained Models that have seen several datasets during pre-training are denoted as Pre-trained Models (Other). Lower MASE/WQL is better. FedTRL reports probabilistic forecasts using 20 generated series, following~\cite{ansari2024chronos}.}
    \label{fig:fev_eval}
    \vspace{-4pt}
\end{figure*}
\paragraph{GIFT-eval.} The results on GIFT-eval are shown in \textbf{Table~\ref{tab:gift_eval}}. This benchmark covers 23 datasets and 15 baselines, including statistical models, task-specific methods (DeepAR~\cite{salinas2020deepar}, TiDE~\cite{das2023long}, N-BEATS~\cite{oreshkin2020nbeatsneuralbasisexpansion}, PatchTST~\cite{nie2022time}, iTransformer~\cite{liu2023itransformer}), time series FMs (TimesFM~\cite{das2024decoder}, TabPFN~\cite{hoo2024tabular}, Chronos~\cite{ansari2024chronos}), and federated FMs (FedAvg~\cite{mcmahan2017communication}, FFTS~\cite{chen2025federated}). Across all unseen datasets, FedTRL achieves the best MASE/CRPS, and even surpasses most advanced pretrained foundation models, highlighting its effectiveness for large-scale cross-domain pretraining.

\paragraph{FEV Leaderboard.} We evaluate FedTRL-trained model on FEV leaderboard established by AutoGluon~\cite{ansari2024chronos}, which covers 27 datasets for probabilistic forecasting. As shown in \textbf{Fig.~\ref{fig:fev_eval}}, our model achieves zero-shot forecasting performance that surpasses most statistical and task-specific models trained with in-domain supervision. Among pretrained FMs, FedTRL ranks second overall (outperforming Chronos) and ranks first among federated FMs (outperforming FFTS/FedAvg). These improvements benefit from the diffusion-based reconstruction strategy, which integrates both patch- and point-wise reconstruction to capture fine-grained temporal dynamics. This enriched representation is further complemented by the dual-head design with a scheduling strategy that first secures accurate point forecasts before introducing uncertainty modeling. Together with DaG aggregation, these components allow FedTRL to scale federated learning into an effective training framework for TSFMs in both point and probabilistic forecasting.

\subsection{Framework Analysis}\label{subsec:analysis}
\paragraph{Ablation Study.} We evaluate the contribution of three key components: (i) GRL-based adversarial domain regularization, (ii) prototype alignment, and (iii) domain-aware aggregation (DaG). Variants include \emph{w/o} GRL, \emph{w/o} Proto, and \emph{w/o} DaG (replaced with vanilla FedAvg). Results on point forecasting is reported in \textbf{Table~\ref{tab:abla_point}}, with probabilistic forecasting degradation shown in \textbf{Fig.~\ref{fig:abla_pro}}. Across both settings, removing any component consistently degrades performance. GRL stabilizes adversarial training and enforces subdomain invariance, yielding the largest gains in probabilistic metrics. Prototype alignment is crucial for bridging client-level representation gaps, especially in zero-shot transfer, while replacing DaG with FedAvg causes consistent performance drops, underscoring its importance.

\begin{table}[tbh]
    \centering
    \caption{Ablation study on point forecasting tasks.}
    \resizebox{.95\columnwidth}{!}{
    \begin{tabular}{c|ccccccc}
    \toprule
         \multicolumn{2}{c}{\textbf{Tasks}}
          & \multicolumn{2}{c}{\textbf{In-domain}}
          & \multicolumn{2}{c}{\textbf{Full-shot}}
          & \multicolumn{2}{c}{\textbf{Zero-shot}} \\
        \cmidrule(lr){3-4}\cmidrule(lr){5-6}\cmidrule(lr){7-8}
        \multicolumn{2}{c}{\textbf{Variants}} & MSE & MAE & MSE & MAE & MSE & MAE \\
        \midrule
        \multirow{4}{*}{\rotatebox{90}{ETTh1}}
        & \cellcolor{gray!15} \bf FedTRL         & \cellcolor{gray!15} \textbf{0.448} & \cellcolor{gray!15} \textbf{0.472} & \cellcolor{gray!15} \bf 0.372 & \cellcolor{gray!15} \bf 0.399 & \cellcolor{gray!15} \bf 0.399 &\cellcolor{gray!15}  \bf 0.420 \\
        & \emph{w/o} GRL & 0.464 & 0.486 & 0.377 & 0.400 & 0.407 & 0.435 \\
        & \emph{w/o} Proto & 0.458 & 0.482 & 0.382 & 0.410 & 0.417 & 0.432 \\
        & \emph{w/o} DaG  & 0.472 & 0.494 & 0.394 & 0.421 & 0.436 & 0.459 \\
        \midrule
        \multirow{4}{*}{\rotatebox{90}{ETTm1}}
        & \cellcolor{gray!15} \bf FedTRL         & \cellcolor{gray!15} \textbf{0.375} & \cellcolor{gray!15} \textbf{0.402} & \cellcolor{gray!15} \bf 0.316 & \cellcolor{gray!15} \bf 0.369 & \cellcolor{gray!15} \bf 0.350 & \cellcolor{gray!15} \bf 0.380 \\
        & \emph{w/o} GRL & 0.388 & 0.414 & 0.324 & 0.372 & 0.359 & 0.384 \\
        & \emph{w/o} Proto & 0.382 & 0.408 & 0.320 & 0.372 & 0.361 & 0.384 \\
        & \emph{w/o} DaG  & 0.400 & 0.420 & 0.333 & 0.364 & 0.380 & 0.399 \\
        \midrule
        \multirow{4}{*}{\rotatebox{90}{Weather}}
        & \cellcolor{gray!15} \bf FedTRL         & \cellcolor{gray!15} \textbf{0.241} & \cellcolor{gray!15} \textbf{0.282} & \cellcolor{gray!15} \bf 0.214 & \cellcolor{gray!15} \bf 0.252 & \cellcolor{gray!15} \bf 0.238 & \cellcolor{gray!15} \bf 0.279 \\
        & \emph{w/o} GRL & 0.260 & 0.295 & 0.217 & 0.257 & 0.240 & 0.287 \\
        & \emph{w/o} Proto & 0.252 & 0.290 & 0.222 & 0.259 & 0.251 & 0.282 \\
        & \emph{w/o} DaG  & 0.254 & 0.292 & 0.239 & 0.275 & 0.289 & 0.292 \\
        \bottomrule
    \end{tabular}}
    \label{tab:abla_point}
    \vspace{-4pt}
\end{table}

\begin{figure}[H]
    \centering
    \includegraphics[width=1\columnwidth]{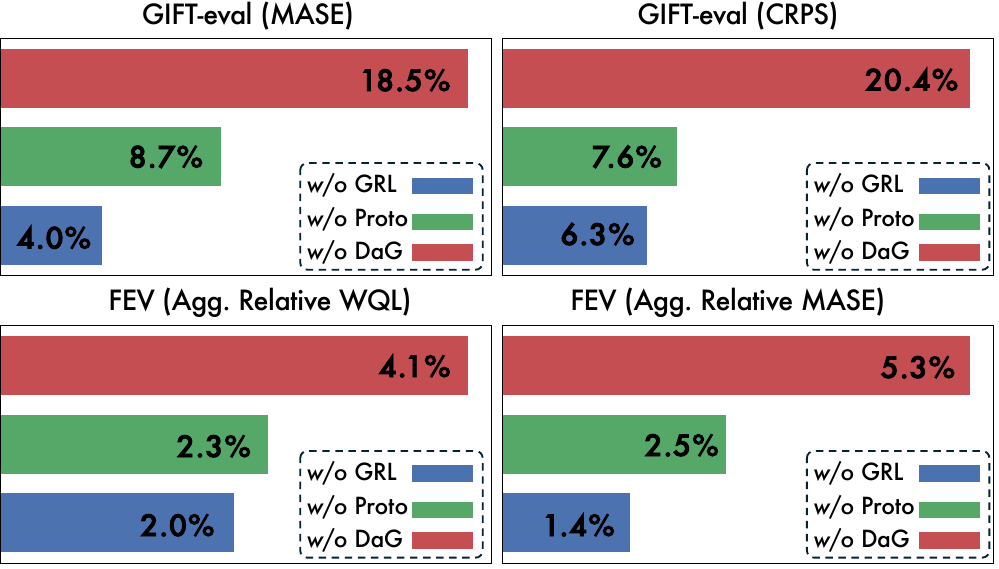}
    \caption{Relative performance drop of FedTRL ablations in zero-shot probabilistic forecasting on GIFT-eval and FEV leaderboard.}
    \label{fig:abla_pro}
    \vspace{-4pt}
\end{figure}

\paragraph{Can FedTRL Improves Masked Reconstruction-based Methods?} To evaluate the generality of our proposed FedTRL, we integrate its key components into other masked reconstruction-based representation methods, including SimMTM and PatchTST. The results are reported in \textbf{Table~\ref{tab:fedtrl_etts}}. We observe that incorporating FedTRL consistently improves performance across four long-term forecasting datasets. These demonstrate that FedTRL can seamlessly enhance different reconstruction-based representation methods, highlighting its effectiveness and broad applicability.

\begin{table}[tbh]
  \centering
  \caption{\small Results on ETT, Weather and Traffic (averaging across four horizons). Improvement (\%) is relative to the baseline.}
  \resizebox{\columnwidth}{!}{
  \begin{tabular}{l|cc|cc|cc|cc}
    \toprule
    \multirow{2}{*}{\bf Model} & \multicolumn{2}{c|}{\bf ETTh1} & \multicolumn{2}{c|}{\bf ETTm1} & 
                             \multicolumn{2}{c|}{\bf Weather} & \multicolumn{2}{c}{\bf Traffic} \\
    \cmidrule(lr){2-9}
     & MSE & MAE & MSE & MAE & MSE & MAE & MSE & MAE \\
    \midrule
    SimMTM$^\dagger$ 
    & 0.495 & 0.512 & 0.415 & 0.439 & 0.285 & 0.310 & 0.472 & 0.315   \\
    \rowcolor{gray!15}
    \quad + FedTRL 
      & 0.476 & 0.501 & 0.401 & 0.430 & 0.263 & 0.301 & 0.459 & 0.305 \\
    \quad \bf Improvement (\%) 
      & \bf 3.84 & \bf 2.15 & \bf 3.37 & \bf 2.05 & \bf 7.88 & \bf 2.48 & \bf 2.82 & \bf 3.13 \\
    \midrule
    PatchTST$^\dagger$
      & 0.478 & 0.495 & 0.386 & 0.408 & 0.260 & 0.295 & 0.438 & 0.295 \\
    \rowcolor{gray!15}
    \quad + FedTRL 
      & 0.461 & 0.487 & 0.378 & 0.404 & 0.252 & 0.288 & 0.430 & 0.291  \\
    \quad \bf Improvement (\%) 
      & \bf 3.56 & \bf 1.62 & \bf 2.07 & \bf 0.98 & \bf 3.08 & \bf 2.37 & \bf 1.83 & \bf 1.36 \\
    \bottomrule
  \end{tabular}}
  \label{tab:fedtrl_etts}
  \vspace{-6pt}
\end{table}

\paragraph{Model Scalability.} Scalability is a key property for large time series foundation models. We evaluate the proposed FedTRL across different backbone sizes (from 38M to 302M, more structure details are provided in \textbf{Table~\ref{tab:scale_config}}) to assess its scalability. Results (\textbf{Fig.~\ref{fig:scale}}) consistently show that FedTRL maintains strong performance improvements on both point and probabilistic forecasting benchmarks as the backbone grows, demonstrating its robustness as a scalable training paradigm for models of varying capacity.

\begin{figure}[tbh]
      \centering
      \includegraphics[width=1\columnwidth]{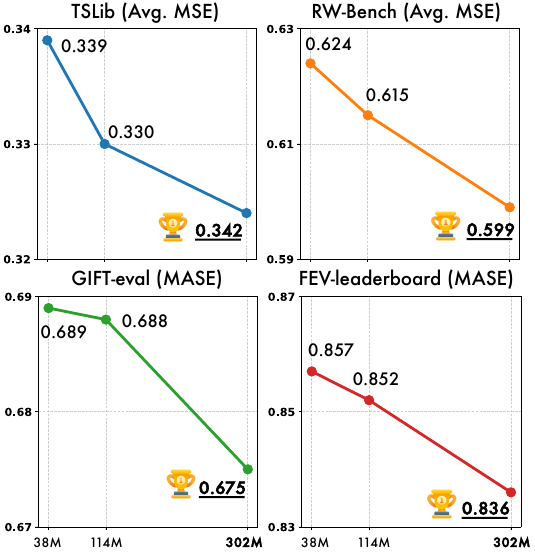}
      \caption{Model scalability across zero-shot both point forecasting and probabilistic forecasting benchmarks.}
      \label{fig:scale}
      \vspace{-2pt}
\end{figure}

\paragraph{Hyperparameter Sensitivity.} We evaluate hyperparameter sensitivity ($\alpha$, $\beta$, $\lambda_{\text{dom}}\&\lambda_{\text{align}}$, $\lambda$) in \textbf{Fig.~\ref{fig:sensi}}, showing consistent trends across zero-shot point forecasting tasks. First, reducing $\alpha$ will gradually degrades performance, highlighting the importance of strong domain discriminability for aggregation. Second, $\beta$ exhibits a U-shaped trend: too small under-utilizes alignment, while too large over-regularizes, confirming that moderate values balance local and global consistency. Third, $\lambda_{\text{dom}},\lambda_{\text{align}}$ are relatively stable within $[0.1,0.3]$, but overly small values harm invariance learning. Finally, $\lambda$ also follows a concave trend, where weak reversal fails to suppress domain bias and strong reversal over-penalizes representations. These show FedTRL is robust across a range of settings, with default parameters achieving a good balance between discriminability and invariance.

\begin{figure}[tbh]
      \centering
      \includegraphics[width=1\columnwidth]{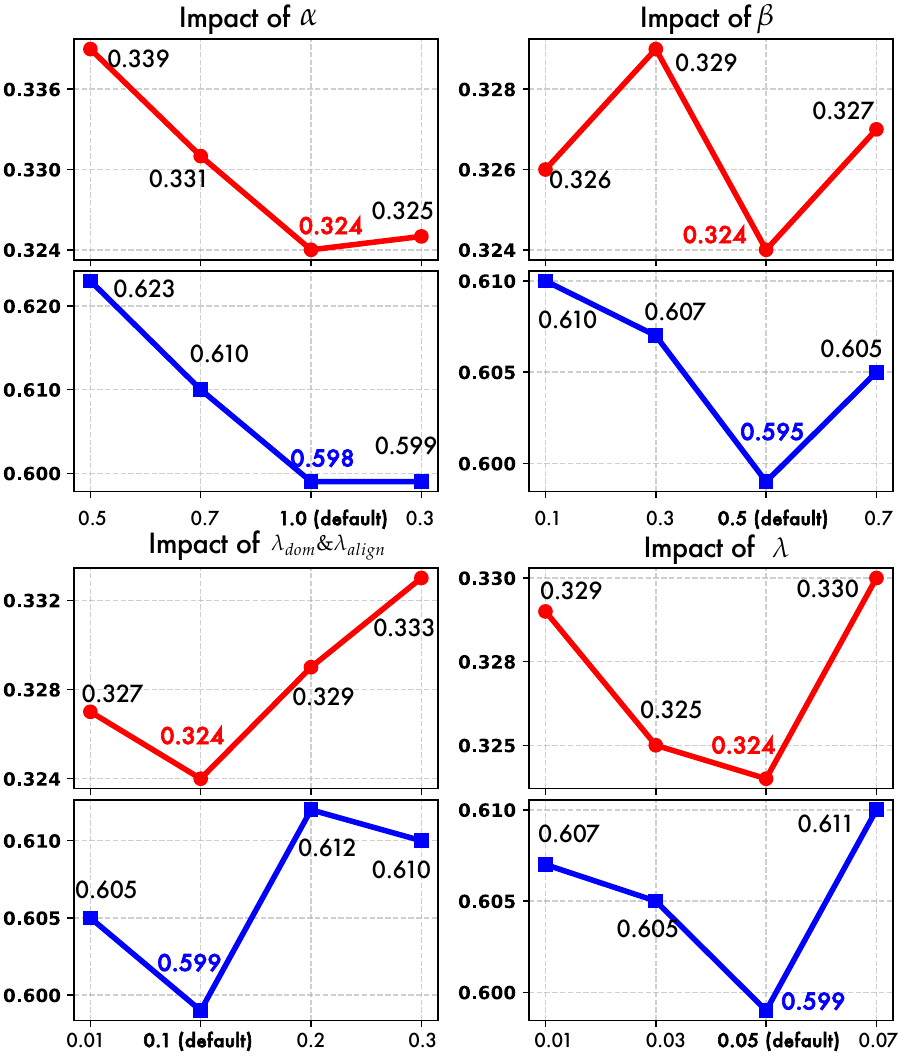}
      \caption{Hyperparameter sensitivity results (Avg. MSE). \textcolor{red}{\bf Red Line}: TSLib results; \textcolor{blue}{\bf Blue Line}: RW-Bench results.}
      \label{fig:sensi}
      \vspace{-2pt}
\end{figure}

\section{Conclusion}
We proposed FedTRL, a federated bi-level heterogeneous learning that tackles both inter- and intra-domain variability in TSFM pretraining. By unifying domain-invariant local optimization with domain-aware aggregation, FedTRL learns patterns that remain consistent within sub-domains while generalizing effectively across domains. Experiments on multiple benchmarks covering point \& probabilistic forecasting show that TSFMs trained with FedTRL not only outperform centralized and FL baselines, but also scale reliably to large sizes, achieving strong zero-shot generalization.
\clearpage

\section*{Impact Statement}
This work proposes a federated pretraining framework for time series foundation models that explicitly handles multi-level data heterogeneity, enabling scalable pretraining without centralized data aggregation. By improving robustness and generalization across diverse and unseen domains, the approach facilitates the development of foundation models applicable to real-world time series settings like energy, climate, transportation, and healthcare. The contributions are methodological and aim to support reliable and general-purpose temporal modeling under realistic data constraints.


\bibliography{example_paper}
\bibliographystyle{icml2026}

\newpage
\appendix
\onecolumn

\paragraph{Roadmap.}
This appendix provides supplementary materials omitted from the main text, together with extended experimental results, theoretical analyses, and additional discussions. Specifically:

\begin{enumerate}
    \item \textbf{Appendix~\ref{sec:more_related}} presents extended related work, covering large-scale \emph{time series foundation models} (Appendix~\ref{sec:tsfm}) and \emph{heterogeneous federated learning} (Appendix~\ref{sec:hfl}).

    \item \textbf{Appendix~\ref{sec:supp}} provides supplementary methodological details, including the formulation of the \emph{Gradient Reversal Layer} (Appendix~\ref{sec:grl}), the design of \emph{sub-domain and global domain classifiers} (Appendix~\ref{sec:subdomain_domain_classifier}), and the complete \emph{workflow of FedTRL} (Appendix~\ref{sec_app:workflow_of_fedtrl}).

    \item \textbf{Appendix~\ref{appendix:theory}} presents formal definitions and theoretical analyses of key concepts in FedTRL, together with proofs establishing domain invariance, domain awareness, and global dynamics.

    \item \textbf{Appendix~\ref{app:baseline_bench}} details the experimental settings, datasets, and baseline implementations used across all forecasting tasks.

    \item \textbf{Appendix~\ref{sec_app:discussion}} reports additional experimental results and in-depth analyses, including model scale comparisons (Appendix~\ref{sec_app:model_scale}), evaluations of federated baselines under heterogeneity (Appendix~\ref{sec_app:fed}), extensive forecasting results (Appendices~\ref{sec_app:forecast}--\ref{sec_app:zero_shot_prob}), scalability studies (Appendix~\ref{sec_app:scalability}), robustness analyses, and further discussions on generalization and representation behavior (Appendix~\ref{sec_app:additional}).
    
    \item \textbf{Appendix~\ref{sec_app:showcases}} provides showcase during pretraining and forecasting.
\end{enumerate}

\section{More Related Work}\label{sec:more_related}
\subsection{Time Series Foundation Models}\label{sec:tsfm}
Pre-trained models that scale in capacity and data can evolve into foundation models capable of transferring across diverse domains. Moving beyond dataset-specific architectures, recent efforts have focused on building large time series foundation models (TSFMs) via cross-domain pretraining for zero-shot inference. Early examples include ForecastFPN~\cite{dooley2024forecastpfn} using synthetic data, CloudOps~\cite{woo2023pushing} with masked modeling for domain-specific prediction, and TimeGPT-1~\cite{garza2023timegpt} as the first commercial API. More recent advances such as MOMENT~\cite{goswami2024moment}, Moirai~\cite{woo2024unifiedtraininguniversaltime}, TimeFM~\cite{das2024decoder}, Chronos~\cite{ansari2024chronos}, Timer~\cite{liu2024timer}, and Time-MoE~\cite{shi2024time} scale pretraining to ultra-large datasets, achieving strong cross-task transfer. However, most TSFMs are trained under centralized paradigms~\cite{shi2024time,das2024decoder,woo2024unifiedtraininguniversaltime,chen2023foundation}, where heterogeneous datasets are simply merged into mixed batches. This practice amplifies gradient conflicts and obscures domain-specific structure, leading to suboptimal representations. Temporal data are also fragmented across organizations and applications, making centralized collection impractical. These challenges call for flexible training strategies that explicitly address bi-level heterogeneity, spanning both inter-domain misalignment and intra-domain variability. Our work therefore introduces a federated learning framework for bi-level heterogeneous learning, enabling scalable pretraining of TSFMs.

\subsection{Heterogeneous Federated Learning}\label{sec:hfl}
FL~\cite{mcmahan2017communication,yan2025federated} enables decentralized model training across distributed clients without sharing raw data. A key challenge is statistical heterogeneity, which degrades performance due to distributional shifts across clients. Early approaches reduce client--server divergence through alignment or regularization, but generally assume mild heterogeneity and shared feature spaces, an assumption that breaks down in diverse real-world domains. Personalized Federated Learning (PFL) tailors models to individual clients using strategies such as regularization-based decomposition\citep{hanzely2020lower,li2021ditto,chenfedal,chen2024free}, partial model sharing~\cite{li2021fedbn,chen2024personalized, collins2021exploiting,chen2025restyled,feng2025taming,feng2026visual}, adaptive aggregation~\cite{zhang2020personalized,chen2024federated}, or meta-learning~\cite{fallah2020personalized}. While effective in vision and language tasks, these methods rely on transferable low-level features~\cite{chen2025federated,ren2024distributed}, which are less applicable to time series due to deeper variability in resolution, semantics, and physical context. More importantly, PFL prioritizes personalization over generalization, making it ill-suited for TSFM pretraining where cross-domain transferability is crucial. Instead of producing personalized models, our goal is to train a unified TSFM that explicitly resolves bi-level heterogeneity, capturing both inter-domain discrepancies and intra-domain conflicts during pretraining.

\section{Supplementary Materials}\label{sec:supp}

\subsection{Gradient Reversal Layer}\label{sec:grl}
The Gradient Reversal Layer (GRL)~\cite{ganin2016domain} is a standard technique in domain-adversarial learning. It acts as an identity function during the forward pass but reverses and scales gradients during backpropagation. Formally, given an input feature $\mathbf{h}$, GRL defines:
\begin{equation}
\mathrm{GRL}(\mathbf{h}) = \mathbf{h}, \quad
\frac{\partial \,\mathrm{GRL}(\mathbf{h})}{\partial \mathbf{h}} = -\lambda I,
\end{equation}
where $\lambda \geq 0$ is a tunable coefficient and $I$ denotes the identity matrix. This operation allows the sub-domain classifier to minimize its standard cross-entropy loss, while the encoder is simultaneously optimized to maximize the same objective, thereby learning domain-invariant representations. The coefficient $\lambda$ controls the strength of adversarial regularization, balancing domain invariance against the preservation of task-relevant temporal dynamics. Note that for training the local sub-domain classifier, we assign each sub-domain an index as its category label $y^{\text{dom}}$, while for the global domain classifier, we use the client index as the category label.

\subsection{Sub-domain and Domain Classifier}\label{sec:subdomain_domain_classifier}
The local sub-domain classifier promotes intra-domain invariance by mitigating fragmentation within each client, whereas the global domain classifier evaluates inter-domain separability to guide adaptive aggregation. Together, they target both levels of heterogeneity and support robust, scalable pretraining of TSFMs. The specific details of these classifiers are described below.
\paragraph{Local Sub-domain Classifier}
On each client, a sub-domain classifier $D_{\theta_C}$ is trained on $\bar{\mathbf{p}}$ with sub-domain index $y^{\text{dom}}$ as supervision. To encourage sub-domain-invariant representations, we apply a GRL to the encoder pathway and optimize the sub-domain classifier. The classifier learns to predict sub-domains while the encoder learns to remove sub-domain cues (GRL scales gradients by $-\lambda$). This min--max game reduces the mutual information between embeddings and sub-domain labels, mitigating intra-domain drift without collapsing temporal content.

\paragraph{Global Domain Classifier}
At the server level, we must evaluate how domain-specific each client's representation remains to guide aggregation. For this purpose, we train a global domain classifier $f_d:\mathbb{R}^d \to \mathbb{R}^K$, where $K$ is the number of clients (domains). Using client indices as supervision, $f_d$ is trained on client prototypes $\{\bar{\mathbf{p}}_k\}$. After training, $f_d$ is fixed and used for inference. For each prototype $\bar{\mathbf{p}}_k$, the Domain Invariance Score is defined as $\ell_k=\mathcal{L}_{\mathrm{CE}}(f_d(\bar{\mathbf{p}}_k),k).$ Low $\ell_k$ indicates that the prototype is easily classifiable and thus domain-specific, while high $\ell_k$ suggests domain-invariance as predictions approach uniform distribution. These scores, combined with semantic alignment to a global prototype, determine domain-aware aggregation weights.

\subsection{Workflow of FedTRL}\label{sec_app:workflow_of_fedtrl}
The workflow of our proposed FedTRL is shown in \textbf{Alg.~\ref{alg:fedtrl_workflow}}. At each round, the server broadcasts the current global encoder and global prototype. Each selected client performs diffusion-based reconstruction to learn temporal features, applies a local sub-domain classifier with a GRL to promote domain invariance, and aligns its prototype to the previous global prototype. The client then uploads only its encoder and prototype. The server forms a new global prototype via weighted averaging, trains a global domain classifier on client prototypes, and computes for each client a domain invariance score $(\ell_k)$ and a prototype deviation $(\delta_k)$ as the semantic alignment vectors. These are combined into scores $s_k=-\alpha\ell_k-\beta\delta_k$, softmax-normalized to weights $s_k$, which are used to aggregate encoders into the updated global model. The decoder and sub-domain classifier remain local throughout, preventing client-specific bias from contaminating the global encoder.

\begin{algorithm}[tbh]
\caption{Workflow of FedTRL}
\label{alg:fedtrl_workflow}
\KwIn{Initial encoder $\theta^{(0)}$; total rounds $T$; local epochs $E$; learning rate $\eta$; weighting factors $\lambda_{\text{domain}}, \lambda_{\text{align}}$}
\KwOut{Global encoder $\theta^{g, T}_E$ as the time series foundation model}

\For{$t = 1, \dots, T$}{
    Server selects a full client set $\mathcal{S}^t \subseteq \{1, \dots, K\}$ of clients\;
    \ForEach{ client $k \in \mathcal{S}^t$ \textbf{in parallel}}{
        Download the last global encoder $\theta^{t-1}_E$ and prototype $\mathbf{p}^{t-1}$\;
        
        \For{$e = 1$ \KwTo $E$}{
            Sample mini-batch $(\mathbf{X}_k, \mathbf{Y}_k) \sim \mathcal{D}_k$\;
            
            Generate noised patches and perform denoising reconstruction\;
            
            Compute reconstruction loss: $\mathcal{L}_{\text{task}}$\;
            
            Compute domain loss via local classifier and GRL: $\mathcal{L}_{\text{dom}}$\;
            
            Compute prototype deviation: $\mathcal{L}_{\text{align}}^t = \|\bar{\mathbf{p}}_k^{t} - \mathbf{p}^{g, t-1}\|^2$\;

            Combine local optimization objective: $\mathcal{L}^t_k = ( \mathcal{L}_{\text{task},k}^t + \lambda_{\text{dom}} \mathcal{L}_{\text{dom},k}^t + \lambda_{\text{align}}\mathcal{L}_{\text{align},k}^t)$;
            
            Update local encoder via: $\theta_{E,k}^t \leftarrow \theta_{E,k}^{t-1} - \eta \nabla_{\theta_{E,k}^t} \mathcal{L}^t_k$;

            \textcolor{gray}{(\textbf{in parallel}) Updated local decoder via: $\theta_{D,k}^t \leftarrow \theta_{D,k}^{t-1} - \eta \nabla_{\theta_{D,k}^t} \mathcal{L}^t_k$;}

            \textcolor{gray}{(\textbf{in parallel}) Updated local classifier via: $\theta_{C,k}^t \leftarrow \theta_{C,k}^{t-1} - \eta \nabla_{\theta_{C,k}^t} \mathcal{L}^t_k$;}
        }
        Upload only $\theta_{E,k}^{t}$, $\bar{\mathbf{p}}_k^t$ (local prototype) to server for aggregation\;
    }

    \tcc{Server-side aggregation}
    Aggregate prototypes: $\mathbf{p}^{g,t} =: \sum_k w_k \bar{\mathbf{p}}_k^t$\;
    
    Updated and evaluate domain classifier on all $\mathbf{p}_k^t$, obtain domain invariance score $\boldsymbol{\ell}_k^t$\;
    
    Compute semantic alignment vectors: $\delta_k = \|\bar{\mathbf{p}}^t_k - \mathbf{p}^{g,t}\|^2$\;
    
    Refine weights: $s_k = \operatorname{Softmax}(-\alpha \cdot \ell_k^t - \beta \cdot \delta_k^t)$\;
    
    Aggregate encoder: $\theta^{t}_E = \sum_k s_k \theta_{E,k}^t$\;
    
    Broadcast new $\theta_E^{g,t}$, $\mathbf{p}^{g,t}$ to all clients\;
}
\Return{Time series foundation model $\theta_E^{T}$}
\end{algorithm}

\section{Experiments Details}\label{app:baseline_bench}

\subsection{In-domain Point Forecasting}\label{sec_app:in_domain_point_forecasting}
This task evaluates self-supervised time series representation learning. We introduce baselines across three categories: (i) masked reconstruction methods, including SimMTM~\cite{dong2023simmtm} and PatchTST~\cite{nie2022time}; (ii) contrastive learning methods, including TS-TCC~\cite{eldele2021time} and CoST~\cite{woo2022cost}; and (iii) federated learning methods, including FFTS~\cite{chen2025federated} and FedAvg~\cite{mcmahan2017communication}, where each client adopts TimeDART~\cite{wang2025timedartdiffusionautoregressivetransformer}. Introducation of these baselines are summarized as follows:
\begin{itemize}
    \item \textbf{Masked Reconstruction Methods:}
    \begin{itemize}
        \item \textbf{SimMTM}~\cite{dong2023simmtm}: It recovers masked time series by learning point-wise similarities between masked and unmasked sequences through a manifold learning approach.
        \item \textbf{PatchTST}~\cite{nie2022time}: It introduces a patching mechanism and channel-independent strategy to extract local semantic information and capture long-range dependencies effectively.
    \end{itemize}

    \item \textbf{Contrastive Learning Methods:}
    \begin{itemize}
        \item \textbf{TS-TCC}~\cite{eldele2021time}: It employs temporal and contextual contrastive modules to learn robust representations by ensuring consistency across different augmented views.
        \item \textbf{CoST}~\cite{woo2022cost}: It learns disentangled seasonal-trend representations by applying contrastive losses in both the time and frequency domains.
    \end{itemize}

    \item \textbf{Federated Learning Methods:}
    \begin{itemize}
        \item \textbf{FedAvg}~\cite{mcmahan2017communication}: The standard federated learning algorithm that aggregates global models by performing a weighted average of local model parameters.
        \item \textbf{FFTS}~\cite{chen2025federated}: A federated learning framework specifically designed for time series that addresses data heterogeneity across distributed clients to achieve a unified representation space.
    \end{itemize}

    \item \textbf{Centralized Counterpart}
    \begin{itemize}
        \item \textbf{Centralized}: A baseline where all datasets are pooled into a single server for training, serving as an empirical performance upper bound for the federated scenarios.
    \end{itemize}
\end{itemize}
For fair comparison, all unsupervised baselines (except the FL ones) are implemented under the FedAvg protocol. In addition, a centralized baseline is constructed by mixing all datasets and training with the channel-independent~\cite{nie2022time}. All methods adopt a consistent Transformer backbone. The results are shown in \textbf{Table~\ref{tab:forecasting-in-sim}}.

\subsection{Full-shot Point Forecasting}\label{sec_app:full_shot_point_forecasting}
This task evaluates long-term point forecasting in a fully supervised setting, where models are trained or adapted using the entire training dataset. We benchmark FedTRL against advanced deep forecasting models, including TimeMixer~\cite{wang2024timemixer}, TimeXer~\cite{wang2024timexer}, PatchTST~\cite{nie2022time}, TimesNet~\cite{wu2022timesnet}, DLinear~\cite{zeng2023transformers}, iTransformer~\cite{liu2023itransformer}, Autoformer~\cite{wu2021autoformer}, Non-Stationary Transformer~\cite{liu2022non}, and LightTS~\cite{zhang2022less}. Brief introductions of these baselines are summarized as follows:
\begin{itemize}
    \item \textbf{TimeMixer}~\cite{wang2024timemixer}: It analyzes temporal variations through a decomposable multiscale mixing architecture, effectively integrating both fine-grained and coarse-grained information.
    \item \textbf{TimeXer}~\cite{wang2024timexer}: It reconciles endogenous and exogenous variables by introducing a novel embedding layer and capturing dependencies at both patch and variate levels.
    \item \textbf{PatchTST}~\cite{nie2022time}: It leverages a patching mechanism and channel-independent strategy to enhance the extraction of local semantic patterns and the modeling of long-range dependencies.
    \item \textbf{TimesNet}~\cite{wu2022timesnet}: It transforms 1D time series into 2D tensors based on multi-periodicity to model intra-period and inter-period variations using 2D convolutional kernels.
    \item \textbf{DLinear}~\cite{zeng2023transformers}: A simple yet effective decomposition-based MLP model that separately maps trend and seasonal components through linear layers.
    \item \textbf{iTransformer}~\cite{liu2023itransformer}: It inverts the Transformer architecture by treating each variate's entire series as a token, allowing attention to capture multivariate correlations and FFNs to model temporal patterns.
    \item \textbf{Autoformer}~\cite{wu2021autoformer}: It integrates a deep decomposition architecture with an auto-correlation mechanism to replace standard self-attention for efficient long-term forecasting.
    \item \textbf{Non-Stationary Transformer}~\cite{liu2022non}: It addresses distribution shifts in real-world data using Series Stationarization and De-stationary Attention to restore original non-stationary information.
    \item \textbf{LightTS}~\cite{zhang2022less}: A lightweight, sampling-oriented MLP-based framework that employs interval sampling and multi-scale structures for fast and robust multivariate forecasting.
\end{itemize}
For FedTRL, the model is first pretrained on Time-MoE-300B and then adapted to each target dataset with only a single epoch for fast transfer. The implementation is based on TSLib~\cite{wu2022timesnet}. Results are reported in \textbf{Table~\ref{tab:cross_domain}}.

\subsection{Zero-shot Point Forecasting}\label{sec_app:zero_shot_point_forecasting}
This task evaluates zero-shot long-term point forecasting, where models are pretrained on large time series datasets. We benchmark FedTRL against advanced foundation models, including Time-MoE~\cite{shi2024time}, TimesFM~\cite{das2024decoder}, Moirai~\cite{woo2024unifiedtraininguniversaltime}, Chronos~\cite{ansari2024chronos}, and Moment~\cite{goswami2024moment}. Evaluation is conducted on two real-world benchmarks: TSLib (four ETT-series and Weather) and RW-Bench (15 datasets from diverse regions; details are in Section~\ref{app:rw-bench}). Following the protocol of~\cite{shi2024time}, the look-back window is set to $\{512, 1024, 2048, 3072\}$ with corresponding prediction horizons $\{96, 192, 336, 720\}$. Among these baselines, TimesFM and Moment require fine-tuning for adaptation. The results on TSLib are officially from \citep{shi2024time}. Time-MoE$_\text{ultra}$ is excluded since it is not publicly available. The results are shown in \textbf{Table~\ref{tab:zero_shot}} and full results in \textbf{Table~\ref{tab:zero_shot_full}}.

\subsection{Zero-shot Probabilistic Forecasting}\label{sec_app:zero_shot_prob_forecasting}
This task evaluates zero-shot probabilistic forecasting. We evaluate FedTRL on two widely used benchmarks. GIFT-Eval~\cite{aksu2024gift} comprehensively assesses performance across 23 datasets, covering 144K time series and 177M data points, with a total of 97 forecasting configurations. We follow the official evaluation protocol provided by Salesforce and report aggregated results in \textbf{Table~\ref{tab:gift_eval}}. In addition, we evaluate both accuracy and inference efficiency on the FEV leaderboard\citep{ansari2024chronos}, maintained by AutoGluon, which includes 27 datasets for standardized zero-shot testing. Aggregated results are reported in \textbf{Fig.~\ref{fig:fev_eval}}.

\subsection{RW-Bench for Zero-shot Point Forecasting Evaluation}\label{app:rw-bench}

To assess the real-world applicability of pretrained FMs, we introduce the Real-world Weather Benchmark (RW-Bench), constructed from 15 ground-based weather observation stations collected in collaboration with our industry partners. Unlike curated academic datasets, RW-Bench preserves the complexity of raw operational data: no outliers are removed, and missing values are simply zero-filled. This design ensures that models are tested under conditions closer to real deployment scenarios. RW-Bench spans January 2022 to August 2025, providing timely and authentic observational records that complement the synthetic reanalysis data (e.g., ERA5) used in large-scale pretraining corpora such as Time-MoE-300B. \emph{Importantly, the spatiotemporal non-overlap between RW-Bench and pretraining datasets enables a fair zero-shot evaluation of model generalization.} Each dataset contains hourly multivariate time series covering ten meteorological variables: air pressure, air temperature, relative humidity, precipitation, wind direction, wind speed, maximum wind direction, maximum wind speed, maximum temperature, and minimum temperature. RW-Bench presents a challenging testbed for zero-shot forecasting. Dataset statistics for TSLib and RW-Bench are provided in \textbf{Table~\ref{tab:dataset-tslib}} and \textbf{Table~\ref{tab:dataset_rwbench}}, while sample visualizations appear in \textbf{Fig.~\ref{fig:rweatehr}}.

\begin{table}[tbh]
  \centering
    \caption{Dataset statistics about TSLib~\citep{wu2022timesnet}. Channels indicates the number of time series (i.e., variables), and the size is organized in (training, validation, testing). $^\dagger$SE: Spectral Entropy; ACF: Autocorrelation-based; Proxy: Error-based. All three quantify time series forecastability, higher scores indicate greater predictability (i.e., lower task difficulty).}
    \resizebox{\textwidth}{!}{
    \begin{tabular}{ccccccccc}
    \toprule
    Dataset & Domain & \# Channels & Frequency & Size  & Forecast Length  & SE$^\dagger$ & ACF$^\dagger$ & Proxy$^\dagger$\\
    \midrule
    ETTh1 & Power & 7     & Hourly & (8545, 2881, 2881) & $\{96, 192, 336, 720\}$ & 0.523 & 20.372 & 0.857 \\
    ETTh2 & Power & 7     & Hourly & (8545, 2881, 2881) & $\{96, 192, 336, 720\}$ & 0.652 & 35.940 & 0.943\\
    ETTm1 & Power & 7     & 15 Minute & (34465, 11521, 11521) & $\{96, 192, 336, 720\}$ & 0.580 & 24.335 & 0.949 \\
    ETTm2 & Power & 7     & 15 Minute & (34465, 11521, 11521) & $\{96, 192, 336, 720\}$ & 0.696 & 40.276 & 0.978 \\
    Electricity & Energy & 321 &  Hourly & (17805, 2537, 5166) & $\{96, 192, 336, 720\}$ & 0.706 & 12.137 & 0.814\\
    Traffic & Transportation & 862 & Hourly & (11673, 1661, 3413) & $\{96, 192, 336, 720\}$&  0.567 & 6.977 & 0.681\\
    Weather & Weather & 21    & 10 Minute & (36792, 5271, 10540) & $\{96, 192, 336, 720\}$ & 0.549 & 25.818 & 0.842 \\
    Exchange & Finance & 8 & 1 Day & (4704, 665, 1422)& $\{96, 192, 336, 720\}$ & 0.793 & 46.054 & 0.998\\
    \bottomrule
    \end{tabular}}
  \label{tab:dataset-tslib}%
\end{table}%

\begin{figure}[tbh]
    \centering
    \includegraphics[width=.9\textwidth]{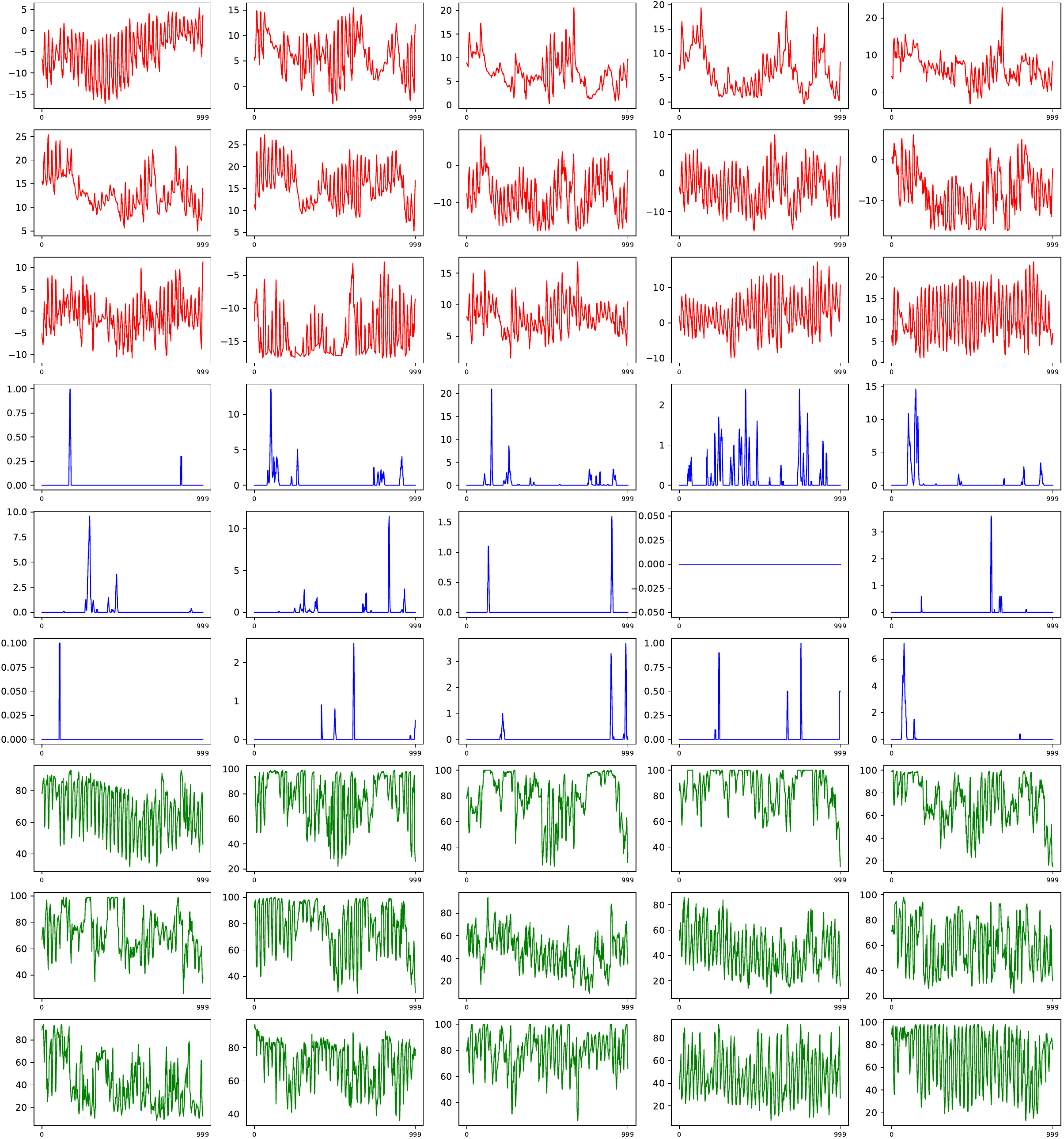}
    \caption{Visualization of RW-Bench samples: red denotes temperature, blue denotes precipitation, and green denotes humidity. Variables are arranged following the order in Table~\ref{tab:dataset_rwbench}. RW-Bench will be open-sourced once the required permissions are obtained.}
    \label{fig:rweatehr}
\end{figure}

\begin{table}[tbh]
  \centering
    \caption{\small Dataset statistics about real-world weather benchmark (RW-Bench). Channels indicates the number of time series (i.e., variables), and the size is organized in (training, validation, testing). $^\dagger$SE: Spectral Entropy; ACF: Autocorrelation-based; Proxy: Error-based. All three quantify time series forecastability, higher scores indicate greater predictability (i.e., lower task difficulty).}
    \resizebox{1\textwidth}{!}{
    \begin{tabular}{ccccccccc}
    \toprule
    Dataset & Region & \# Channels & Frequency & Size  & Forecast Length  & SE$^\dagger$ & ACF$^\dagger$ & Proxy$^\dagger$\\
    \midrule
    RW-1 & Yuzhno-Sakhalinsk & 10 & Hourly & January 2022 - August 2025 & $\{96, 192, 336, 720\}$ & 0.438 & 17.089 & 0.682\\
    
    RW-2 & Tokyo & 10 & Hourly & January 2022 - August 2025 & $\{96, 192, 336, 720\}$ & 0.438 & 17.090 & 0.682 \\
    
    RW-3 & Beijing & 10 & Hourly & January 2022 - August 2025 & $\{96, 192, 336, 720\}$ & 0.431 & 16.660 & 0.711\\
    
    RW-4 & Mumbai/Santacruz & 10 & Hourly & January 2022 - August 2025 & $\{96, 192, 336, 720\}$ & 0.398 & 16.474 & 0.706 \\
    
    RW-5 & Ho Chi Minh City & 10 & Hourly & January 2022 - August 2025 & $\{96, 192, 336, 720\}$ & 0.462 & 17.292 & 0.768\\
    
    RW-6 & Cairo & 10 & Hourly & January 2022 - August 2025 & $\{96, 192, 336, 720\}$ & 0.456 & 16.912 & 0.734\\
    
    RW-7 & Nairobi Dagoretti & 10 & Hourly &January 2022 - August 2025 & $\{96, 192, 336, 720\}$ & 0.425 & 16.217 & 0.656\\
    
    RW-8 & Nauru & 10 & Hourly &January 2022 - August 2025 & $\{96, 192, 336, 720\}$ & 0.452 & 11.586 & 0.705\\
    
    RW-9 & Honolulu & 10 & Hourly & January 2022 - August 2025 & $\{96, 192, 336, 720\}$ & 0.435 & 11.696 & 0.734\\
    
    RW-10 & Los Angeles  & 10 & Hourly & January 2022 - August 2025 & $\{96, 192, 336, 720\}$ & 0.425 & 15.879 & 0.726\\
    
    RW-11 & Edmonton & 10 & Hourly & January 2022 - August 2025 & $\{96, 192, 336, 720\}$ & 0.437 & 18.865 & 0.805 \\
    
    RW-12 & Oxford & 10 & Hourly & January 2022 - August 2025 & $\{96, 192, 336, 720\}$ & 0.455 & 18.871 & 0.816\\
    
    RW-13 & Zurich & 10 & Hourly & January 2022 - August 2025 & $\{96, 192, 336, 720\}$ & 0.463 & 17.805 & 0.801\\
    
    RW-14 & McMurdo  & 10 & Hourly & January 2022 - August 2025 & $\{96, 192, 336, 720\}$ & 0.431 & 16.438 & 0.720\\
    
    RW-15 & Novolazarevskaya & 10 & Hourly & January 2022 - August 2025 & $\{96, 192, 336, 720\}$ & 0.426 & 15.508 & 0.699\\
    \bottomrule
    \end{tabular}}
  \label{tab:dataset_rwbench}%
\end{table}%

\subsection{Model Architecture and Training Configuration}\label{sec_app:model_architecture}
We adopt a vanilla Transformer~\cite{vaswani2017attention} as the backbone. The local encoder is implemented as a causal Transformer encoder. For in-domain forecasting, we employ a standard multi-head attention mechanism with positional embeddings. For large-scale foundation model pretraining, we replace this with FlashAttention~\cite{dao2022flashattention} and RoPE~\cite{su2024roformer} to improve efficiency, consistent with advanced foundation model practices. Details in \textbf{Table~\ref{tab:configuration}}.
\begin{table}[tbh]
\centering
\caption{Detailed model architecture and training configuration.}
\resizebox{.8\textwidth}{!}{
\begin{tabular}{l|c|c}
\toprule
\multirow{2}[2]{*}{\textbf{Category}} & \textbf{Configuration} & \textbf{Configuration}\\
   &  \bf (In-domain Forecasting) & \bf (Large Foundation Model) \\
\midrule
Optimizer & Adam  & Adam \\
Batch Size & 128 & 2048 \\
Local Epochs & 10 & 100 \\
Global Rounds & 200 & 10,000\\
Learning Rate (Local Updating) & 0.0001 & 0.0005 \\
Input Sequence Length & 512 & 3072\\
Learning Scheduler & StepLR & StepLR \\
Diffusion Step & 1250 & 2500 \\
Computation Devices & 4 $\times$ Nvidia RTX A5500-24GB GPUs & 16$ \times$ Nvidia A100-80G GPUs \\
\midrule
Attention & Vanilia & FlashAttention with RoPE \\
Encoder Layer & 4  & 24 \\
(Denoising) Decoder Layers & 2 & 6\\
Feedforward Dimension & 512 & 4096\\
Model Dimension & 512 & 1024\\
Number of Heads & 4 & 8\\
Dropout & 0.4 & 0.4\\
Activation Function & GELU & GELU\\
Patch Length & 16 & 16\\
Stride & 16 & 16 \\
\midrule
Domain Classifier Structure & 2 $\times$ MLP with ReLU & 4 $\times$ MLP with ReLU \\
Training Epoch on Server & 5 & 20 \\
Learning Rate & 0.0005 & 0.0001 \\
\midrule
Domain Discriminability Weight $\alpha$ & 1.0 & 1.0 \\
Semantic Alignment Weight $\beta$ & 0.5 & 0.5 \\
Domain-adversarial Weight $\lambda_{\text{dom}}$ & 0.1 & 0.1 \\
Prototype Alignment Weight $\lambda_{\text{align}}$ & 0.1 & 0.1\\
Warmup Coefficients $R_{\text{warm}}$ & Not applicable & 4,000 \\
\bottomrule
\end{tabular}}
\label{tab:configuration}%
\end{table}

\section{Theoretical Analysis and Proofs}
\label{appendix:theory}

In this section, we provide formal definitions and complete proofs for the concepts of \textbf{Domain-Invariant Representations}, \textbf{Domain Awareness}, and \textbf{Global Dynamics}. We formulate the learning problem within the framework of domain adaptation and distributed optimization. Let $\mathcal{X} \subseteq \mathbb{R}^{T \times C}$ be the input space and $\mathcal{Z} \subseteq \mathbb{R}^{N \times d}$ be the representation space. We consider $K$ source domains (clients), where the $k$-th domain is characterized by a distribution $\mathcal{D}_k$ over $\mathcal{X} \times \mathcal{Y}$. The encoder $E_{\theta_E}: \mathcal{X} \to \mathcal{Z}$ maps input time series to latent representations.

\subsection{Domain-Invariant Representations}

\begin{definition}[\textbf{$\mathcal{H}$-Divergence for Domain Invariance}]
Let $\mathcal{H}$ be a hypothesis class of binary classifiers (domain discriminators). The $\mathcal{H}$-divergence between two distributions $\mathcal{D}_i^Z$ and $\mathcal{D}_j^Z$ in the latent space induced by $E$ is defined as:
\begin{equation}
d_{\mathcal{H}}(\mathcal{D}_i^Z, \mathcal{D}_j^Z) = 2 \sup_{h \in \mathcal{H}} \left| \operatorname{Pr}_{\mathbf{z} \sim \mathcal{D}_i^Z}[h(\mathbf{z})=1] - \operatorname{Pr}_{\mathbf{z} \sim \mathcal{D}_j^Z}[h(\mathbf{z})=1] \right|.
\end{equation}
A representation is \textit{domain-invariant} if $d_{\mathcal{H}}(\mathcal{D}_i^Z, \mathcal{D}_j^Z)$ is minimized for all pairs $i, j$.
\end{definition}

\begin{theorem}[\textbf{Adversarial Minimization of Divergence}]
\label{thm:invariance_proof}
The local adversarial optimization in FedTRL, defined by $\min_{\theta_E} \max_{\theta_C} \mathcal{L}_{\mathrm{dom}}$, minimizes the empirical $\mathcal{H}$-divergence between the local domain distribution and the uniform distribution over sub-domains, thereby upper-bounding the generalization error on unseen domains.
\end{theorem}

\begin{proof}
Consider the domain discriminator $D_{\theta_C}$ as the hypothesis $h \in \mathcal{H}$. The objective $\mathcal{L}_{\mathrm{dom}}$ is the Cross-Entropy loss. For a specific client $k$ (which shares a sub-domain label), the discriminator tries to distinguish samples from distribution $\mathcal{D}_k$ versus others.
The empirical risk of the discriminator $D$ is:
\begin{equation}
R(D) = \frac{1}{N} \sum_{i=1}^N \mathbb{I}[D(\mathbf{z}_i) \neq y_i^{dom}].
\end{equation}
According to~\cite{ben2010theory}, the $\mathcal{H}$-divergence is related to the optimal discriminator error $\epsilon(h)$ by $d_{\mathcal{H}} = 2(1 - \min_{h \in \mathcal{H}} (\epsilon_i(h) + \epsilon_j(h)))$.
In FedTRL, GRL reverses the gradient for the encoder $\theta_E$. Thus, the encoder update corresponds to:
\begin{equation}
\theta_E^* = \arg\min_{\theta_E} \left( \max_{\theta_C} \mathcal{L}_{\mathrm{dom}} \right).
\end{equation}
Maximizing $\mathcal{L}_{\mathrm{dom}}$ with respect to $\theta_C$ finds the supremum in the divergence definition (the best discriminator). Minimizing this maximum with respect to $\theta_E$ directly minimizes the upper bound of $d_{\mathcal{H}}(\mathcal{D}_k^Z, \mathcal{D}_{\text{others}}^Z)$.
Therefore, the representation $\mathbf{z}$ learned by $\theta_E$ is forced into a space where the domain discriminator performs no better than random guessing (maximizing entropy of prediction), which satisfies the definition of domain invariance.
\end{proof}

\subsection{Domain Awareness via Entropy-Regularized Aggregation}

Here we provide a theoretical justification for the specific softmax weighting mechanism used in DaG.

\begin{definition}[\textbf{Domain Awareness Risk}]
We define the \textit{Global Incompatibility Cost} for client $k$ as $C_k = \alpha \ell_k + \beta \delta_k$, where $\ell_k$ measures domain distinctiveness (risk of overfitting to local artifacts) and $\delta_k$ measures semantic deviation from the global centroid.
\end{definition}

\begin{theorem}[\textbf{Optimality of Softmax Aggregation}]
\label{thm:softmax_proof}
The weighting scheme $s_k = \operatorname{Softmax}(-C_k)$ employed in FedTRL is the optimal solution to the problem of minimizing the expected global incompatibility cost subject to an entropy regularization constraint, which prevents collapse to a single client.
\end{theorem}

\begin{proof}
Let $\mathbf{s} = [s_1, \dots, s_K]$ be the aggregation weights lying on the probability simplex $\Delta^{K-1}$ (i.e., $\sum s_k = 1, s_k \ge 0$). We aim to find weights that minimize the weighted incompatibility cost while maintaining diversity (uncertainty). This is formulated as the Entropy-Regularized Minimization problem:
\begin{equation}
\min_{\mathbf{s} \in \Delta^{K-1}} \mathcal{J}(\mathbf{s}) = \sum_{k=1}^K s_k C_k - \tau H(\mathbf{s}),
\end{equation}
where $H(\mathbf{s}) = - \sum_{k} s_k \log s_k$ is the Shannon entropy, and $\tau$ is the temperature parameter.
We form the Lagrangian with multiplier $\nu$ for the constraint $\sum s_k = 1$:
\begin{equation}
\mathcal{L}(\mathbf{s}, \nu) = \sum_{k=1}^K s_k C_k + \tau \sum_{k=1}^K s_k \log s_k + \nu (\sum_{k=1}^K s_k - 1).
\end{equation}
Taking the derivative w.r.t $s_k$ and setting to 0:
\begin{equation}
\frac{\partial \mathcal{L}}{\partial s_k} = C_k + \tau (1 + \log s_k) + \nu = 0.
\end{equation}
Solving for $s_k$:
\begin{equation}
\log s_k = -\frac{C_k + \nu}{\tau} - 1 \implies s_k = \exp\left( \frac{-C_k}{\tau} \right) \cdot \exp\left( \frac{-\nu}{\tau} - 1 \right).
\end{equation}
Using the constraint $\sum s_k = 1$ to eliminate the normalization term, we obtain:
\begin{equation}
s_k = \frac{\exp(-C_k / \tau)}{\sum_{j=1}^K \exp(-C_j / \tau)}.
\end{equation}
By setting $\tau=1$ and substituting our defined cost $C_k = \alpha \ell_k + \beta \delta_k$, we recover the exact formulation of the DaG mechanism in FedTRL:
\begin{equation}
s_k = \operatorname{Softmax}(-\alpha \ell_k - \beta \delta_k).
\end{equation}
This proves that FedTRL's ``Domain Awareness`` theoretically balances minimizing risk from unreliable domains (high $C_k$) with the need to aggregate information from diverse sources (entropy regularization).
\end{proof}

\subsection{Global Dynamics as Wasserstein Barycenter}

\begin{definition}[\textbf{Global Dynamics}]
The Global Dynamics are defined as the probabilistic process governed by the distribution $\mathcal{D}^g$ that minimizes the sum of squared Wasserstein distances to all local client prototype distributions $\{\mathcal{D}^p_k\}_{k=1}^K$.
\end{definition}

\begin{theorem}[\textbf{Convergence to Global Barycenter}]
\label{thm:barycenter_proof}
The prototype alignment objective $\mathcal{L}_{\text{align}}$ and the weighted model aggregation in FedTRL drive the global prototype $\mathbf{p}^g$ to the Fréchet mean (Barycenter) of the local semantic spaces.
\end{theorem}

\begin{proof}
Let the local prototype $\bar{\mathbf{p}}_k$ be a sample representing the centroid of distribution $\mathcal{D}_k$ in representation space $\mathbb{R}^d$. The server-side objective implicitly minimizes the weighted Euclidean distance to all local prototypes (via aggregation and subsequent broadcast alignment):
\begin{equation}
\min_{\mathbf{p}^g} \mathcal{J}(\mathbf{p}^g) = \sum_{k=1}^K w_k \| \bar{\mathbf{p}}_k - \mathbf{p}^g \|_2^2,
\end{equation}
where $w_k$ are the aggregation weights.
This is a convex optimization problem. Taking the gradient w.r.t $\mathbf{p}^g$:
\begin{equation}
\nabla_{\mathbf{p}^g} \mathcal{J} = \sum_{k=1}^K 2 w_k (\mathbf{p}^g - \bar{\mathbf{p}}_k) = 0.
\end{equation}
Solving for $\mathbf{p}^g$ (assuming $\sum w_k = 1$):
\begin{equation}
\mathbf{p}^g = \sum_{k=1}^K w_k \bar{\mathbf{p}}_k.
\end{equation}
In Optimal Transport theory~\cite{villani2008optimal}, for Gaussian distributions (or point masses approximated by prototypes), the Fréchet mean with respect to the 2-Wasserstein distance is exactly the weighted Euclidean average of the means. Thus, $\mathbf{p}^g$ represents the center of mass of the Global Dynamics. By aligning local encoders to this $\mathbf{p}^g$ via Eq. (\ref{eq:align}) in the paper, clients are effectively pulling their representation manifolds towards this shared geometric center, ensuring that the global model captures the consensus dynamics across heterogeneous time series.
\end{proof}

\section{Discussion and Additional Results}\label{sec_app:discussion}
\subsection{Model Scale Discussion}\label{sec_app:model_scale}
\textbf{Table~\ref{tab:model_dis}} compares recent time series foundation models across architecture, scale, tokenization, and training setups. Centralized approaches such as Time-MoE, Moirai, and Chronos demonstrate strong zero-shot performance (\textbf{Tables~\ref{tab:zero_shot}} and\textbf{~\ref{tab:gift_eval}}, \textbf{Figs.~\ref{fig:fev_eval}}), validating the effectiveness of large-scale pretraining on aggregated data. Unlike these centralized pipelines, FedTRL is pretrained in a fully federated manner, showing that models with hundreds of millions of parameters can be effectively trained without centralizing data. While comparable in scale to centralized FMs, FedTRL yields flexible yet domain-invariant representations. Beyond matching performance, FedTRL exemplifies a paradigm shift in FM training, demonstrating that FL can serve not only as a privacy-preserving alternative but also as a viable strategy for building large-scale models across distributed domains.

\begin{table}[tbh]
  \caption{Comparison of time series foundation models. \emph{Architecture} indicates the Transformer variant. \emph{Model Size} reports parameter counts across scales. \emph{Pre-training Scale} refers to the number of time points in pre-training datasets. \emph{Token Level} specifies the granularity of time-series tokens. \emph{Tokenization} describes which values are embedded from the series. \emph{Context Length} denotes the maximum supported input length. \emph{Probabilistic} indicates the ability to generate multiple possible predictions, in contrast to deterministic forecasters. \emph{Training} refers to the training scenarios for this model, including centralized and federated (decentralized).}
  \centering
  \resizebox{1\textwidth}{!}{
  \begin{tabular}{c|ccccccccc}
    \toprule
    \multirow{2}{*}{Method} & \textbf{FedTRL} &Time-MoE  &  Moirai & MOMENT & LLMTime & Chronos & Lag-Llama & TimesFM  \\ 
     & \textbf{(Ours)} & \citeyearpar{shi2024time} &
     \citeyearpar{woo2024unifiedtraininguniversaltime} & 
     \citeyearpar{goswami2024moment} & \citeyearpar{gruver2024large}  & \citeyearpar{ansari2024chronos} &
     \citeyearpar{rasul2023lag} &  \citeyearpar{das2023decoder}   \\
    \toprule
    Architecture & Encoder & Decoder & Encoder & Encoder & Decoder & EncDec & Decoder & Decoder \\
    \midrule
    \multirow{3}{*}{Model Size} &  & 113M  & 14M & 40M &   & 46M  &  & 17M  \\
    & 302M & 453M & 91M & 125M & -- &   200M  & 200M & 70M   \\
    &  & 2.4B & 311M & 385M &  &   710M & & 200M  \\
    \midrule
    Pre-training Scale & 300B & 300B  & 231B & 1.13B & -- & 84B & 0.36B & 100B  \\
    \midrule
    Token Level & Patch  & Point & Patch & Patch & Point & Point & Point & Patch  \\
    \midrule
    Tokenization & \scalebox{0.82}{Continuous}  & \scalebox{0.82}{Continuous} & \scalebox{0.82}{Continuous} & \scalebox{0.82}{Continuous} & \scalebox{0.82}{Discrete} & \scalebox{0.82}{Discrete} & \scalebox{0.82}{Continuous} & \scalebox{0.82}{Continuous}  \\
    \midrule
    Context Length & $\le$3072 & $\le$4096  & $\le$5000 & = 512 & - & $\le$512 & $\le$1024 & $\le$512 \\
    \midrule
    Probabilistic & True & False & True & False & True & True & True & False  \\
    \midrule
    Training & \bf Federated & Centralized & Centralized & Centralized & Centralized & Centralized & Centralized & Centralized \\
    \bottomrule
  \end{tabular}}
    \label{tab:model_dis}
\end{table}
\begin{table}[tbh]
  \centering
  \caption{In-domain forecasting results of FL baselines. \textbf{Bold}: the best.}
  \resizebox{.7\textwidth}{!}{
    \begin{tabular}{c|ccccccc}
    \toprule
    Dataset & FedTRL & FFTS   & FedAvg & FedProx & FedPer & FedRep & pFedMe \\
    \midrule
    ETT-h1 & \bf 0.448 & 0.463  & 0.476 & 0.477 & 0.492 & 0.482 & 0.499 \\
    ETT-m1 & \bf 0.375 & 0.380   & 0.399 & 0.395 & 0.405 & 0.401 & 0.414 \\
    Weather & \bf 0.241 & 0.252  & 0.264 & 0.262 & 0.281 & 0.280  & 0.287 \\
    \bottomrule
    \end{tabular}}
  \label{tab:fed}
\end{table}%

\subsection{Federated Baselines under Heterogeneity}\label{sec_app:fed}
A wide range of FL algorithms have been developed to address data heterogeneity~\cite{tan2022towards}, with strong results in CV and NLP. However, our experiments reveal that these methods transfer poorly to time series representation learning. Using the setup in Section 4.1, we evaluate representative approaches (FedProx~\cite{li2020federated}, FedPer~\cite{arivazhagan2019federated}, FedRep~\cite{collins2021exploiting}, and pFedMe~\cite{t2020personalized}) on in-domain point forecasting, aggregating final-round local models for fair comparison. As shown in \textbf{Table~\ref{tab:fed}}, only FedProx slightly improves over vanilla FedAvg, while others perform worse. We attribute this to fundamental differences between discrete classification benchmarks and continuous time series forecasting. Classification tasks have low-dimensional outputs where personalization strategies are effective, whereas forecasting requires modeling high-dimensional, continuous targets with complex temporal dependencies. Algorithms designed for discrete label spaces struggle to capture such dynamics, leading to poor generalization. In contrast, FedTRL directly addresses both inter- and intra-client heterogeneity through domain-invariant local objectives and domain-aware aggregation, yielding more robust TSFMs.

\subsection{In-domain Forecasting}\label{sec_app:forecast}
The full results of in-domain forecasting are shown in \textbf{Table~\ref{tab:forecasting-full}}, where our proposed FedTRL consistently achieves the best performance across nearly all datasets and horizons. Compared to federated baselines, it shows clear error reductions and greater stability, while also rivaling or surpassing centralized training in several cases. These results highlight that FedTRL not only mitigates heterogeneity more effectively than existing FL methods but also delivers representations strong enough to match centralized pretraining, validating its robustness across diverse forecasting scenarios.

\begin{table}[tbh]
    \centering
    \caption{\small Full results of in-domain forecasting (for four different forecasting horizons \(\{96,192,336,720\}\)). \textbf{Bold}: the best; \underline{Underline}: the second best. The $^\dagger$ symbol denotes that the federated unsupervised methods are based on the standard FedAvg~\citep{mcmahan2017communication} aggregation protocol.}
    \label{tab:forecasting-full}
    \renewcommand\arraystretch{1.25}
    \resizebox{1\textwidth}{!}{
    \begin{tabular}{c|c|cc|cccccccc|cccc|cc}
    \toprule
    \multicolumn{2}{c|}{} & \multicolumn{2}{c|}{\textsc{Ours}}     & \multicolumn{8}{c|}{\textsc{Federated Unsupervised}$^\dagger$}                          & \multicolumn{4}{c|}{\textsc{Federated FMs}} & \multicolumn{2}{c}{\textsc{Centralized}} \\
    \multicolumn{2}{c|}{Methods} & \multicolumn{2}{c|}{\textbf{FedTRL}} & \multicolumn{2}{c}{SimMTM} & \multicolumn{2}{c}{PatchTST} & \multicolumn{2}{c}{TimeMAE} & \multicolumn{2}{c|}{CoST} & \multicolumn{2}{c}{FFTS} & \multicolumn{2}{c|}{FedAvg} & \multicolumn{2}{c}{All Mixed} \\
    \multicolumn{2}{c|}{Metric} & MSE & MAE & MSE & MAE & MSE & MAE & MSE & MAE & MSE & MAE & MSE & MAE & MSE & MAE & MSE & MAE \\
    \midrule
    \multirow{4}[2]{*}{\rotatebox{90}{ETTh1}}
    &   96    & \textbf{0.367} & \textbf{0.441} & 0.399 & 0.477 & 0.391 & 0.462 & 0.410 & 0.485 & 0.414 & 0.498 & 0.387 & 0.456 & 0.393 & 0.462 & \underline{0.372} & \underline{0.454} \\
    &  192    & \textbf{0.433} & \textbf{0.464} & 0.466 & 0.503 & 0.449 & 0.482 & 0.476 & 0.507 & 0.500 & 0.531 & 0.444 & 0.481 & 0.454 & 0.490 & \underline{0.437} & \underline{0.478} \\
    &  336    & \textbf{0.467} & \textbf{0.482} & 0.519 & 0.523 & 0.504 & 0.508 & 0.531 & 0.530 & 0.550 & 0.555 & \underline{0.479} & \underline{0.491} & 0.495 & 0.502 & 0.485 & 0.500 \\
    &  720    & \textbf{0.525} & \textbf{0.501} & 0.595 & 0.546 & 0.569 & 0.528 & 0.591 & 0.549 & 0.620 & 0.576 & 0.542 & \underline{0.511} & 0.562 & 0.527 & \underline{0.539} & 0.520 \\
\midrule

\multirow{4}[2]{*}{\rotatebox{90}{ETTh2}}
    &   96    & \textbf{0.317} & \textbf{0.389} & 0.358 & 0.434 & 0.327 & 0.403 & 0.369 & 0.441 & 0.372 & 0.451 & 0.327 & 0.393 & 0.338 & 0.405 & \underline{0.325} & \underline{0.393} \\
    &  192    & \textbf{0.371} & \textbf{0.411} & 0.416 & 0.460 & 0.397 & 0.427 & 0.439 & 0.468 & 0.443 & 0.481 & 0.383 & \underline{0.417} & 0.400 & 0.433 & \underline{0.381} & 0.419 \\
    &  336    & \textbf{0.396} & \textbf{0.419} & 0.469 & 0.480 & 0.440 & 0.442 & 0.479 & 0.482 & 0.499 & 0.501 & \underline{0.409} & \underline{0.424} & 0.421 & 0.436 & 0.418 & 0.433 \\
    &  720    & \textbf{0.445} & \textbf{0.437} & 0.526 & 0.497 & 0.476 & 0.456 & 0.532 & 0.505 & 0.570 & 0.526 & \underline{0.462} & \underline{0.442} & 0.481 & 0.458 & 0.472 & 0.452 \\
\midrule
    \multirow{4}[2]{*}{\rotatebox{90}{ETTm1}}
    &   96    & 0.312 & 0.378 & 0.335 & 0.410 & 0.314 & 0.380 & 0.337 & 0.413 & 0.358 & 0.426 & \underline{0.309} & \underline{0.378} & 0.321 & 0.380 & \textbf{0.309} & \textbf{0.374} \\
    
    &  192    & \textbf{0.351} & \underline{0.395} & 0.396 & 0.433 & 0.363 & 0.402 & 0.405 & 0.439 & 0.408 & 0.447 & 0.369 & 0.403 & 0.385 & 0.404 & \underline{0.352} & \textbf{0.392} \\
    &  336    & \underline{0.397} & \underline{0.409} & 0.443 & 0.447 & 0.411 & 0.418 & 0.443 & 0.450 & 0.453 & 0.461 & 0.399 & 0.413 & 0.419 & 0.416 & \textbf{0.382} & \textbf{0.403} \\
    &  720    & \textbf{0.440} & \underline{0.426} & 0.486 & 0.466 & 0.456 & 0.432 & 0.499 & 0.474 & 0.525 & 0.487 & \underline{0.442} & 0.430 & 0.471 & 0.424 & 0.445 & \textbf{0.423} \\
\midrule
    \multirow{4}[2]{*}{\rotatebox{90}{ETTm2}}
    &   96    & \underline{0.234} & \underline{0.313} & 0.261 & 0.328 & 0.240 & 0.320 & 0.269 & 0.337 & 0.278 & 0.345 & 0.236 & 0.318 & 0.239 & 0.320 & \textbf{0.232} & \textbf{0.313} \\
    &  192    & \textbf{0.269} & \textbf{0.325} & 0.313 & 0.347 & 0.285 & 0.336 & 0.310 & 0.353 & 0.330 & 0.363 & 0.280 & 0.333 & 0.287 & 0.337 & \underline{0.272} & \underline{0.331} \\
    &  336    & \textbf{0.292} & \textbf{0.338} & 0.340 & 0.359 & 0.307 & \underline{0.343} & 0.350 & 0.365 & 0.365 & 0.375 & \underline{0.301} & 0.344 & 0.312 & 0.348 & 0.302 & 0.343 \\
    &  720    & \textbf{0.317} & \textbf{0.344} & 0.382 & 0.374 & 0.353 & 0.361 & 0.390 & 0.384 & 0.407 & 0.393 & 0.343 & \underline{0.356} & 0.360 & 0.359 & \underline{0.339} & 0.356 \\
\midrule
    \multirow{4}[2]{*}{\rotatebox{90}{Electricity}}
    &   96    & \textbf{0.151} & \textbf{0.259} & 0.173 & 0.287 & 0.155 & 0.271 & 0.179 & 0.293 & 0.198 & 0.309 & 0.153 & \underline{0.264} & 0.158 & 0.271 & \underline{0.153} & 0.265 \\
    &  192    & \textbf{0.173} & \textbf{0.269} & 0.209 & 0.302 & 0.185 & 0.282 & 0.213 & 0.309 & 0.229 & 0.323 & \underline{0.178} & \underline{0.275} & 0.188 & 0.283 & 0.182 & 0.283 \\
    &  336    & \textbf{0.189} & \textbf{0.274} & 0.230 & 0.313 & 0.202 & 0.295 & 0.240 & 0.323 & 0.255 & 0.336 & \underline{0.195} & \underline{0.283} & 0.209 & 0.296 & 0.200 & 0.291 \\
    &  720    & \textbf{0.211} & \textbf{0.287} & 0.251 & 0.322 & 0.225 & 0.304 & 0.268 & 0.335 & 0.286 & 0.349 & \underline{0.221} & \underline{0.298} & 0.237 & 0.322 & 0.221 & 0.301 \\
\midrule
\multirow{4}[2]{*}{\rotatebox{90}{Traffic}}
    &   96    & 0.355 & 0.272 & 0.384 & 0.293 & 0.354 & 0.275 & 0.395 & 0.300 & 0.394 & 0.310 & \textbf{0.339} & \textbf{0.265} & 0.343 & 0.269 & \underline{0.352} & \underline{0.272} \\
    &  192    & \underline{0.408} & \underline{0.283} & 0.453 & 0.310 & 0.422 & 0.291 & 0.462 & 0.316 & 0.478 & 0.331 & \textbf{0.395} & \textbf{0.281} & 0.405 & 0.288 & 0.414 & 0.285 \\
    &  336    & \underline{0.451} & \underline{0.295} & 0.505 & 0.324 & 0.461 & 0.300 & 0.511 & 0.329 & 0.533 & 0.346 & \textbf{0.444} & \textbf{0.294} & 0.460 & 0.304 & 0.452 & 0.296 \\
    &  720    & \textbf{0.491} & \underline{0.303} & 0.547 & 0.334 & 0.516 & 0.314 & 0.572 & 0.343 & 0.588 & 0.353 & 0.502 & \textbf{0.301} & 0.528 & 0.316 & \underline{0.495} & 0.304 \\
\midrule
    \multirow{4}[2]{*}{\rotatebox{90}{Weather}}
    &   96    & \textbf{0.201} & \textbf{0.266} & 0.229 & 0.288 & 0.208 & 0.277 & 0.237 & 0.289 & 0.248 & 0.303 & \underline{0.204} & \underline{0.271} & 0.228 & 0.285 & 0.210 & 0.277 \\
    &  192    & \textbf{0.231} & \textbf{0.279} & 0.276 & 0.308 & 0.252 & 0.293 & 0.280 & 0.310 & 0.290 & 0.318 & 0.244 & 0.288 & 0.252 & 0.295 & \underline{0.239} & \underline{0.286} \\
    &  336    & \textbf{0.254} & \textbf{0.286} & 0.302 & 0.319 & 0.277 & 0.301 & 0.310 & 0.317 & 0.316 & 0.326 & \underline{0.262} & \underline{0.296} & 0.270 & 0.304 & 0.266 & 0.297 \\
    &  720    & \textbf{0.277} & \textbf{0.297} & 0.333 & 0.325 & 0.303 & 0.309 & 0.353 & 0.332 & 0.365 & 0.342 & \underline{0.298} & \underline{0.305} & 0.306 & 0.312 & 0.301 & 0.308 \\
\midrule
\multirow{4}[2]{*}{\rotatebox{90}{Exchange}}
    &   96    & \textbf{0.311} & \textbf{0.397} & 0.376 & 0.449 & 0.326 & 0.414 & 0.374 & 0.452 & 0.385 & 0.459 & \underline{0.319} & 0.410 & 0.332 & 0.414 & 0.322 & \underline{0.410} \\
    &  192    & \textbf{0.358} & \textbf{0.419} & 0.442 & 0.474 & 0.388 & 0.439 & 0.449 & 0.485 & 0.446 & 0.483 & 0.367 & 0.432 & 0.380 & 0.436 & \underline{0.367} & \underline{0.430} \\
    &  336    & \textbf{0.393} & \textbf{0.429} & 0.477 & 0.485 & 0.431 & 0.458 & 0.493 & 0.499 & 0.502 & 0.503 & \underline{0.409} & \underline{0.446} & 0.430 & 0.452 & 0.413 & 0.448 \\
    &  720    & \textbf{0.438} & \textbf{0.446} & 0.545 & 0.512 & 0.475 & 0.470 & 0.556 & 0.523 & 0.555 & 0.523 & 0.465 & 0.472 & 0.470 & 0.476 & \underline{0.458} & \underline{0.464} \\
\midrule
\rowcolor[gray]{0.95}
\multicolumn{2}{c|}{\bf $1^{\text{st}}$ Count}  & \multicolumn{2}{c|}{\bf 49} & \multicolumn{2}{c}{0} & \multicolumn{2}{c}{0} & \multicolumn{2}{c}{0} & \multicolumn{2}{c|}{0} & \multicolumn{2}{c}{7} & \multicolumn{2}{c|}{0} & \multicolumn{2}{c}{\underline{8}} \\
    \bottomrule
    \end{tabular}
  }
\end{table}

\subsection{Zero-shot Point Forecasting}\label{sec_app:zero_shot}
We evaluate our FedTRL-trained model on two zero-shot point forecasting benchmarks: TSLib (5 datasets) and RW-Bench (15 datasets), following the evaluation protocol of~\cite{shi2024time}. The full results are reported in Table~\ref{tab:zero_shot_full} (TSLib) and Table~\ref{tab:zero_shot_full_rw} (RW-Bench). Across both benchmarks and multiple prediction horizons, FedTRL consistently achieves state-of-the-art performance, surpassing centralized foundation models trained on aggregated data. These results further demonstrate the effectiveness of FedTRL.

In addition, we provide comparisons with recent TSFM models, including VisionTS~\cite{chen2024visionts} and Sundial (94/230/1032B)~\cite{liu2025sundial}, using the same zero-shot protocol as Table~\ref{tab:zero_shot}. The results are shown in Table~\ref{tab:R6_zero_shot_forecasting} below, indicating that our FedTRL-trained TSFM continues to outperform these models under the same setting.

\begin{table}[H]
  \centering
  \renewcommand{\tabcolsep}{3pt}
  \renewcommand\arraystretch{1.15}
  \caption{Full results of zero-shot forecasting experiments. A lower MSE or MAE indicates a better prediction. TimesFM, due to its use of Weather datasets in pretraining, is not evaluated on this dataset and is denoted by a dash ($-$). \boldres{Bold}: the best, \secondres{Underline}: the second best.}
  \resizebox{\textwidth}{!}{
    \begin{tabular}{cc|cc|cc|cc|cc|cc|cc|cc|cc|cc|cc|cc|cc}
      \toprule
      \multicolumn{2}{c}{\multirow{3}{*}{\textbf{\scalebox{1.2}{Models}}}} & \multicolumn{2}{c}{\textbf{Ours}} & \multicolumn{22}{c}{\textbf{Pretrained Time Series Foundation Models (Zero-shot)}} \\
      
      \cmidrule(lr){3-4} \cmidrule(lr){5-26}
      & & \multicolumn{2}{c}{\bf FedTRL} & \multicolumn{2}{c}{\bf Time-MoE$_{base}$} & \multicolumn{2}{c}{\bf Time-MoE$_{large}$} & \multicolumn{2}{c}{\bf Time-MoE$_{ultra}$} & \multicolumn{2}{c}{\textbf{Moirai$_{small}$}} & \multicolumn{2}{c}{\textbf{Moirai$_{base}$}} & \multicolumn{2}{c}{\textbf{Moirai$_{large}$}} & \multicolumn{2}{c}{\textbf{TimesFM}} & \multicolumn{2}{c}{\textbf{Moment}} & \multicolumn{2}{c}{\textbf{Chronos$_{small}$}} & \multicolumn{2}{c}{\textbf{Chronos$_{base}$}} & \multicolumn{2}{c}{\textbf{Chronos$_{large}$}} \\
      
      \cmidrule(lr){3-4} \cmidrule(lr){5-6} \cmidrule(lr){7-8} \cmidrule(lr){9-10} \cmidrule(lr){11-12} \cmidrule(lr){13-14} \cmidrule(lr){15-16} \cmidrule(lr){17-18} \cmidrule(lr){19-20} \cmidrule(lr){21-22} \cmidrule(lr){23-24} \cmidrule(lr){25-26}
      \multicolumn{2}{c}{\scalebox{1.2}{\textbf{Metrics}}} & \textbf{MSE} & \textbf{MAE} & \textbf{MSE} & \textbf{MAE} & \textbf{MSE} & \textbf{MAE} & \textbf{MSE} & \textbf{MAE} & \textbf{MSE} & \textbf{MAE} & \textbf{MSE} & \textbf{MAE} & \textbf{MSE} & \textbf{MAE} & \textbf{MSE} & \textbf{MAE} & \textbf{MSE} & \textbf{MAE} & \textbf{MSE} & \textbf{MAE} & \textbf{MSE} & \textbf{MAE} & \textbf{MSE} & \textbf{MAE} \\
      \midrule
      
      \multirow{4}[1]{*}{ETTh1} 
        & 96    & \boldres{0.346} & \secondres{0.381} & 0.357 & \secondres{0.381} & 0.350 & 0.382 & \secondres{0.349} & \boldres{0.379} & 0.401 & 0.402 & 0.376 & 0.392 & \secondres{0.349} & \boldres{0.379} & 0.414 & 0.404 & 0.688 & 0.557 & 0.466 & 0.409 & 0.440 & 0.393 & 0.441 & 0.390 \\
        
        & 192   & \secondres{0.386} & \secondres{0.407} & \boldres{0.384} & \boldres{0.404} & 0.388 & 0.412 & 0.395 & 0.413 & 0.388 & 0.412 & 0.412 & 0.413 & 0.434 & 0.415 & 0.465 & 0.434 & 0.688 & 0.560 & 0.530 & 0.450 & 0.492 & 0.426 & 0.502 & 0.424 \\
        
        & 336   & \secondres{0.420} & 0.443 & \boldres{0.411} & 0.434 & \boldres{0.411} & \secondres{0.430} & 0.447 & 0.453 & 0.433 & \boldres{0.428} & 0.433 & \boldres{0.428} & 0.495 & 0.445 & 0.503 & 0.456 & 0.675 & 0.563 & 0.570 & 0.486 & 0.550 & 0.462 & 0.576 & 0.467 \\
        
        & 720   & 0.444 & \secondres{0.449} & 0.449 & 0.477 & \boldres{0.427} & 0.455 & 0.457 & 0.462 & \secondres{0.439} & 0.454 & 0.447 & \boldres{0.444} & 0.611 & 0.510 & 0.511 & 0.481 & 0.683 & 0.585 & 0.615 & 0.543 & 0.882 & 0.591 & 0.835 & 0.583 \\
        
        & {\textbf{Avg.}} & \secondres{0.399} & \secondres{0.420} & 0.400 & 0.424 & \boldres{0.394} & \boldres{0.419} & 0.412 & 0.426 & 0.428 & 0.427 & 0.417 & \boldres{0.419} & 0.480 & 0.439 & 0.473 & 0.443 & 0.683 & 0.566 & 0.545 & 0.472 & 0.591 & 0.468 & 0.588 & 0.466 \\
      \midrule
      
      \multirow{4}[0]{*}{ETTh2}
        & 96    & 0.300 & 0.352 & 0.305 & 0.359 & 0.302 & 0.354 & \boldres{0.292} & 0.352 & 0.297 & \secondres{0.336} & \secondres{0.294} & \boldres{0.330} & 0.296 & \boldres{0.330} & 0.315 & 0.349 & 0.342 & 0.396 & 0.307 & 0.356 & 0.308 & 0.343 & 0.320 & 0.345 \\
        
        & 192   & \boldres{0.335} & \boldres{0.369} & 0.351 & 0.386 & 0.364 & 0.385 & \secondres{0.347} & 0.379 & 0.368 & 0.381 & 0.365 & 0.375 & 0.361 & \secondres{0.371} & 0.388 & 0.395 & 0.354 & 0.402 & 0.376 & 0.401 & 0.384 & 0.392 & 0.406 & 0.399 \\
        
        & 336   & 0.371 & \secondres{0.392} & 0.391 & 0.418 & 0.417 & 0.425 & 0.406 & 0.419 & \secondres{0.370} & 0.393 & 0.376 & \boldres{0.390} & 0.390 & \boldres{0.390} & 0.422 & 0.427 & \boldres{0.356} & 0.407 & 0.408 & 0.431 & 0.429 & 0.430 & 0.492 & 0.453 \\
        
        & 720   & \boldres{0.394} & \boldres{0.407} & 0.419 & 0.454 & 0.537 & 0.496 & 0.439 & 0.447 & 0.411 & 0.426 & 0.416 & 0.433 & 0.423 & \secondres{0.418} & 0.443 & 0.454 & \secondres{0.395} & 0.434 & 0.604 & 0.533 & 0.501 & 0.477 & 0.603 & 0.511 \\
        
        & {\textbf{Avg.}} & \boldres{0.350} & \secondres{0.380} & 0.366 & 0.404 & 0.405 & 0.415 & 0.371 & 0.399 & \secondres{0.361} & 0.384 & 0.362 & 0.382 & 0.367 & \boldres{0.377} & 0.392 & 0.406 & \secondres{0.361} & 0.409 & 0.424 & 0.430 & 0.405 & 0.410 & 0.455 & 0.427 \\
      \midrule
      
      \multirow{4}[0]{*}{ETTm1}
        & 96    & \secondres{0.284} & \secondres{0.342} & 0.338 & 0.368 & 0.309 & 0.557 & \boldres{0.281} & \boldres{0.341} & 0.418 & 0.392 & 0.363 & 0.356 & 0.380 & 0.361 & 0.361 & 0.370 & 0.654 & 0.527 & 0.511 & 0.423 & 0.454 & 0.408 & 0.457 & 0.403 \\
        
        & 192   & \secondres{0.319} & \secondres{0.366} & 0.353 & 0.388 & 0.346 & 0.381 & \boldres{0.305} & \boldres{0.358} & 0.431 & 0.405 & 0.388 & 0.375 & 0.412 & 0.383 & 0.414 & 0.405 & 0.662 & 0.532 & 0.618 & 0.485 & 0.567 & 0.477 & 0.530 & 0.450 \\
        
        & 336   & \boldres{0.358} & 0.405 & 0.381 & 0.413 & 0.373 & 0.408 & \secondres{0.369} & \secondres{0.395} & 0.433 & 0.412 & 0.416 & \boldres{0.392} & 0.436 & 0.400 & 0.445 & 0.429 & 0.672 & 0.537 & 0.683 & 0.524 & 0.662 & 0.525 & 0.577 & 0.481 \\
        
        & 720   & \boldres{0.435} & \secondres{0.419} & 0.504 & 0.493 & 0.475 & 0.477 & 0.469 & 0.472 & 0.462 & 0.432 & \secondres{0.460} & \boldres{0.418} & 0.462 & 0.420 & 0.512 & 0.471 & 0.692 & 0.551 & 0.748 & 0.566 & 0.900 & 0.591 & 0.660 & 0.526 \\
        
        & {\textbf{Avg.}} & \boldres{0.349} & \boldres{0.383} & 0.394 & 0.415 & 0.376 & 0.405 & \secondres{0.356} & 0.391 & 0.436 & 0.410 & 0.406 & \secondres{0.385} & 0.422 & 0.391 & 0.433 & 0.418 & 0.670 & 0.536 & 0.640 & 0.499 & 0.645 & 0.500 & 0.555 & 0.465 \\
      \midrule
      \multirow{4}[0]{*}{ETTm2}
        & 96    & 0.205 & 0.279 & 0.201 & 0.291 & \boldres{0.197} & 0.286 & \secondres{0.198} & 0.288 & 0.214 & 0.288 & 0.205 & 0.273 & 0.211 & 0.274 & 0.202 & \boldres{0.270} & 0.260 & 0.335 & 0.209 & 0.291 & 0.199 & 0.274 & \boldres{0.197} & \secondres{0.271} \\
        
        & 192   & \secondres{0.245} & 0.321 & 0.258 & 0.334 & 0.250 & 0.322 & \boldres{0.235} & \boldres{0.312} & 0.284 & 0.332 & 0.275 & 0.316 & 0.281 & 0.318 & 0.289 & 0.321 & 0.289 & 0.350 & 0.280 & 0.341 & 0.261 & 0.322 & 0.254 & \secondres{0.314} \\
        
        & 336   & \secondres{0.311} & \boldres{0.343} & 0.324 & 0.373 & 0.337 & 0.375 & \boldres{0.293} & \secondres{0.348} & 0.331 & 0.362 & 0.329 & 0.350 & 0.341 & 0.355 & 0.360 & 0.366 & 0.324 & 0.369 & 0.354 & 0.390 & 0.326 & 0.366 & 0.313 & 0.353 \\
        
        & 720   & \boldres{0.367} & \boldres{0.377} & 0.488 & 0.464 & 0.480 & 0.461 & 0.427 & 0.423 & 0.402 & \secondres{0.408} & 0.437 & 0.411 & 0.485 & 0.428 & 0.462 & 0.430 & \secondres{0.394} & 0.409 & 0.553 & 0.499 & 0.455 & 0.439 & 0.416 & 0.415 \\
        & {\textbf{Avg.}} & \boldres{0.282} & \boldres{0.330} & 0.317 & 0.365 & 0.316 & 0.361 & \secondres{0.288} & 0.344 & 0.307 & 0.347 & 0.311 & \secondres{0.337} & 0.329 & 0.343 & 0.328 & 0.346 & 0.316 & 0.365 & 0.349 & 0.380 & 0.310 & 0.350 & 0.295 & 0.338 \\
      \midrule
      \multirow{4}[0]{*}{Weather}
        & 96    & \boldres{0.149} & \boldres{0.210} & 0.160 & 0.214 & 0.159 & 0.213 & \secondres{0.157} & \secondres{0.211} & 0.198 & 0.222 & 0.220 & 0.217 & 0.199 & \secondres{0.211}& -- & -- & 0.243 & 0.255 & 0.211 & 0.243 & 0.203 & 0.238 & 0.194 & 0.235 \\
        
        & 192   & \boldres{0.188} & \boldres{0.240} & 0.210 & 0.260 & 0.215 & 0.266 & \secondres{0.208} & 0.256 & 0.247 & 0.265 & 0.271 & 0.259 & 0.246 & \secondres{0.251} & -- & -- & 0.278 & 0.329 & 0.263 & 0.294 & 0.256 & 0.290 & 0.249 & 0.285 \\
        
        & 336   & \boldres{0.249} & \boldres{0.282} & 0.274 & 0.309 & 0.291 & 0.322 & \secondres{0.255} & \secondres{0.290} & 0.283 & 0.303 & 0.286 & 0.297 & 0.274 & 0.291 & -- & -- & 0.306 & 0.346 & 0.321 & 0.339 & 0.314 & 0.336 & 0.302 & 0.327 \\
        
        & 720   & 0.366 & 0.384 & 0.418 & 0.405 & 0.415 & 0.400 & 0.405 & 0.397 & 0.373 & \secondres{0.354} & 0.373 & \secondres{0.354} & \boldres{0.337} & \boldres{0.340} & -- & -- & \secondres{0.350} & 0.374 & 0.404 & 0.397 & 0.397 & 0.396 & 0.372 & 0.378 \\
        
        & {\textbf{Avg.}} & \boldres{0.238} & \secondres{0.279} & 0.265 & 0.297 & 0.270 & 0.300 & \secondres{0.256} & 0.288 & 0.275 & 0.286 & 0.287 & 0.281 & 0.264 & \boldres{0.273} & -- & -- & 0.294 & 0.326 & 0.300 & 0.318 & 0.292 & 0.315 & 0.279 & 0.306 \\
      \midrule
      \rowcolor[gray]{0.95}
      \multicolumn{2}{c|}{\scalebox{1.1}{\textbf{{$1^{\text{st}}$ Count}}}}
            &
    \multicolumn{2}{c|}{\boldres{22}} &
    \multicolumn{2}{c|}{3} &
    \multicolumn{2}{c|}{5} &
    \multicolumn{2}{c|}{9} &
    \multicolumn{2}{c|}{1} &
    \multicolumn{2}{c|}{\secondres{7}} &
    \multicolumn{2}{c|}{\secondres{7}} &
    \multicolumn{2}{c|}{1} &
    \multicolumn{2}{c|}{1} &
    \multicolumn{2}{c|}{0} &
    \multicolumn{2}{c|}{0} &
    \multicolumn{2}{c}{1} \\
      \bottomrule
    \end{tabular}
  }
  \label{tab:zero_shot_full}%
\end{table}
\begin{table}[H]
  \centering
  \caption{Full results of real-world weather benchmark (RW-bench).}
  \resizebox{\textwidth}{!}{
    \begin{tabular}{cc|cc|cc|cc|cc|cc|cc|cc|cc|cc|cc|cc}
      \toprule
      \multicolumn{2}{c}{\multirow{3}{*}{\textbf{\scalebox{1.2}{Models}}}} & \multicolumn{2}{c}{\textbf{Ours}} & \multicolumn{20}{c}{\textbf{Pretrained Time Series Foundation Models (Zero-shot)}} \\
      
      \cmidrule(lr){3-4} \cmidrule(lr){5-24}
      
      & & \multicolumn{2}{c}{\bf FedTRL} & \multicolumn{2}{c}{\bf Time-MoE$_{base}$} & \multicolumn{2}{c}{\bf Time-MoE$_{large}$}  & \multicolumn{2}{c}{\textbf{Moirai$_{small}$}} & \multicolumn{2}{c}{\textbf{Moirai$_{base}$}} & \multicolumn{2}{c}{\textbf{Moirai$_{large}$}} & \multicolumn{2}{c}{\textbf{TimesFM}} & \multicolumn{2}{c}{\textbf{Moment}} & \multicolumn{2}{c}{\textbf{Chronos$_{small}$}} & \multicolumn{2}{c}{\textbf{Chronos$_{base}$}} & \multicolumn{2}{c}{\textbf{Chronos$_{large}$}} \\
      
      \cmidrule(lr){3-4} \cmidrule(lr){5-6} \cmidrule(lr){7-8} \cmidrule(lr){9-10} \cmidrule(lr){11-12} \cmidrule(lr){13-14} \cmidrule(lr){15-16} \cmidrule(lr){17-18} \cmidrule(lr){19-20} \cmidrule(lr){21-22} \cmidrule(lr){23-24} 
      \multicolumn{2}{c}{\scalebox{1.2}{\textbf{Metrics}}} & \textbf{MSE} & \textbf{MAE} & \textbf{MSE} & \textbf{MAE} & \textbf{MSE} & \textbf{MAE} & \textbf{MSE} & \textbf{MAE} & \textbf{MSE} & \textbf{MAE} & \textbf{MSE} & \textbf{MAE} & \textbf{MSE} & \textbf{MAE} & \textbf{MSE} & \textbf{MAE} & \textbf{MSE} & \textbf{MAE} & \textbf{MSE} & \textbf{MAE} & \textbf{MSE} & \textbf{MAE} \\
      \midrule
    \multirow{4}[0]{*}{RW1} & 96    &  \secondres{0.508} & \boldres{0.453} & 0.510 & 0.460 & \boldres{0.505} & \secondres{0.455} & 0.885 & 0.547 & 0.879 & 0.529 & 0.986 & 0.530 & 0.883 & 0.703 & 0.674 & 0.594 & 0.850 & 0.581 & 0.879 & 0.589 & 0.873 & 0.583 \\
          & 192   & \boldres{0.528} & \boldres{0.473} & 0.540 & 0.492 & \secondres{0.532} & \secondres{0.485} & 0.950 & 0.569 & 0.896 & 0.546 & 1.011 & 0.548 & 0.903 & 0.711 & 0.717 & 0.600 & 0.917 & 0.620 & 0.979 & 0.634 & 0.950 & 0.623 \\
          & 336   & \boldres{0.546} & \boldres{0.488} & 0.566 & 0.519 & \secondres{0.556} & \secondres{0.505} & 1.148 & 0.588 & 0.813 & 0.552 & 0.868 & 0.560 & 0.889 & 0.708 & 0.913 & 0.700 & 0.957 & 0.649 & 1.047 & 0.667 & 0.973 & 0.646 \\
          & 720   & \boldres{0.552} & \boldres{0.497} & 0.643 & 0.580 & \secondres{0.568} & \secondres{0.519} & 1.869 & 0.667 & 0.958 & 0.602 & 1.259 & 0.625 & 0.877 & 0.705 & 1.740 & 1.019 & 1.089 & 0.708 & 1.142 & 0.715 & 1.012 & 0.687 \\
          & \textbf{Avg} & \boldres{0.534} & \boldres{0.477} & 0.565 & 0.513 & \secondres{0.540} & \secondres{0.491} & 1.213 & 0.593 & 0.887 & 0.557 & 1.031 & 0.566 & 0.888 & 0.707 & 1.011 & 0.728 & 0.953 & 0.639 & 1.012 & 0.651 & 0.952 & 0.635 \\
    \midrule
    \multirow{4}[0]{*}{RW2} & 96   & \secondres{0.433} & \boldres{0.420} & \boldres{0.431} & 0.426 & 0.436 & \secondres{0.424} & 1.541 & 0.526 & 1.314 & 0.483 & 1.877 & 0.491 & 0.846 & 0.702 & 0.677 & 0.567 & 0.778 & 0.543 & 0.779 & 0.546 & 0.811 & 0.553 \\
          & 192   & \boldres{0.460} & \boldres{0.445} & \boldres{0.460} & \secondres{0.455} & \secondres{0.470} & 0.456 & 1.407 & 0.560 & 1.021 & 0.504 & 2.416 & 0.529 & 0.841 & 0.700 & 0.788 & 0.623 & 0.855 & 0.587 & 0.874 & 0.595 & 0.895 & 0.598 \\
          & 336   & \boldres{0.481} & \boldres{0.459} & \secondres{0.497} & 0.487 & \secondres{0.497} & \secondres{0.477} & 1.666 & 0.596 & 1.168 & 0.529 & 1.910 & 0.553 & 0.842 & 0.700 & 0.881 & 0.653 & 0.943 & 0.625 & 0.942 & 0.629 & 0.960 & 0.630 \\
          & 720   & \boldres{0.505} & \boldres{0.481} & 0.575 & 0.545 & \secondres{0.539} & \secondres{0.505} & 2.236 & 0.660 & 2.600 & 0.603 & 4.302 & 0.637 & 0.847 & 0.702 & 1.933 & 1.044 & 1.146 & 0.691 & 1.054 & 0.681 & 1.049 & 0.677 \\
          & \textbf{Avg} & \boldres{0.470} & \boldres{0.451} & 0.491 & 0.479 & \secondres{0.485} & \secondres{0.465} & 1.713 & 0.585 & 1.526 & 0.530 & 2.626 & 0.552 & 0.844 & 0.701 & 1.070 & 0.721 & 0.931 & 0.612 & 0.912 & 0.613 & 0.929 & 0.614 \\
    \midrule
    \multirow{4}[0]{*}{RW3} & 96    & 0.580 & \secondres{0.485} & \boldres{0.566} & \boldres{0.474} & \secondres{0.579} & \secondres{0.485} & 1.461 & 0.616 & 1.254 & 0.576 & 1.511 & 0.561 & 0.952 & 0.728 & 0.726 & 0.598 & 0.977 & 0.614 & 0.980 & 0.625 & 0.943 & 0.600 \\
          & 192   & \boldres{0.594} & \secondres{0.501} & \boldres{0.594} & \boldres{0.498} & \secondres{0.597} & 0.508 & 1.224 & 0.610 & 0.991 & 0.577 & 1.450 & 0.571 & 0.963 & 0.730 & 0.808 & 0.629 & 1.032 & 0.649 & 1.058 & 0.668 & 1.029 & 0.640 \\
          & 336   & \boldres{0.605} & \boldres{0.511} & 0.620 & 0.522 & \secondres{0.608} & \secondres{0.519} & 1.578 & 0.629 & 0.900 & 0.583 & 1.053 & 0.587 & 0.973 & 0.735 & 0.908 & 0.674 & 1.084 & 0.675 & 1.121 & 0.696 & 1.054 & 0.658 \\
          & 720   & \boldres{0.617} & \boldres{0.529} & 0.681 & 0.576 & \secondres{0.631} & \secondres{0.544} & 2.740 & 0.709 & 1.233 & 0.650 & 1.755 & 0.653 & 0.953 & 0.732 & 2.197 & 1.137 & 1.181 & 0.725 & 1.218 & 0.741 & 1.165 & 0.729 \\
          & \textbf{Avg} & \boldres{0.599} & \boldres{0.506} & 0.615 & 0.517 & \secondres{0.604} & \secondres{0.514} & 1.751 & 0.641 & 1.094 & 0.597 & 1.442 & 0.593 & 0.960 & 0.731 & 1.159 & 0.760 & 1.068 & 0.666 & 1.094 & 0.682 & 1.048 & 0.657 \\
    \midrule
    \multirow{4}[0]{*}{RW4} & 96    & \boldres{0.501} & \secondres{0.484} & 0.512 & 0.486 & \secondres{0.502} & \boldres{0.483} & 1.634 & 0.603 & 1.313 & 0.562 & 1.808 & 0.571 & 0.895 & 0.728 & 0.707 & 0.588 & 0.891 & 0.629 & 0.870 & 0.628 & 0.889 & 0.630 \\
          & 192   & \boldres{0.517} & \boldres{0.501} & 0.535 & 0.510 & \secondres{0.522} & \secondres{0.505} & 1.299 & 0.606 & 1.104 & 0.571 & 1.651 & 0.592 & 0.886 & 0.726 & 0.776 & 0.619 & 0.951 & 0.663 & 0.959 & 0.671 & 0.961 & 0.669 \\
          & 336   & \boldres{0.534} & \boldres{0.512} & 0.557 & 0.534 & \secondres{0.544} & \secondres{0.522} & 1.457 & 0.629 & 0.931 & 0.575 & 1.174 & 0.610 & 0.893 & 0.727 & 0.862 & 0.667 & 1.004 & 0.690 & 1.025 & 0.701 & 1.013 & 0.695 \\
          & 720   & \boldres{0.558} & \boldres{0.530} & 0.616 & 0.582 & \secondres{0.582} & \secondres{0.547} & 2.364 & 0.703 & 1.576 & 0.644 & 2.106 & 0.683 & 0.888 & 0.726 & 2.205 & 1.126 & 1.169 & 0.747 & 1.170 & 0.757 & 1.089 & 0.740 \\
          & \textbf{Avg} & \boldres{0.527} & \boldres{0.507} & 0.555 & 0.528 & \secondres{0.538} & \secondres{0.514} & 1.688 & 0.635 & 1.231 & 0.588 & 1.685 & 0.614 & 0.890 & 0.727 & 1.138 & 0.750 & 1.004 & 0.682 & 1.006 & 0.689 & 0.988 & 0.684 \\
    \midrule
    \multirow{4}[0]{*}{RW5} & 96    & \boldres{0.489} & \secondres{0.469} & \secondres{0.491} & 0.470 & 0.498 & \boldres{0.468} & 1.002 & 0.571 & 0.911 & 0.541 & 0.941 & 0.533 & 0.900 & 0.732 & 0.707 & 0.600 & 0.897 & 0.622 & 0.911 & 0.626 & 0.924 & 0.627 \\
          & 192   & \boldres{0.509} & \boldres{0.490} & \secondres{0.513} & 0.494 & 0.514 & \secondres{0.491} & 1.049 & 0.588 & 0.815 & 0.547 & 0.863 & 0.551 & 0.912 & 0.735 & 0.759 & 0.617 & 0.973 & 0.661 & 0.998 & 0.672 & 0.991 & 0.668 \\
          & 336   & \boldres{0.517} & \boldres{0.499} & 0.538 & 0.516 & \secondres{0.529} & \secondres{0.507} & 1.503 & 0.614 & 0.805 & 0.555 & 1.040 & 0.577 & 0.905 & 0.732 & 1.013 & 0.727 & 1.039 & 0.695 & 1.071 & 0.706 & 1.053 & 0.696 \\
          & 720   & \boldres{0.526} & \boldres{0.510} & 0.613 & 0.570 & \secondres{0.541} & \secondres{0.524} & 2.697 & 0.687 & 1.226 & 0.616 & 1.825 & 0.653 & 0.910 & 0.735 & 1.741 & 1.008 & 1.155 & 0.748 & 1.199 & 0.755 & 1.166 & 0.738 \\
          & \textbf{Avg} & \boldres{0.510} & \boldres{0.492} & 0.539 & 0.512 & \secondres{0.520} & \secondres{0.497} & 1.563 & 0.615 & 0.939 & 0.565 & 1.167 & 0.578 & 0.907 & 0.733 & 1.055 & 0.738 & 1.016 & 0.682 & 1.045 & 0.690 & 1.033 & 0.683 \\
    \midrule
    \multirow{4}[0]{*}{RW6} & 96    & 0.786 & \boldres{0.499} & 0.796 & \secondres{0.516} & \secondres{0.785} & \boldres{0.499} & 1.442 & 0.596 & 1.169 & 0.561 & 1.555 & 0.566 & 1.290 & 0.778 & \boldres{0.692} & 0.597 & 1.365 & 0.656 & 1.388 & 0.660 & 1.349 & 0.660 \\
          & 192   & \secondres{0.808} & \boldres{0.517} & \secondres{0.808} & 0.531 & 0.809 & \secondres{0.521} & 1.228 & 0.597 & 1.021 & 0.565 & 1.706 & 0.586 & 1.299 & 0.779 & \boldres{0.754} & 0.627 & 1.462 & 0.692 & 1.533 & 0.704 & 1.463 & 0.699 \\
          & 336   & \boldres{0.821} & \boldres{0.525} & 0.835 & 0.549 & \secondres{0.828} & \secondres{0.537} & 1.733 & 0.623 & 1.126 & 0.585 & 1.456 & 0.610 & 1.301 & 0.780 & 0.860 & 0.667 & 1.524 & 0.720 & 1.630 & 0.737 & 1.552 & 0.727 \\
          & 720   & \boldres{0.831} & \boldres{0.539} & 0.897 & 0.599 & \secondres{0.837} & \secondres{0.552} & 2.410 & 0.694 & 1.877 & 0.657 & 2.881 & 0.690 & 1.290 & 0.780 & 1.703 & 0.982 & 1.627 & 0.772 & 1.752 & 0.786 & 1.670 & 0.766 \\
          & \textbf{Avg} & \boldres{0.812} & \boldres{0.520} & 0.834 & 0.549 & \secondres{0.815} & \secondres{0.527} & 1.703 & 0.628 & 1.298 & 0.592 & 1.900 & 0.613 & 1.295 & 0.779 & 1.002 & 0.718 & 1.495 & 0.710 & 1.576 & 0.722 & 1.508 & 0.713 \\
    \midrule
    \multirow{4}[0]{*}{RW7} & 96    & \boldres{0.479} & \boldres{0.478} & \secondres{0.492} & 0.481 & 0.493 & \secondres{0.479} & 0.703 & 0.550 & 0.695 & 0.541 & 0.660 & 0.531 & 0.887 & 0.726 & 0.728 & 0.615 & 0.879 & 0.625 & 0.897 & 0.626 & 0.888 & 0.626 \\
          & 192   & \boldres{0.503} & \secondres{0.504} & \secondres{0.514} & \boldres{0.503} & \secondres{0.514} & \secondres{0.504} & 0.712 & 0.561 & 0.686 & 0.549 & 0.655 & 0.553 & 0.891 & 0.727 & 0.817 & 0.634 & 0.939 & 0.662 & 0.981 & 0.670 & 0.953 & 0.665 \\
          & 336   & \boldres{0.518} & \boldres{0.520} & 0.536 & 0.525 & \secondres{0.527} & \secondres{0.522} & 0.791 & 0.582 & 0.696 & 0.567 & 0.661 & 0.566 & 0.885 & 0.725 & 0.899 & 0.696 & 0.981 & 0.687 & 1.028 & 0.697 & 1.001 & 0.691 \\
          & 720   & \boldres{0.526} & \boldres{0.532} & 0.582 & 0.564 & \secondres{0.537} & \secondres{0.533} & 1.462 & 0.643 & 0.733 & 0.610 & 0.868 & 0.615 & 0.880 & 0.724 & 1.961 & 1.071 & 1.081 & 0.731 & 1.094 & 0.735 & 1.021 & 0.722 \\
          & \textbf{Avg} & \boldres{0.506} & \boldres{0.508} & 0.531 & 0.518 & \secondres{0.518} & \secondres{0.509} & 0.917 & 0.584 & 0.703 & 0.567 & 0.711 & 0.566 & 0.886 & 0.725 & 1.101 & 0.754 & 0.970 & 0.676 & 1.000 & 0.682 & 0.966 & 0.676 \\
    \midrule
    \multirow{4}[0]{*}{RW8} & 96    & 0.520 & 0.468 & 0.525 & 0.471 & 0.527 & 0.466 & 0.853 & 0.552 & 0.828 & 0.525 & 0.881 & 0.531 & 1.035 & 0.748 & 0.786 & 0.643 & 0.858 & 0.598 & 0.929 & 0.622 & 0.910 & 0.615 \\
          & 192   & \secondres{0.539} & \secondres{0.492} & \boldres{0.538} & \boldres{0.489} & 0.542 & \boldres{0.489} & 0.769 & 0.563 & 0.773 & 0.539 & 0.736 & 0.546 & 1.019 & 0.746 & 0.920 & 0.701 & 0.925 & 0.638 & 1.058 & 0.679 & 1.012 & 0.665 \\
          & 336   & \boldres{0.545} & \secondres{0.503} & 0.557 & 0.509 & \secondres{0.547} & \boldres{0.502} & 0.800 & 0.581 & 0.805 & 0.557 & 0.718 & 0.559 & 1.018 & 0.748 & 1.026 & 0.752 & 0.986 & 0.668 & 1.170 & 0.723 & 1.119 & 0.702 \\
          & 720   & \boldres{0.560} & \boldres{0.516} & 0.621 & 0.556 & \secondres{0.573} & \secondres{0.527} & 1.054 & 0.631 & 0.818 & 0.589 & 0.937 & 0.614 & 1.018 & 0.749 & 2.017 & 1.111 & 1.147 & 0.718 & 1.356 & 0.789 & 1.298 & 0.777 \\
          & \textbf{Avg} & \boldres{0.541} & \boldres{0.495} & 0.560 & 0.506 & \secondres{0.547} & \secondres{0.496} & 0.869 & 0.582 & 0.806 & 0.553 & 0.818 & 0.563 & 1.022 & 0.748 & 1.187 & 0.802 & 0.979 & 0.655 & 1.128 & 0.703 & 1.085 & 0.690 \\
    \midrule
    \multirow{4}[0]{*}{RW9} & 96    & \boldres{0.576} & \secondres{0.479} & \boldres{0.576} & \boldres{0.476} & \secondres{0.585} & 0.482 & 1.143 & 0.578 & 1.005 & 0.557 & 1.304 & 0.559 & 1.012 & 0.720 & 0.791 & 0.645 & 0.939 & 0.605 & 1.062 & 0.636 & 1.041 & 0.630 \\
          & 192   & \secondres{0.622} & \secondres{0.513} & \boldres{0.613} & \boldres{0.508} & 0.629 & 0.517 & 1.124 & 0.607 & 0.996 & 0.582 & 1.066 & 0.586 & 1.025 & 0.725 & 0.913 & 0.685 & 1.039 & 0.654 & 1.228 & 0.703 & 1.178 & 0.695 \\
          & 336   & \boldres{0.644} & \boldres{0.528} & \secondres{0.648} & 0.532 & \secondres{0.648} & \secondres{0.529} & 1.121 & 0.623 & 0.966 & 0.596 & 1.061 & 0.618 & 1.036 & 0.728 & 1.091 & 0.791 & 1.106 & 0.686 & 1.357 & 0.748 & 1.321 & 0.738 \\
          & 720   & \boldres{0.658} & \boldres{0.544} & 0.707 & 0.573 & \secondres{0.669} & \secondres{0.552} & 1.313 & 0.661 & 1.029 & 0.634 & 1.557 & 0.704 & 1.037 & 0.731 & 2.049 & 1.133 & 1.221 & 0.734 & 1.518 & 0.802 & 1.500 & 0.778 \\
          & \textbf{Avg} & \boldres{0.625} & \boldres{0.516} & 0.636 & 0.522 & \secondres{0.633} & \secondres{0.520} & 1.175 & 0.617 & 0.999 & 0.592 & 1.247 & 0.617 & 1.028 & 0.726 & 1.211 & 0.813 & 1.076 & 0.670 & 1.291 & 0.722 & 1.260 & 0.710 \\
    \midrule
    \multirow{4}[0]{*}{RW10} & 96    & \boldres{0.566} & \boldres{0.505} & 0.585 & \secondres{0.518} & \secondres{0.579} & \boldres{0.505} & 0.923 & 0.607 & 0.926 & 0.591 & 0.836 & 0.578 & 0.989 & 0.769 & 0.752 & 0.626 & 1.014 & 0.657 & 1.073 & 0.661 & 1.016 & 0.651 \\
          & 192   & \boldres{0.599} & \boldres{0.537} & 0.612 & \secondres{0.546} & \secondres{0.610} & \boldres{0.537} & 0.935 & 0.629 & 0.949 & 0.617 & 0.874 & 0.615 & 0.986 & 0.768 & 0.777 & 0.638 & 1.087 & 0.703 & 1.144 & 0.708 & 1.098 & 0.698 \\
          & 336   & \boldres{0.616} & \boldres{0.553} & 0.626 & 0.566 & \secondres{0.621} & \secondres{0.554} & 0.907 & 0.637 & 0.972 & 0.633 & 0.864 & 0.634 & 0.976 & 0.768 & 0.852 & 0.678 & 1.117 & 0.728 & 1.173 & 0.730 & 1.120 & 0.720 \\
          & 720   & \boldres{0.618} & \boldres{0.564} & 0.655 & \secondres{0.600} & \secondres{0.629} & \boldres{0.564} & 1.119 & 0.682 & 1.063 & 0.680 & 1.070 & 0.690 & 0.968 & 0.767 & 1.872 & 1.069 & 1.206 & 0.774 & 1.260 & 0.774 & 1.186 & 0.755 \\
          & \textbf{Avg} & \boldres{0.600} & \boldres{0.540} & 0.620 & \secondres{0.558} & \secondres{0.609} & \boldres{0.540} & 0.971 & 0.639 & 0.977 & 0.630 & 0.911 & 0.629 & 0.980 & 0.768 & 1.063 & 0.753 & 1.106 & 0.716 & 1.162 & 0.718 & 1.105 & 0.706 \\
    \midrule
    \multirow{4}[0]{*}{RW11} & 96    & \boldres{0.529} & \boldres{0.487} & 0.556 & 0.498 & \secondres{0.546} & \secondres{0.495} & 0.975 & 0.601 & 0.929 & 0.593 & 0.967 & 0.583 & 0.901 & 0.740 & 0.707 & 0.623 & 0.906 & 0.601 & 0.956 & 0.608 & 0.906 & 0.596 \\
          & 192   & \boldres{0.586} & \secondres{0.539} & \secondres{0.588} & \boldres{0.535} & 0.596 & 0.542 & 1.062 & 0.642 & 0.999 & 0.630 & 0.936 & 0.619 & 0.898 & 0.738 & 0.726 & 0.625 & 0.999 & 0.656 & 1.108 & 0.670 & 1.014 & 0.651 \\
          & 336   & \secondres{0.628} & \secondres{0.567} & \boldres{0.614} & \boldres{0.559} & 0.639 & 0.571 & 0.916 & 0.635 & 0.900 & 0.643 & 0.846 & 0.635 & 0.906 & 0.740 & 0.985 & 0.765 & 1.089 & 0.690 & 1.146 & 0.698 & 1.070 & 0.679 \\
          & 720   & \boldres{0.655} & \secondres{0.586} & \secondres{0.673} & 0.600 & \boldres{0.655} & \boldres{0.585} & 0.931 & 0.669 & 0.915 & 0.675 & 0.920 & 0.675 & 0.908 & 0.739 & 1.685 & 1.023 & 1.264 & 0.738 & 1.267 & 0.746 & 1.328 & 0.712 \\
          & \textbf{Avg} & \boldres{0.599} & \boldres{0.545} & \secondres{0.608} & \secondres{0.548} & 0.609 & \secondres{0.548} & 0.971 & 0.637 & 0.936 & 0.635 & 0.917 & 0.628 & 0.903 & 0.739 & 1.026 & 0.759 & 1.065 & 0.671 & 1.119 & 0.681 & 1.080 & 0.659 \\
    \midrule
    \multirow{4}[0]{*}{RW12} & 96    & \boldres{0.708} & \boldres{0.554} & 0.751 & 0.569 & \secondres{0.726} & \secondres{0.556} & 1.124 & 0.661 & 1.168 & 0.657 & 1.134 & 0.641 & 1.140 & 0.798 & 0.782 & 0.654 & 1.196 & 0.693 & 1.227 & 0.684 & 1.191 & 0.682 \\
          & 192   & \boldres{0.711} & \boldres{0.605} & 0.791 & \boldres{0.605} & \secondres{0.781} & \secondres{0.606} & 1.285 & 0.707 & 1.206 & 0.690 & 1.183 & 0.688 & 1.149 & 0.802 & 0.836 & 0.650 & 1.331 & 0.754 & 1.415 & 0.755 & 1.342 & 0.746 \\
          & 336   & \boldres{0.811} & \boldres{0.629} & \secondres{0.821} & \secondres{0.631} & 0.827 & 0.637 & 1.544 & 0.720 & 1.252 & 0.720 & 1.236 & 0.722 & 1.146 & 0.800 & 1.011 & 0.755 & 1.423 & 0.792 & 1.502 & 0.793 & 1.454 & 0.783 \\
          & 720  & \boldres{0.838} & \boldres{0.650} & 0.880 & 0.668 & \secondres{0.858} & \secondres{0.657} & 2.417 & 0.780 & 1.342 & 0.768 & 1.529 & 0.790 & 1.158 & 0.804 & 1.910 & 1.092 & 1.589 & 0.842 & 1.662 & 0.846 & 1.572 & 0.827 \\
          & \textbf{Avg}   & \boldres{0.782} & \boldres{0.610} & 0.811 & 0.618 & \secondres{0.798} & \secondres{0.614} & 1.592 & 0.717 & 1.242 & 0.709 & 1.270 & 0.710 & 1.148 & 0.801 & 1.135 & 0.788 & 1.385 & 0.770 & 1.451 & 0.770 & 1.390 & 0.759 \\
    \midrule
    \multirow{4}[0]{*}{RW13} & 96 & \boldres{0.584} & \boldres{0.493} & 0.601 & 0.509 & \secondres{0.597} & \secondres{0.499} & 0.969 & 0.593 & 0.832 & 0.571 & 0.845 & 0.570 & 1.047 & 0.773 & 0.736 & 0.623 & 1.057 & 0.657 & 1.081 & 0.650 & 1.075 & 0.650 \\
          & 192   & \boldres{0.619} & \boldres{0.522} & 0.643 & 0.541 & \secondres{0.622} & \secondres{0.523} & 1.055 & 0.616 & 0.908 & 0.601 & 0.944 & 0.609 & 1.045 & 0.770 & 0.781 & 0.636 & 1.149 & 0.706 & 1.202 & 0.701 & 1.151 & 0.693 \\
          & 336   & \boldres{0.644} & \boldres{0.539} & 0.668 & 0.564 & \secondres{0.648} & \secondres{0.542} & 0.986 & 0.632 & 0.901 & 0.617 & 0.914 & 0.627 & 1.048 & 0.771 & 0.901 & 0.688 & 1.232 & 0.740 & 1.248 & 0.732 & 1.195 & 0.723 \\
          & 720   & \boldres{0.657} & \boldres{0.556} & 0.725 & \secondres{0.613} & \secondres{0.660} & \boldres{0.556} & 1.290 & 0.689 & 0.921 & 0.654 & 1.021 & 0.666 & 1.035 & 0.771 & 1.925 & 1.065 & 1.329 & 0.793 & 1.299 & 0.776 & 1.243 & 0.755 \\
          & \textbf{Avg} & \boldres{0.626} & \boldres{0.528} & 0.659 & 0.557 & \secondres{0.632} & \secondres{0.530} & 1.075 & 0.633 & 0.891 & 0.611 & 0.931 & 0.618 & 1.044 & 0.771 & 1.086 & 0.753 & 1.192 & 0.724 & 1.208 & 0.715 & 1.166 & 0.705 \\
    \midrule
    \multirow{4}[0]{*}{RW14} & 96    & \boldres{0.648} & \boldres{0.520} & 0.666 & 0.532 & \secondres{0.653} & \secondres{0.525} & 1.144 & 0.620 & 1.044 & 0.602 & 1.111 & 0.597 & 1.107 & 0.782 & 0.743 & 0.630 & 1.118 & 0.667 & 1.166 & 0.674 & 1.129 & 0.669 \\
          & 192  & \boldres{0.682} & \boldres{0.550} & 0.696 & 0.560 & \secondres{0.688} & \secondres{0.557} & 1.211 & 0.646 & 0.989 & 0.615 & 1.026 & 0.623 & 1.098 & 0.777 & 0.819 & 0.649 & 1.221 & 0.713 & 1.281 & 0.727 & 1.256 & 0.721 \\
          & 336   & \boldres{0.714} & \boldres{0.573} & 0.722 & 0.581 & \secondres{0.720} & \secondres{0.580} & 1.358 & 0.674 & 1.029 & 0.637 & 1.080 & 0.648 & 1.109 & 0.780 & 0.883 & 0.688 & 1.277 & 0.742 & 1.401 & 0.766 & 1.314 & 0.752 \\
          & 720   & \boldres{0.724} & \boldres{0.587} & 0.789 & 0.628 & \secondres{0.734} & \secondres{0.592} & 1.793 & 0.732 & 1.096 & 0.689 & 1.219 & 0.697 & 1.099 & 0.778 & 1.891 & 1.067 & 1.359 & 0.786 & 1.459 & 0.805 & 1.389 & 0.783 \\
          & \textbf{Avg} & \boldres{0.692} & \boldres{0.557} & 0.718 & 0.575 & \secondres{0.699} & \secondres{0.563} & 1.376 & 0.668 & 1.039 & 0.636 & 1.109 & 0.641 & 1.103 & 0.780 & 1.084 & 0.759 & 1.244 & 0.727 & 1.327 & 0.743 & 1.272 & 0.731 \\
    \midrule
    \multirow{4}[0]{*}{RW15} & 96    & \secondres{0.508} & \secondres{0.484} & \boldres{0.505} & \boldres{0.483} & 0.512 & 0.488 & 0.891 & 0.565 & 0.797 & 0.559 & 0.765 & 0.540 & 0.972 & 0.760 & 0.740 & 0.623 & 0.843 & 0.587 & 0.861 & 0.600 & 0.836 & 0.588 \\
          & 192   & \secondres{0.559} & \secondres{0.524} & \boldres{0.547} & \boldres{0.519} & \secondres{0.559} & 0.529 & 1.051 & 0.613 & 0.837 & 0.599 & 0.758 & 0.584 & 0.962 & 0.758 & 0.751 & 0.632 & 0.932 & 0.635 & 0.980 & 0.659 & 0.933 & 0.639 \\
          & 336   & \boldres{0.583} & \boldres{0.543} & 0.595 & \secondres{0.555} & \secondres{0.594} & \secondres{0.555} & 1.329 & 0.638 & 0.852 & 0.615 & 0.886 & 0.616 & 0.973 & 0.762 & 0.852 & 0.689 & 0.999 & 0.664 & 1.045 & 0.693 & 0.998 & 0.670 \\
          & 720   & \boldres{0.616} & \boldres{0.568} & 0.672 & 0.608 & \secondres{0.641} & \secondres{0.585} & 2.119 & 0.704 & 0.967 & 0.663 & 1.145 & 0.682 & 0.976 & 0.762 & 1.787 & 1.044 & 1.120 & 0.715 & 1.169 & 0.745 & 1.101 & 0.740 \\
          & \textbf{Avg} & \boldres{0.566} & \boldres{0.530} & 0.580 & 0.541 & \secondres{0.576} & \secondres{0.539} & 1.348 & 0.630 & 0.863 & 0.609 & 0.888 & 0.606 & 0.971 & 0.760 & 1.033 & 0.747 & 0.974 & 0.650 & 1.014 & 0.675 & 0.967 & 0.659 \\
\midrule
      \rowcolor[gray]{0.95}
      \multicolumn{2}{c|}{\scalebox{1.1}{\textbf{{$1^{\text{st}}$ Count}}}}
            &
    \multicolumn{2}{c|}{\boldres{124}} &
    \multicolumn{2}{c|}{21} &
    \multicolumn{2}{c|}{12} &
    \multicolumn{2}{c|}{0} &
    \multicolumn{2}{c|}{0} &
    \multicolumn{2}{c|}{0} &
    \multicolumn{2}{c|}{0} &
    \multicolumn{2}{c|}{2} &
    \multicolumn{2}{c|}{0} &
    \multicolumn{2}{c|}{0} &
    \multicolumn{2}{c}{0} \\
      \bottomrule
    \end{tabular}
  }
  \label{tab:zero_shot_full_rw}%
\end{table}

\begin{table}[tbh]
    \centering
    \small
    \setlength{\tabcolsep}{4.5pt}
    \renewcommand{\arraystretch}{1.15}
    \caption{Additional zero-shot forecasting results averaged over four horizons (96, 192, 336, 720). The results from Sundial are from the original paper~\cite{liu2025sundial} because its evaluation protocol is consistent with ours.}
    \begin{tabular}{l c c c c c}
        \toprule
        \textbf{Dataset/Model} & \textbf{FedTRL (Ours)} & \textbf{VisionTS} & \textbf{Sundial (94B)} & \textbf{Sundial (230B)} & \textbf{Sundial (1032B)} \\
        \midrule
        ETTh1 & \textbf{0.399} & 0.412 & 0.402 & 0.403 & 0.411 \\
        ETTh2 & \textbf{0.350} & 0.365 & 0.377 & 0.364 & 0.333 \\
        ETTm1 & \textbf{0.349} & 0.391 & 0.367 & 0.352 & 0.336 \\
        ETTm2 & \textbf{0.282} & 0.296 & 0.280 & 0.273 & 0.258 \\
        Weather & \textbf{0.238} & 0.303 & 0.254 & 0.252 & 0.234 \\
        \midrule
        $1^{st}$ Count & \textbf{5} & 0 & 0 & 0 & 0 \\
        \bottomrule
    \end{tabular}
    \label{tab:R6_zero_shot_forecasting}
\end{table}

\subsection{Zero-shot Probabilistic Forecasting}\label{sec_app:zero_shot_prob}
We evaluate our FedTRL-trained model on GIFT-Eval~\cite{aksu2024gift}, a benchmark designed to comprehensively assess forecasting across diverse time series. It contains 23 datasets with 144,000 series and 177 million data points, covering 97 forecasting configurations. Following the official evaluation suite, we report aggregated results in Table~\ref{tab:gift_eval}. We further assess zero-shot forecasting and inference efficiency on FEV leaderboard~\cite{ansari2024chronos}, established by AutoGluon, which includes 27 datasets. Aggregated forecasting metrics are presented in Fig.~\ref{fig:fev_eval}.

\subsection{Scalability}\label{sec_app:scalability}
We evaluate the scalability of FedTRL by training models of different sizes, with configurations summarized in Table~\ref{tab:scale_config} and results reported in Fig.~\ref{fig:scale}. We further assess scalability with respect to dataset size by controlling the amount of data used for pretraining across \{90\text{B}, 120\text{B}, 180\text{B}, 300\text{B}, 450\text{B}, 540\text{B}\}, where the 450B and 540B configurations are augmented using ERA5 daily, weekly, and monthly data. We evaluate the resulting TSFMs on full-shot, zero-shot, and zero-shot probabilistic forecasting tasks. As shown in \textbf{Table~\ref{tab:rebuttal_forecasting_scales}}, performance consistently improves with larger pretraining datasets, demonstrating the positive scaling behavior of FedTRL.
\begin{table}[tbh]
  \caption{Model configurations of scalability exploration, where gray shading is the default setting.}
  \centering
  \begin{tabular}{ccccccc}
    \toprule
     \scalebox{0.9}{Patch Size} & \scalebox{0.9}{Context Length} & \scalebox{0.9}{Prediction Length}  & \scalebox{0.9}{Layers} & \scalebox{0.9}{Dimension} & \scalebox{0.9}{Heads} & \scalebox{0.9}{Total Parameters} \\
     $(P)$ & $(T)$ & $(F)$  & $(L)$ & $(D, D_\text{ff})$ & $H$  & $\#\text{Count}$ \\
    \midrule
     $16$ & $3072$ & $\{16, 720\}$ & 12 & $(512, 2048)$ & $8$ & $38$M \\
    \midrule
      $16$ & $3072$ & $\{16, 720\}$ & 16 & $(768, 3072)$ & $12$ & $114$M \\
    \midrule
    \rowcolor[gray]{0.95}
     $16$ &  $3072$ & $\{16, 720\}$ & 24 & $(1024, 4096)$ & $16$ & $302$M \\
    \bottomrule
  \end{tabular}
    \label{tab:scale_config}
\end{table}
\begin{table}[tbh]
    \centering
    \small
    \setlength{\tabcolsep}{3.5pt}
    \renewcommand{\arraystretch}{1.12}
    \caption{Full/Zero-shot point forecasting (MSE report) and zero-shot probabilistic forecasting (Avg. MASE report) results across different pretraining datasets scales.}
    \resizebox{.7\textwidth}{!}{%
    \begin{tabular}{llcccccc}
        \toprule
        \multirow{2}{*}{\textbf{Task / Report}} & \multirow{2}{*}{\textbf{Datasets / Metric}} & \multicolumn{6}{c}{\textbf{Pretraining Datasets Scale (Billion)}} \\
        \cmidrule(lr){3-8}
        & & \textbf{90B} & \textbf{120B} & \textbf{180B} & \textbf{300B (Original)} & \textbf{450B} & \textbf{540B} \\
        \midrule
        \multicolumn{8}{l}{\textbf{Full-shot Forecasting}} \\
        \midrule
        \multirow{3}{*}{MSE Report}
        & ETTh1 & 0.401 & 0.387 & 0.372 & 0.371 & 0.363 & \textbf{0.363} \\
        & ETTm1 & 0.343 & 0.336 & 0.328 & 0.316 & 0.312 & \textbf{0.310} \\
        & Weather & 0.232 & 0.231 & 0.227 & 0.214 & 0.204 & \textbf{0.202} \\
        \midrule
        \multicolumn{8}{l}{\textbf{Zero-shot Point Forecasting}} \\
        \midrule
        \multirow{4}{*}{MSE Report}
        & ETTh1 & 0.428 & 0.432 & 0.411 & 0.399 & 0.400 & \textbf{0.393} \\
        & ETTm1 & 0.402 & 0.371 & 0.376 & 0.350 & 0.344 & \textbf{0.341} \\
        & Weather & 0.275 & 0.273 & 0.269 & 0.238 & 0.214 & \textbf{0.208} \\
        & RW-Bench & 0.689 & 0.647 & 0.641 & 0.599 & 0.581 & \textbf{0.575} \\
        \midrule
        \multicolumn{8}{l}{\textbf{Zero-shot Probabilistic}} \\
        \midrule
        \multirow{2}{*}{Ave. MASE Report}
        & GIFT-eval & 0.770 & 0.744 & 0.721 & 0.675 & 0.677 & \textbf{0.669} \\
        & FEV Leaderboard & 0.945 & 0.873 & 0.847 & 0.836 & 0.833 & \textbf{0.829} \\
        \bottomrule
    \end{tabular}}
    \label{tab:rebuttal_forecasting_scales}
\end{table}

\subsection{Additional Discussions}\label{sec_app:additional}
This subsection provides additional analyses of our proposed FedTRL, including its generalization to unseen domains, its representation- and gradient-level behavior in handling heterogeneity, and its robustness under different simulated heterogeneity settings.

\paragraph{Generalization to Unseen Domain.} To further test generalization, we conducted a leave-one-domain-out study (\textbf{Tables~\ref{tab:R3_in_domain_forecasting}–\ref{tab:R4_zero_shot_forecasting}} below), where an entire domain was removed during pretraining. Even under this stricter setting, FedTRL consistently outperforms centralized mixed-domain pretraining and other FL baselines on the unseen domain. Performance naturally decreases when excluding related domains (an expected phenomenon for all foundation models) but FedTRL remains the most robust under such conditions.

\begin{table}[tbh]
    \centering
    \small
    \setlength{\tabcolsep}{5pt}
    \renewcommand{\arraystretch}{1.15}
    \caption{In-domain forecasting results (MSE averaged across $\{96, 192, 336, 720\}$). Corresponding domain dataset has been excluded from the pretraining dataset (energy, traffic, respectively).}
    \begin{tabular}{l c c c c}
        \toprule
        \textbf{Method} & \textbf{ETTh1} & \textbf{ETTm1} & \textbf{Traffic} & \textbf{Weather} \\
        \midrule
        FedAvg & 0.511 & 0.433 & 0.488 & 0.295 \\
        FFTS & 0.492 & 0.421 & 0.479 & 0.290 \\
        Centralized All Mixed & 0.498 & 0.427 & 0.483 & 0.290 \\
        FedTRL (Original) & 0.448 & 0.375 & 0.426 & 0.241 \\
        \textbf{FedTRL (Ours)} & \textbf{0.472} & \textbf{0.411} & \textbf{0.477} & \textbf{0.282} \\
        \bottomrule
    \end{tabular}
    \label{tab:R3_in_domain_forecasting}
\end{table}

\begin{table}[tbh]
    \centering
    \small
    \setlength{\tabcolsep}{5pt}
    \renewcommand{\arraystretch}{1.15}
    \caption{Zero-shot forecasting results (MSE averaged across $\{96, 192, 336, 720\}$). Corresponding domain dataset has been excluded from the pretraining dataset (energy, traffic, respectively).}
    \begin{tabular}{l c c c}
        \toprule
        \textbf{Method} & \textbf{ETTh1} & \textbf{ETTm1} & \textbf{Weather} \\
        \midrule
        FedAvg & 0.438 & 0.384 & 0.292 \\
        FFTS & 0.427 & 0.380 & 0.282 \\
        Centralized All Mixed & 0.431 & 0.382 & 0.288 \\
        FedTRL (Original) & 0.399 & 0.349 & 0.238 \\
        \midrule
        \textbf{FedTRL} & 0.418 ($\downarrow$ 4.76\%) & 0.366 ($\downarrow$ 4.87\%) & 0.269 ($\downarrow$ 13.03\%) \\
        \bottomrule
    \end{tabular}
    \label{tab:R4_zero_shot_forecasting}
\end{table}

\paragraph{Representation- and Gradient-Level Analysis.} \textbf{Table~\ref{tab:R1_domain_analysis}} shows that centralized mixed-domain pretraining yields the smallest inter-domain centroid distance and the largest within-domain variance, indicating collapsed and weakly structured representations. In contrast, FedTRL maintains clear domain separation while producing the most compact within-domain clusters. \textbf{Table~\ref{tab:R2_gradient_conflict}} further shows that All-Mixed exhibits the strongest gradient conflict (lowest cosine similarity and highest gradient-norm variance), whereas FedTRL substantially alleviates these conflicts.

\begin{table}[tbh]
    \centering
    \small
    \setlength{\tabcolsep}{5pt}
    \renewcommand{\arraystretch}{1.15}
    \caption{Representation-level Domain Analysis. We report the average inter-domain centroid distance (higher is better separation) and average within-domain variance (lower is more compact).}
    \begin{tabular}{l c c}
        \toprule
        \textbf{Method} & \textbf{Inter-domain Centroid Distance} ($\uparrow$) & \textbf{Within-domain Variance} ($\downarrow$) \\
        \midrule
        All Mixed (Centralized) & 1.08 & 0.91 \\
        FedAvg & 1.06 & 0.85 \\
        FFTS & 1.34 & 0.69 \\
        \textbf{FedTRL (Ours)} & \textbf{1.27} & \textbf{0.53} \\
        \bottomrule
    \end{tabular}
    \label{tab:R1_domain_analysis}
\end{table}

\begin{table}[tbh]
    \centering
    \small
    \setlength{\tabcolsep}{5pt}
    \renewcommand{\arraystretch}{1.15}
    \caption{Gradient Conflict Analysis Across Domains. We report the average pairwise cosine similarity between domain-wise gradients (higher means less conflict) and the variance of gradient norms across domains (lower means more balanced contributions).}
    \begin{tabular}{l c c}
        \toprule
        \textbf{Method} & \textbf{Avg. Gradient Cosine Similarity} ($\uparrow$) & \textbf{Gradient Norm Variance} ($\downarrow$) \\
        \midrule
        All Mixed (Centralized) & 0.23 & 0.021 \\
        FedAvg & 0.33 & 0.015 \\
        FFTS & 0.37 & 0.013 \\
        \textbf{FedTRL (Ours)} & \textbf{0.43} & \textbf{0.010} \\
        \bottomrule
    \end{tabular}
    \label{tab:R2_gradient_conflict}
\end{table}

\paragraph{Robustness Under Increasing Heterogeneity.} By gradually increasing domain imbalance (from Balanced to Mildly Imbalanced to Real Skew) by controlling the data scale of each domain, we observe consistent performance degradation for All-Mixed, FedAvg, and FFTS across in-domain, full-shot, zero-shot, and probabilistic forecasting (as shown in \textbf{Table~\ref{tab:combined_imbalance_results}}). In contrast, FedTRL remains markedly more stable in all settings, indicating that severe heterogeneity harms federated training and that FedTRL effectively mitigates this effect.

\begin{table}[tbh]
    \centering
    \small
    \setlength{\tabcolsep}{3.5pt}
    \renewcommand{\arraystretch}{1.12}
    \caption{Forecasting results under different inter-domain imbalance levels.}
    \resizebox{\textwidth}{!}{%
    \begin{tabular}{llccc}
        \toprule
        \multirow{2}{*}{\textbf{Task Type / Dataset (Metric)}} & \multirow{2}{*}{\textbf{Method}} & \multicolumn{3}{c}{\textbf{Inter-domain Imbalance Level}} \\
        \cmidrule(lr){3-5}
        & & \textbf{Balanced (Simulated)} & \textbf{Mildly Imbalanced (Simulated)} & \textbf{Real Skew (Original Pretraining)} \\
        \midrule
        \multirow{4}{*}{\shortstack[l]{\textbf{In-domain Forecasting} \\ (MSE on Weather)}}
        & All Mixed (Centralized) & 0.235 & 0.248 & 0.254 \\
        & FedAvg & 0.231 & 0.248 & 0.264 \\
        & FFTS & 0.224 & 0.235 & 0.252 \\
        & \textbf{FedTRL (Ours)} & \textbf{0.216} & \textbf{0.230} & \textbf{0.241} \\
        \midrule
        \multirow{4}{*}{\shortstack[l]{\textbf{Full-shot Forecasting} \\ (MSE on Weather)}}
        & All Mixed (Centralized) & 0.227 & 0.235 & 0.238 \\
        & FedAvg & 0.228 & 0.239 & 0.240 \\
        & FFTS & 0.212 & 0.219 & 0.227 \\
        & \textbf{FedTRL (Ours)} & \textbf{0.204} & \textbf{0.211} & \textbf{0.214} \\
        \midrule
        \multirow{4}{*}{\shortstack[l]{\textbf{Zero-shot Forecasting} \\ (MSE on RW-Bench)}}
        & All Mixed (Centralized) & 0.610 & 0.608 & 0.614 \\
        & FedAvg & 0.708 & 0.710 & 0.729 \\
        & FFTS & 0.605 & 0.614 & 0.614 \\
        & \textbf{FedTRL (Ours)} & \textbf{0.587} & \textbf{0.592} & \textbf{0.599} \\
        \midrule
        \multirow{4}{*}{\shortstack[l]{\textbf{Zero-shot Probabilistic} \\ (MASE on GIFT-eval / FEV)}}
        & All Mixed (Centralized) & 0.812 / 0.867 & 0.820 / 0.875 & 0.823 / 0.894 \\
        & FedAvg & 0.835 / 0.877 & 0.849 / 0.879 & 0.872 / 0.890 \\
        & FFTS & 0.752 / 0.852 & 0.754 / 0.855 & 0.766 / 0.876 \\
        & \textbf{FedTRL (Ours)} & \textbf{0.670 / 0.826} & \textbf{0.669 / 0.832} & \textbf{0.675 / 0.836} \\
        \bottomrule
    \end{tabular}}
    \label{tab:combined_imbalance_results}
\end{table}

\section{Showcases}\label{sec_app:showcases}
\paragraph{Reconstruction from Noise in Pretraining.} To illustrate the effectiveness of our diffusion-based reconstruction strategy, we showcase examples of noisy inputs, reconstructions, and original sequences. As shown in \textbf{Fig.~\ref{fig:showcase_noise}}, the model successfully restores a wide range of temporal dynamics, including non-stationary series with high-frequency fluctuations, periodic components, approximate stationarity, and nonlinear trends. These cases highlight the ability of FedTRL to preserve fine-grained temporal structures and denoise corrupted inputs, even under complex and unstable conditions, thereby ensuring robust representation learning during pretraining.

\begin{figure}[tbh]
    \centering
    \includegraphics[width=1\textwidth]{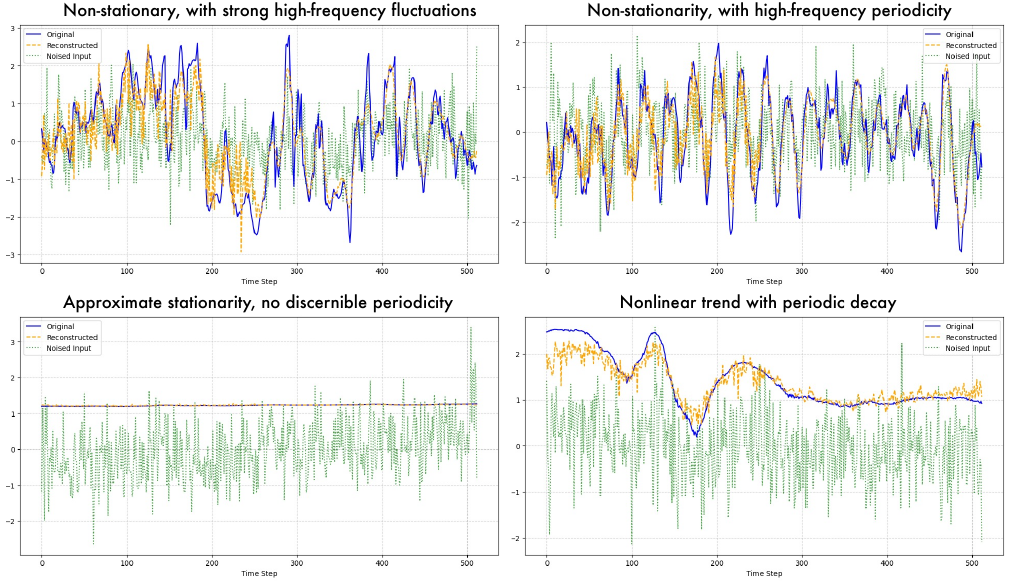}
    \caption{Showcase of diffusion-based reconstruction during pretraining.}
    \label{fig:showcase_noise}
\end{figure}

\paragraph{Forecasting.} \textbf{Fig.~\ref{fig:poin_showcase}-~\ref{fig:pro_showcase}} present zero-shot forecasting showcases on all the datasets from our RW-Bench (point forecasting) and FEV-leaderboard~\cite{ansari2024chronos} (probabilistic forecasting).

\begin{figure}[tbh]
    \centering
    \includegraphics[width=1\textwidth]{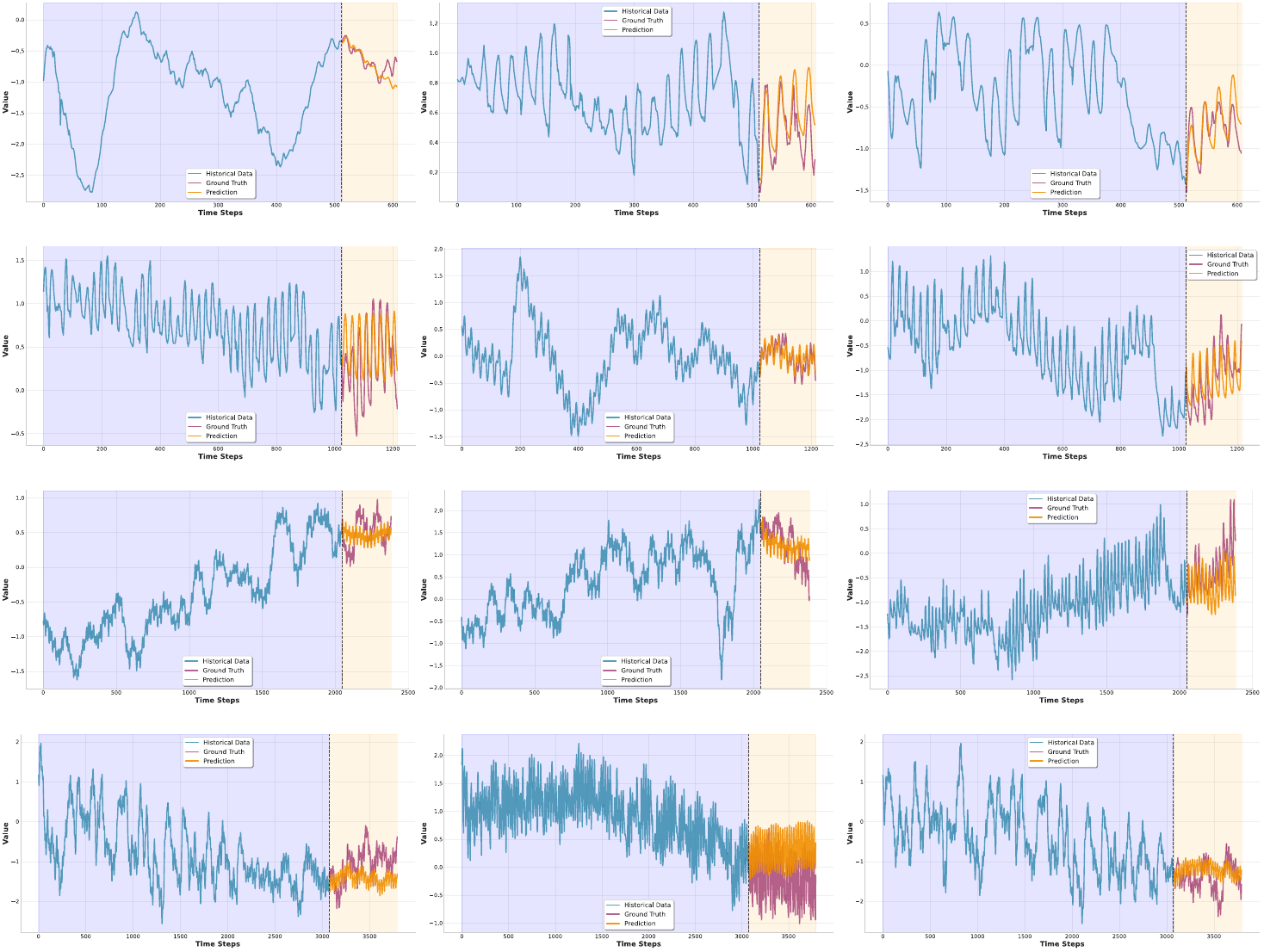}
    \caption{\small Zero-shot point forecasting results on our RW-Bench. From top to bottom, the prediction horizons are 96, 192, 336, and 720, corresponding to lookback window lengths of 512, 1024, 2048, and 3072, respectively. The visual samples are randomly selected from RW-Bench.}
    \label{fig:poin_showcase}
\end{figure}

\begin{figure}[tbh]
    \centering
    \includegraphics[width=1\linewidth]{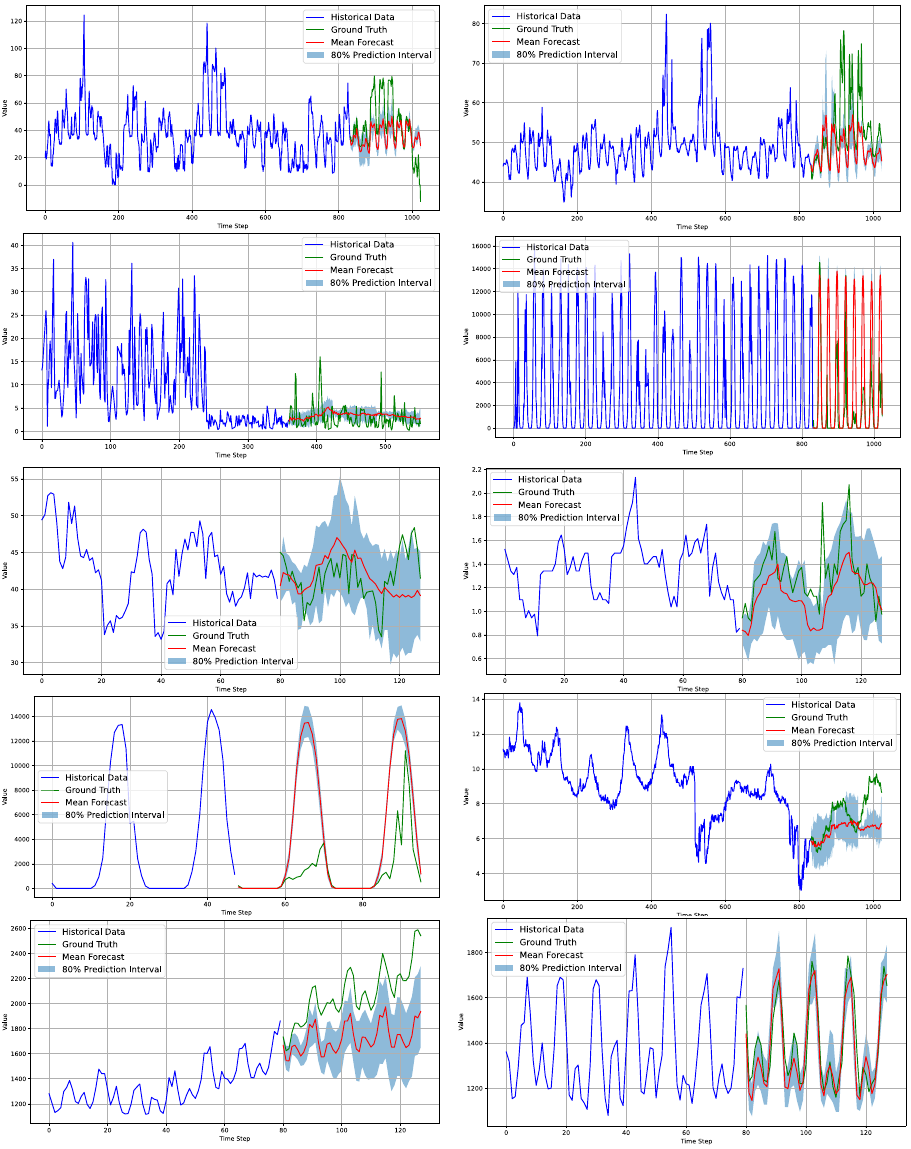}
    \caption{\small Zero-shot probabilistic forecasting results on FEV-leaderboard datasets, with visual samples randomly drawn from the benchmark.}
    \label{fig:pro_showcase}
\end{figure}


\end{document}